\documentclass[opre,nonblindrev]{informs4a}
\OneAndAHalfSpacedXI

\usepackage{mathtools,bm}
\usepackage{aliascnt}
\usepackage{booktabs,tabularx}
\usepackage{enumitem}
\usepackage[nopatch=footnote]{microtype}
\usepackage{natbib}
\usepackage[hyperfootnotes=false]{hyperref}
\usepackage[capitalize,noabbrev]{cleveref}
\usepackage{algorithm}
\usepackage[noend]{algpseudocode}
\usepackage{float}

\bibpunct[, ]{(}{)}{,}{a}{}{,}
\def\bibfont{\footnotesize\linespread{1}\selectfont}

\hypersetup{
 hypertexnames=false,
 colorlinks=true,
 linkcolor=blue,
 citecolor=blue,
 urlcolor=blue,
 filecolor=blue,
}
\newtheorem{theorem}{Theorem}[section]
\newaliascnt{lemma}{theorem}
\newtheorem{lemma}[lemma]{Lemma}
\aliascntresetthe{lemma}
\newaliascnt{proposition}{theorem}
\newtheorem{proposition}[proposition]{Proposition}
\aliascntresetthe{proposition}
\newaliascnt{corollary}{theorem}

\aliascntresetthe{corollary}
\newaliascnt{assumption}{theorem}
\newtheorem{assumption}[assumption]{Assumption}
\aliascntresetthe{assumption}
\newaliascnt{definition}{theorem}

\aliascntresetthe{definition}
\newaliascnt{remark}{theorem}
\newtheorem{remark}[remark]{Remark}
\aliascntresetthe{remark}
\numberwithin{equation}{section}
\crefname{theorem}{Theorem}{Theorems}
\Crefname{theorem}{Theorem}{Theorems}
\crefname{lemma}{Lemma}{Lemmas}
\Crefname{lemma}{Lemma}{Lemmas}
\crefname{proposition}{Proposition}{Propositions}
\Crefname{proposition}{Proposition}{Propositions}
\crefname{corollary}{Corollary}{Corollaries}
\Crefname{corollary}{Corollary}{Corollaries}
\crefname{assumption}{Assumption}{Assumptions}
\Crefname{assumption}{Assumption}{Assumptions}
\crefname{definition}{Definition}{Definitions}
\Crefname{definition}{Definition}{Definitions}
\crefname{remark}{Remark}{Remarks}
\Crefname{remark}{Remark}{Remarks}

\newcommand{\E}{\mathbb E}
\newcommand{\1}{\mathbf 1}
\newcommand{\dd}{\mathrm d}
\newcommand{\proj}{\operatorname{Proj}}
\newcommand{\calH}{\mathcal H}
\newcommand{\calK}{\mathcal K}
\newcommand{\Dsum}{\overline D}
\newcommand{\RAPDL}{\ensuremath{\mathsf{RAPDL}}}

\begin{document}

\RUNTITLE{Resource-Adaptive Primal-Dual Learning for OWMS Systems}
\TITLE{Resource-Adaptive Primal-Dual Learning for One-Warehouse Multi-Store
Systems with Censored Demand}
\RUNAUTHOR{Lyu}
\ARTICLEAUTHORS{
\AUTHOR{Jiameng Lyu}
\AFF{Department of Management Science, School of Management, Fudan University,
Shanghai 200433, China, \EMAIL{jiamenglyu@fudan.edu.cn}}
}

\ABSTRACT{
The one-warehouse multi-store (OWMS) system is a fundamental inventory
network in which a nonreplenishable warehouse allocates shared stock across
multiple stores over time.
Existing OWMS learning policies are built around
a fixed target calibrated to the initial
average resource rate, but such a fixed-target architecture cannot re-center
after realized sales change
the remaining resource available per future period.  We develop
Resource-Adaptive Primal-Dual Learning, a new
learning framework that tracks the  Primal-Dual re-solving path with censored demand as
the remaining-resource state evolves.  In each period, the current resource
rate indexes the target store allocations and dual variable, while censored
sales provide gradient estimates for updating both.  The analysis
combines expected-sales geometry with a moving-target argument to yield
logarithmic expected regret, improving on the state-of-the-art
square-root-order guarantees
of existing OWMS learning policies.  The underlying design and analytical
ideas may inform other online learning problems
with depleting shared resources.  Numerical experiments further demonstrate
good finite-horizon performance of a practical variant across different
horizon lengths and inventory regimes.

}

\maketitle

\section{Introduction}

The one-warehouse multi-store (OWMS) system is a basic architecture of inventory distribution, which pools finite inventory at a
central warehouse and allocates it across stores and time.  Each shipment
trades current service at one location against the option of serving another
later, creating a systemwide dynamic allocation problem
\citep{jackson1988,chenZheng1997}.
Applications include general-merchandise distribution,
perishable-product distribution \citep{huangEtAl2025}, fast-fashion
distribution \citep{caroGallien2010,bekci2023}, service-parts logistics
\citep{kutanogluMahajan2009}, emergency-medical supply allocation
\citep{mehrotraEtAl2020}.

This paper focuses on a finite-horizon version in which the warehouse
receives an initial stock and ships it to stores with different demand
distributions and economic margins.  Each allocation changes both tomorrow's
inventory state and the opportunity cost of every future unit.  Even with
known demand laws, the exact dynamic program is high dimensional and strongly
state dependent.  With unknown distributions and censored feedback, the
manager must learn demand responses while preserving the same inventory
needed to earn future reward.

For newsvendor and related inventory problems, a
positive density lower bound supplies the strong convexity that converts
stochastic first-order learning into logarithmic regret
\citep{huhRusmevichientong2009,shi2016,zhangChaoShi2018,lyuEtAl2025}. 
These models effectively draw replenishment from an unconstrained upstream
source. With a nonreplenishable warehouse, by contrast, all stores and
periods compete for the same finite inventory pool.
A sale at one store irreversibly
reduces the resource available to every store in subsequent periods.
The best existing finite-stock censored-demand
guarantees are  square-root order for OWMS and
its multiwarehouse extension \citep{bekci2023,miaoWangZhao2023}.

Re-solving is a powerful and well-received organizing principle for dynamic resource
allocation in both academia and industries. It reveals how allocations and scarcity prices should adapt
to an evolving resource state, but learning work rarely treats the resulting
Primal-Dual path itself as the object to be tracked. We
develop a new resource-adaptive
Primal-Dual learning framework that replaces fixed-target learning with
online tracking of the
endogenous Primal-Dual re-solving path using censored demand.  This framework
yields a logarithmic regret guarantee for the OWMS learning problem, and its
design and analytical tools may inform
other online
learning-and-control problems with censored feedback and depleting shared
resources.

\subsection{Contributions}
Our contributions are summarized as follows.
\begin{enumerate}[leftmargin=2em]
\item \textbf{Resource-adaptive Primal-Dual learning framework.}
We introduce Resource-Adaptive Primal-Dual Learning (\RAPDL), a new
framework that replaces the fixed
target calibrated to the initial average resource rate with online tracking
of the endogenous Primal-Dual path generated by resource-state re-solving.
\RAPDL\ uses
remaining system inventory per remaining period to re-index the
fluid KKT target---the preferred store allocations and their common scarcity
price.  Censored sales provide observable stochastic Primal-Dual directions
for updating this target and simultaneously update the resource
state that re-centers the next target.  In this way, \RAPDL\ tracks the allocation-and-price path that
known-distribution fluid re-solving would produce,
without estimating the demand distributions or repeatedly solving a fluid
program.

\item \textbf{Analysis of endogenous moving targets.}
Our analysis separates regret into three components.  Value-function
smoothness controls the intrinsic gap of the resource-indexed re-solving path.
Expected-sales geometry and a moving-target argument control the gap from
learning that path.  The inventory dynamics control the gap from implementing
the learned state as a feasible physical action.  Each component is
logarithmic, yielding an \(O(\log T)\) regret guarantee for \RAPDL\ and
improving on the state-of-the-art square-root-order guarantees in the
finite-stock censored-demand OWMS literature.

\item \textbf{Numerical evaluation.}
We compare our resource-adaptive learning method \RAPDL\
with the Double Binary Search policy of \citet{bekci2023}, a state-of-the-art
benchmark across different inventory regimes.  Across both homogeneous and
heterogeneous experiments, \RAPDL\ has lower mean total
cost in all 24 paired comparisons, providing short-horizon evidence for
our algorithm.

\end{enumerate}

\subsection{Related Literature}
\label{sec:literature}
\textbf{Re-solving and Primal-Dual learning with unreplenishable
resources.}
A broad revenue-management and online-allocation literature feeds observed
resource consumption back into future decisions.  Re-solving policies
recompute a fluid program with the remaining capacity and horizon and achieve
sublinear guarantees under suitable structural conditions
\citep{jasinKumar2012,jasin2014}.  These
models assume known demand distributions and thus involve no learning.
Related online-LP work combines
learning with repeated LP optimization under fully observed requests
\citep{agrawalWangYe2014,liYe2022,liWangZhang2024}.  These methods learn
dual prices from observed arrivals through empirical LP solves.

Primal-Dual demand learning treats the fluid primal and dual
solutions themselves as unknown.  Related approaches include  explicit learning of an
inventory shadow price in personalized pricing \citep{chenGallego2022}, and
nonparametric or large-action-space network revenue management
\citep{chenLyuWangZhou2024,miaoWangDemandBalancing2024,miaoWang2025,miaoWangZhang2026}.
These approaches combine demand learning, dual-price learning, and resource
feedback through different architectures.  In the finite-stock OWMS setting,
\RAPDL\ instead uses censored sales to jointly track store-level allocation
targets and their common scarcity price along the KKT path indexed by remaining
system inventory per remaining period.

\textbf{OWMS control and learning.}
The finite-stock OWMS problem originates in the two-echelon allocation
setting of \citet{jackson1988}.  With known demand laws, the literature
develops tractable decomposition and Lagrangian policies for OWMS problem
\citep{marklundRosling2012,miaoJasinChao2022}.
Within this known-demand control literature, the policy of
\citet{chaoJasinMiao2025} adaptively readjusts its Lagrangian parameters
using realized demand history.

The closest learning papers retain that finite shared stock and explicitly
learn its scarcity price.  \citet{bekci2023} proposes Double Binary Search
for OWMS, and \citet{miaoWangZhao2023} combines dual cutting planes with
primal learning in a multiwarehouse extension.  Both learn a fixed
Lagrangian target and obtain the best available square-root-order regret
guarantees for this class.
By contrast, \RAPDL\ learns a resource-indexed Primal-Dual path and re-centers
that path after realized sales change the remaining resource state.

\textbf{Gradient-based inventory learning.}
For basic censored-demand inventory models, sales and stockout observations
can provide first-order information for gradient-based learning
\citep{huhRusmevichientong2009,shi2016,lyuEtAl2025,guoEtAl2026}.
Extensions cover positive lead times, perishable inventory, random
capacities, and multi-retailer systems
\citep{huh2009,zhangChaoShi2018,zhangChaoShi2020,chenShiDuenyas2020,
fanChenXiaoZhou2023}.
\citet{lyuEtAl2025} applies a minibatch-SGD metapolicy to a
two-echelon system in which the warehouse itself can replenish, whereas \citet{huangEtAl2025} develops an offline, feature-based approximation for a
perishable network in which both the warehouse and retailers replenish.

\textbf{Organization.}
\Cref{sec:model} presents the physical system, exact comparator, fluid
benchmark.  \Cref{sec:algorithm} derives and
states \RAPDL.  \Cref{sec:main-results} gives the main guarantee and develops
its regret analysis, including the expected-sales geometry used only in the
analysis. \Cref{sec:numerical-experiments} reports the computational study,
and \Cref{sec:discussion} concludes the paper.  The appendices collect the
proofs of all propositions and lemmas stated in the main text, the supporting
technical verifications, and the extension sketches.

\textbf{Notation.}
Throughout, \(T\ge2\), \(\log\) denotes the natural logarithm, and
\(C<\infty\) denotes a constant independent of \(T\) that may change from
line to line. Let \(x^+:=\max\{x,0\}\).

\section{Problem Formulation}
\label{sec:model}

There are stores \(i\in[N]:=\{1,\ldots,N\}\) and periods
\(t\in[T]:=\{1,\ldots,T\}\).  Immediately before the period-\(t\)
shipment, \(B_t\) is the stock physically remaining at the warehouse and
\(I_{i,t}\) is the on-hand stock already located at store \(i\).
The decision \(Y_{i,t}\) is store \(i\)'s stock \emph{after} that shipment.
The warehouse begins with a nonreplenishable stock
\(B_1=W=\gamma T\), and every store begins empty, \(I_{i,1}=0\).
Unless stated otherwise, vector norms are Euclidean.
At the beginning of period \(t\), a feasible post-shipment vector satisfies
\begin{equation}
 Y_{i,t}\ge I_{i,t},\qquad
 \sum_{i=1}^N(Y_{i,t}-I_{i,t})\le B_t.
 \label{eq:physical-feasibility}
\end{equation}
Demands, sales, and states obey
\begin{equation}
 S_{i,t}=D_{i,t}\wedge Y_{i,t},\quad
 I_{i,t+1}=Y_{i,t}-S_{i,t},\quad
 B_{t+1}=B_t-\sum_{i=1}^N(Y_{i,t}-I_{i,t}).
 \label{eq:dynamics}
\end{equation}
In each period, the manager observes the state \((B_t,\bm I_t)\)  and past sales, chooses
\(\bm Y_t\) before demand, observes only censored sales \(\bm S_t\), and then
updates the state by \eqref{eq:dynamics}. Unmet demand is lost.

We impose the following conditions on the primitive uncertainty.
\begin{assumption}
\label{ass:primitives}
\begin{enumerate}[label=\textup{(\roman*)},leftmargin=*]

\item
The demand vectors \(\{\bm D_t\}_{t\ge1}\) are i.i.d.\ over time.
Within each period, their coordinates may be arbitrarily dependent,
and the marginal distribution of \(D_{i,t}\) is \(F_i\).

\item
The law \(F_i\) is absolutely continuous on its support
\([0,\bar d_i]\), satisfies \(F_i(x)=1\) for \(x\ge\bar d_i\), and has a
continuous density obeying the known bounds
\begin{equation}
 0<\kappa_i\le f_i(x)\le K_i<\infty,
 \qquad 0\le x\le\bar d_i.
 \label{eq:density-bounds}
\end{equation}
\end{enumerate}
\end{assumption}

Write \(\Dsum:=\sum_{i=1}^N\bar d_i\) and suppose that the initial
resource rate satisfies \(0\le\gamma\le\Dsum\).

\begingroup

\begin{remark}
\label{rem:envelope-calibration}
For the directly clipped formulation, the endpoint \(\bar d_i\) and
valid density bounds \((\kappa_i,K_i)\) are known problem data.  This does not
mean that the demand law is known: the policy does not know the complete
function \(F_i\), and the density bounds need not equal its exact extrema.
Appendix~\ref{app:endpoint-envelope} describes a fixed-support-envelope
implementation and sketches a localization route that avoids using the exact
endpoint values.
Using such envelopes to certify theoretical parameter tuning is standard in
inventory-learning analyses
\citep{huhRusmevichientong2009,shi2016,zhangChaoShi2018,bekci2023,lyuEtAl2025}.\footnote{
For example, 
in \citet{huhRusmevichientong2009,shi2016}, the density lower bound enters
the stepsize explicitly. In the closest OWMS/MWMS work, \citet{bekci2023,miaoWangZhao2023} assumes
known support and lower/upper density bounds, which also enter the theoretical parameter-tuning of their algorithms.
}
Looser valid envelopes enter the associated safeguards and
primitive-dependent bounds.
\end{remark}
\endgroup

For a policy \(\pi\), let \(\mathcal U^\pi\) denote any exogenous
policy randomization, independent of demand, and define the pre-demand
history by
\[
 \calH_t^\pi:=
 \sigma\!\left(
 \mathcal U^\pi,B_1,\bm I_1,
 (\bm Y_\tau,\bm S_\tau,B_{\tau+1},\bm I_{\tau+1})_{\tau<t}
 \right).
\]
Under \Cref{ass:primitives}, \(\bm D_t\) is independent of
\(\calH_t^\pi\).
Let \(\Pi_{\rm ad}\) denote the class of policies that, from the stated
initial state, choose \(\bm Y_t\) measurably with respect to
\(\calH_t^\pi\), satisfy \eqref{eq:physical-feasibility}, and evolve
according to \eqref{eq:dynamics} in every period.
For any \(\pi\in\Pi_{\rm ad}\), define its expected cost by
\begin{equation*}
\begin{aligned}
 J_T^\pi
 &:=
 \E^\pi\!\Bigg[
\sum_{t=1}^T\sum_{i=1}^N
\Bigl(
 \underbrace{k_i(Y_{i,t}-I_{i,t})}_{\text{shipment cost}}
 +\hspace{-2mm}\underbrace{h_iI_{i,t+1}}_{\text{store holding cost}}\hspace{-2mm}
 +\underbrace{b_i(D_{i,t}-Y_{i,t})^+}_{\text{lost-sales penalty}}
\Bigr)
 +\hspace{-5mm}\underbrace{wB_{T+1}}_{\text{terminal warehouse cost}}
 -\underbrace{\sum_{i=1}^N c_iI_{i,T+1}}_{\text{sell-back credit}}
\Bigg].
\end{aligned}
\end{equation*}
Here \(k_i,h_i,b_i\) are per unit shipment, store-holding, and lost-sales
costs, respectively, and \(w\) is terminal warehouse holding cost.  Terminal
store inventory is credited at \(c_i=k_i-w\) as convention. Throughout,
\(c_i\ge0\), \(h_i>0\), and \(b_i-c_i>0\).

For the fluid comparison, let \(\mu_i:=\E D_i\), and, for a stock threshold
\(y\ge0\), define the expected sales
\begin{equation}
 m_i(y):=\E[D_i\wedge y]
 =\int_0^y(1-F_i(u))\,\dd u
 \label{eq:expected-sales}
\end{equation}
and the expected adjusted operating loss
\begin{align}
 \ell_i(y)
 &:=\E\!\left[
 c_i(D_i\wedge y)+h_i(y-D_i)^+
 +b_i(D_i-y)^+
 \right]\notag\\
 &=h_i y+b_i\mu_i-(h_i+b_i-c_i)m_i(y).
 \label{eq:expected-loss}
\end{align}
Here \(m_i(y)\) is the expected depletion of the shared system inventory,
whereas \(\ell_i(y)\) is the corresponding expected decision-dependent
cost after the terminal warehouse charge is accounted for.  Through direct
inventory-accounting calculations, for every admissible policy \(\pi\),
\begin{align}
 J_T^\pi
 &=wW+\sum_{t=1}^T\sum_{i=1}^N
 \E^\pi\ell_i(Y_{i,t}),
 \label{eq:exact-cost}\\
 \sum_{t=1}^T\sum_{i=1}^N\E^\pi m_i(Y_{i,t})
 &=W-\E^\pi\!\left[B_{T+1}+\sum_{i=1}^NI_{i,T+1}\right]
 \le W=\gamma T.
 \label{eq:total-sales-budget}
\end{align}

The exact known-distribution dynamic benchmark is the optimal value over the
same admissible policy class:
\begin{equation*}
\begin{aligned}
 J_T^*&:=\inf_{\pi\in\Pi_{\rm ad}}J_T^\pi.
\end{aligned}
\end{equation*}
The identities \eqref{eq:exact-cost} and \eqref{eq:total-sales-budget}
motivate a one-period fluid relaxation that retains only the expected-sales
budget.  The following proposition defines its value and shows that it
lower-bounds the exact dynamic benchmark.

\begin{proposition}
\label{prop:fluid-benchmark}
For \(r\ge0\), define
\begin{equation}
 v(r):=\min_{0\le y_i\le\bar d_i}
 \left\{\sum_{i=1}^N\ell_i(y_i):
 \sum_{i=1}^Nm_i(y_i)\le r\right\}.
 \label{eq:y-fluid}
\end{equation}
Then we have $J_T^*\ge wW+Tv(\gamma)$.
\end{proposition}

By \Cref{prop:fluid-benchmark}, the fluid-benchmark gap is nonnegative and upper-bounds
the performance gap relative to the exact dynamic oracle:
\begin{equation*}
 0\le J_T^\pi-J_T^*\le J_T^\pi-\bigl(wW+Tv(\gamma)\bigr).
\end{equation*}
We therefore define the paper's regret directly relative to the constrained
fluid benchmark:
 \begin{equation}
 R_T(\pi)
 :=J_T^\pi-\bigl(wW+Tv(\gamma)\bigr).
 \label{eq:regret-definition}
\end{equation}

For each resource rate \(r\), let \(\bm y^*(r)\) denote the unique optimizer
of the one-period fluid program \eqref{eq:y-fluid}, and let
\(\lambda^*(r)\) be its selected smallest resource multiplier.   Their existence and regularity are established in
\Cref{lem:kkt-path}.

\section{Resource-Adaptive Primal-Dual Learning}
\label{sec:algorithm}

The state driving our resource-adaptive framework is the remaining-resource rate.  At period
\(t\), define
\begin{equation*}
 C_t:=B_t+\sum_{i=1}^NI_{i,t},\qquad
 n_t:=T-t+1,\qquad r_t:=\frac{C_t}{n_t}.
\end{equation*}
Here \(C_t\) includes inventory at both the warehouse and the stores, and
\(r_t\) is the inventory available per remaining period. In particular,
\(r_1=W/T=\gamma\).  Because shipments only relocate inventory while sales
deplete it,
\begin{equation*}
 C_{t+1}=C_t-\sum_{i=1}^NS_{i,t},\qquad
 r_{t+1}
 =r_t+\frac{r_t-\sum_{i=1}^NS_{i,t}}{n_t-1}
 \qquad(n_t\ge2).
\end{equation*}
Thus each sales observation both provides censored feedback and changes the
next resource index.  \RAPDL\ maintains preferred thresholds \(\bm X_t\) and
a common shadow price \(\lambda_t\) to track the moving target
\((\bm y^*(r_t),\lambda^*(r_t))\), rather than the fixed target indexed by
\(\gamma\).

At a high level, each period has four operations: project the preferred
thresholds onto the physically feasible inventory set, observe censored sales
and form a Primal-Dual sample, update the remaining-resource rate, and take a
projected stochastic step.  The rate \(r_t\) moves the fluid target after every
sales realization.  Section~3.1 specifies the complete executable recursion, while
Sections~3.2 and~3.3 explain its gradient construction, safeguards, and
step-size design.

\subsection{\texorpdfstring{The Executable RAPDL Recursion}{The Executable RAPDL Recursion}}
\label{sec:executable-recursion}

\Cref{alg:rapdl} gives the complete executable RAPDL pseudocode.  We first
define the action projection and the quantities appearing in its updates.

The preferred threshold \(\bm X_t\) is a learning state and need not be a
feasible post-shipment inventory.  Given the current physical state, the
feasible action set is
\begin{equation*}
 \mathcal F_t
 :=\left\{\bm y:y_i\ge I_{i,t},\
       \sum_{i=1}^N(y_i-I_{i,t})\le B_t\right\}
 =\left\{\bm y:y_i\ge I_{i,t},\
       \sum_{i=1}^N y_i\le C_t\right\}.
\end{equation*}
The right-hand side uses \(C_t\), rather than \(B_t\), because store
carryover is already part of the available system inventory.  We implement
the Euclidean projection of \(\bm X_t\) onto \(\mathcal F_t\),
\[
 \bm Y_t=\argmin_{\bm y\in\mathcal F_t}
 \frac12\|\bm y-\bm X_t\|^2,
\]
which has the water-filling representation
\citep[Lemma~2]{duchi2008}
\begin{equation*}
 \begin{aligned}
 \nu_t:=\inf\left\{\nu\ge0:
   \sum_{i=1}^N\max\{I_{i,t},X_{i,t}-\nu\}\le C_t\right\},~~~
 Y_{i,t}=\max\{I_{i,t},X_{i,t}-\nu_t\},
 \qquad i\in[N].
 \end{aligned}
\end{equation*}

To specify the projected update, define the cost-based survival ratio
\begin{equation}
 p_i:=\frac{h_i}{h_i+b_i-c_i},
 \qquad i\in[N].
 \label{eq:critical-ratio}
\end{equation}
Define the capped aggregate resource rate by
\(r^{\mathrm c}:=\min\{r,\Dsum\}\).  Independently of \(r\), for each store
define the fixed safe threshold cap
\begin{equation}
 \bar X_i:=\bar d_i-\frac{p_i}{2K_i},
 \qquad i\in[N].
 \label{eq:fixed-safe-cap}
\end{equation}
The algorithm combines \(\bar X_i\) with the resource-dependent bound through
\(U_i(r)\).  With the dual cap
\(\lambda_{\max}:=\max_{i\in[N]}(b_i-c_i)\), the moving Primal-Dual
projection set is
\begin{equation}
 U_i(r):=\min\left\{\bar X_i,\frac{r^{\mathrm c}}{p_i}\right\},
 \quad i\in[N],
 \qquad
 \calK(r):=
 \prod_{i=1}^N[0,U_i(r)]\times[0,\lambda_{\max}].
 \label{eq:projection-set}
\end{equation}
Let \(j:=\min\operatorname*{arg\,max}_{i\in[N]}(b_i-c_i)\) be the
maximum-margin anchor, with ties broken by the smallest index.
After implementing \(\bm Y_t\) and observing
\(S_{i,t}:=D_{i,t}\wedge Y_{i,t}\), form
\begin{equation}
 \widehat G_{i,t}
 :=(h_i+b_i-c_i)\1\{S_{i,t}<Y_{i,t}\}-(b_i-c_i)
   +\lambda_t\1\{S_{i,t}=Y_{i,t}\},
 \qquad i\in[N],
 \label{eq:observable-samples}
\end{equation}
and
\begin{equation}
 \widehat G_{0,t}
 :=r_t^{\mathrm c}-\sum_{i=1}^NS_{i,t}
   +\theta\widehat G_{j,t},
 \qquad
 \widehat G_t
 :=(\widehat G_{1,t},\ldots,\widehat G_{N,t},\widehat G_{0,t}).
 \label{eq:stochastic-field}
\end{equation}
Here \(\theta>0\) is the weight assigned to the anchor correction in the
dual update. Its stabilizing role is discussed in
Section~\ref{sec:stochastic-step}.
Finally, use the two-sided harmonic step
\begin{equation}
 \alpha_t:=\frac{\chi}{\tau_0+\min\{t,n_t\}}
 =\frac{\chi}{\tau_0+\min\{t,T-t+1\}}.
 \label{eq:two-sided-stepsize}
\end{equation}

\begin{algorithm}[h]
\caption{Resource-Adaptive Primal-Dual Learning (RAPDL) for OWMS }
\label{alg:rapdl}
\small
\begin{algorithmic}[1]
\State \textbf{Initialize}
\(B_1=C_1:=\gamma T\), \(n_1:=T\), \(r_1:=\gamma\),
\(r_1^{\mathrm c}:=\min\{\gamma,\Dsum\}\) and
\(I_{i,1}=X_{i,1}:=0\) for all \(i\), and
\(\lambda_1:=0\).
\For{\(t=1,\ldots,T\)}
\State \textbf{Inventory decision:} compute
\(\nu_t:=\inf\{\nu\ge0:
\sum_{i=1}^N\max\{I_{i,t},X_{i,t}-\nu\}\le C_t\}\) and set
\(Y_{i,t}:=\max\{I_{i,t},X_{i,t}-\nu_t\}\) for all \(i\).
\State \textbf{Gradient estimation:} observe
\(S_{i,t}=D_{i,t}\wedge Y_{i,t}\) and compute, for all \(i\in[N]\),
\[
\widehat G_{i,t}
:=(h_i+b_i-c_i)\1\{S_{i,t}<Y_{i,t}\}-(b_i-c_i)
 +\lambda_t\1\{S_{i,t}=Y_{i,t}\},
\]
\[
\widehat G_{0,t}
:=r_t^{\mathrm c}-\sum_{i=1}^NS_{i,t}
  +\theta\widehat G_{j,t}.
\]
\State \textbf{Physical-state update:}
\(I_{i,t+1}:=Y_{i,t}-S_{i,t}\),
\(C_{t+1}:=C_t-\sum_{i=1}^NS_{i,t}\), and
\(B_{t+1}:=C_{t+1}-\sum_{i=1}^NI_{i,t+1}\).
\If{\(t<T\)}
\State \textbf{Resource update:}
\(n_{t+1}:=n_t-1\),
\(r_{t+1}:=C_{t+1}/n_{t+1}\), and
\(r_{t+1}^{\mathrm c}:=\min\{r_{t+1},\Dsum\}\).
\State \textbf{Primal-Dual update:}
\[
 \begin{aligned}
 X_{i,t+1}
 &:=\proj_{[0,U_i(r_{t+1})]}
       (X_{i,t}-\alpha_t\widehat G_{i,t}),
 &&i\in[N],\\
 \lambda_{t+1}
 &:=\proj_{[0,\lambda_{\max}]}
       (\lambda_t-\alpha_t\widehat G_{0,t}).
 \end{aligned}
\]
\Statex {\footnotesize The projection endpoints are given in
\eqref{eq:projection-set}, and \(\alpha_t\) is given in
\eqref{eq:two-sided-stepsize}.}
\EndIf
\EndFor
\end{algorithmic}
\end{algorithm}

\begin{remark}
If \(C_t=0\), then
\(\bm I_t=\bm X_t=\bm Y_t=\bm S_t=\bm0\), while the dual update continues for
\(t<T\).  At \(t=T\), only the terminal physical state is updated, and no
next-period learning state is defined.
\end{remark}

\subsection{\texorpdfstring{Gradient Estimation and Dual Stabilization}{Gradient Estimation and Dual Stabilization}}
\label{sec:stochastic-step}

The recursion above uses censored sales to construct
\(\widehat G_t\).  We first identify the fluid direction estimated at the
implemented action \(\bm Y_t\), and then explain the anchor correction that
stabilizes its dual coordinate.

Write \(\bm m(\bm y):=(m_i(y_i))_{i=1}^N\).  At resource rate
\(r_t\), define the threshold-space fluid Lagrangian
and its derivatives by
\begin{equation}
 \begin{aligned}
 \widetilde{\mathcal L}_{r_t}(\bm y,\lambda)
 &:=\sum_{i=1}^N\ell_i(y_i)
   +\lambda\left(\sum_{i=1}^Nm_i(y_i)-r_t\right),\\
 \partial_{y_i}\widetilde{\mathcal L}_{r_t}(\bm y,\lambda)
 &=(h_i+b_i-c_i)F_i(y_i)-(b_i-c_i)
   +\lambda\bigl(1-F_i(y_i)\bigr),
 &&i\in[N],\\[2pt]
 \partial_\lambda\widetilde{\mathcal L}_{r_t}(\bm y,\lambda)
 &=\sum_{i=1}^Nm_i(y_i)-r_t.
 \end{aligned}
 \label{eq:threshold-field}
\end{equation}

Although this field is not computable because \(F_i\) and \(m_i\) are
unknown, the algorithm samples it at the implemented action as follows. 

Because \(D_{i,t}<Y_{i,t}\) iff \(S_{i,t}<Y_{i,t}\), censored sales
make \eqref{eq:observable-samples} observable and satisfy\footnote{Throughout this section and the regret analysis, write
\(\calH_t:=\calH_t^{\RAPDL}\).}
\[
 \E[\widehat G_{i,t}\mid\calH_t]
 =\partial_{y_i}\widetilde{\mathcal L}_{r_t}(\bm Y_t,\lambda_t),
 \qquad
 \E\!\left[r_t-\sum_iS_{i,t}\mid\calH_t\right]
 =-\partial_\lambda\widetilde{\mathcal L}_{r_t}(\bm Y_t,\lambda_t).
\]

These identities establish the gradient-estimation component, but the
uncorrected joint direction still lacks a restoring force in the dual
coordinate.  This motivates the following stabilization.

\textbf{The anchor correction.} The sales-space primal Hessian is uniformly positive on the safe region, but
the joint Lagrangian has zero curvature in the dual coordinate.  Without an
additional restoring term, the usual \(O(1/t)\) mean-square contraction is
unavailable.
The generic \(O(t^{-1/2})\) tracking scale would accumulate to
\(O(\sqrt T)\), rather than \(O(\log T)\).  The anchor borrows curvature
from one primal coordinate \citep{benziEtAl2005,latafatEtAl2019}.  For the
maximum-margin anchor \(j\), its KKT residual vanishes along the selected target
path, including at zero resource, so the correction stabilizes the dual
direction without moving the target.  Specifically, use
\begin{equation*}
 \begin{aligned}
 G^\theta(\bm y,\lambda;r)
 &:=
 \begin{pmatrix}
  \nabla_{\bm y}\widetilde{\mathcal L}_r(\bm y,\lambda)\\[2pt]
  r^{\mathrm c}-\displaystyle\sum_{i=1}^Nm_i(y_i)
   +\theta\,\partial_{y_j}\widetilde{\mathcal L}_r(\bm y,\lambda)
 \end{pmatrix}.
 \end{aligned}
\end{equation*}
Its observable version is exactly \eqref{eq:stochastic-field}, so the anchor
requires no additional sample.  In particular,
\[
 \E[\widehat G_t\mid\calH_t]
 =G^\theta(\bm Y_t,\lambda_t;r_t).
\]
Thus feedback estimates the stabilized field at the implemented action
\(\bm Y_t\), while the preferred-state direction is
\(G^\theta(\bm X_t,\lambda_t;r_t)\).  Their difference is the implementation
bias bounded in \eqref{eq:feasibility-bias}.

\subsection{\texorpdfstring{Safeguards and Step-Size Design}{Safeguards and Step-Size Design}}
\label{sec:stabilized-field}

This subsection turns to the remaining
safeguards and step-size design.  Resource capping controls sample magnitude,
safe state clipping preserves curvature, and the two-sided step-size schedule
balances learning across the horizon.

\textbf{Resource capping and a safe clipping box.}
Near the horizon, \(r_t=C_t/n_t\) can be of order \(T\), making both
the uncapped dual sample \(r_t-\sum_{i=1}^NS_{i,t}\) and the resulting
quadratic term in the one-step analysis unbounded in \(T\).

The cap \(r^{\mathrm c}=\min\{r,\Dsum\}\) removes this endpoint explosion.
By construction, \(0\le r^{\mathrm c}\le\Dsum\), and
\(r\mapsto r^{\mathrm c}\) is one-Lipschitz.  Since total one-period sales
also lies in \([0,\Dsum]\), the capped dual sample is bounded in absolute
value by \(\Dsum\).
It leaves the fluid target unchanged because it truncates only the
nonbinding region. See \Cref{lem:capped-rate}.

The fixed caps in \eqref{eq:fixed-safe-cap} keep every preferred
threshold in a region of uniformly positive survival probability without
excluding the fluid target.  Indeed, because \(\lambda^*(r)\ge0\),
\[
 y_i^*(r)\le F_i^{-1}(1-p_i)
 \le\bar d_i-\frac{p_i}{K_i}<\bar X_i.
\]
For later use, define the associated survival-probability floor
\(\beta_i:=\kappa_i p_i/(2K_i)\), \(i\in[N]\).  Then, for every
\(0\le x\le\bar X_i\),
\[
 1-F_i(x)\ge\kappa_i(\bar d_i-x)\ge\beta_i.
\]

The second clipping bound responds to the remaining resource.  Recall from
\eqref{eq:critical-ratio} that \(p_i=h_i/(h_i+b_i-c_i)\) is the survival
probability at the zero-price newsvendor threshold.  At an active fluid
target, survival is therefore at least \(p_i\), so
\[
 p_i y_i^*(r)
 \le m_i(y_i^*(r))
 \le r^{\mathrm c},
 \qquad\text{and therefore}\qquad
 y_i^*(r)\le\frac{r^{\mathrm c}}{p_i}.
\]
If the store is inactive, \(y_i^*(r)=0\), so the same bound holds.
This gives precisely the cap \(U_i(r)\) and projection box
\(\calK(r)\) defined in \eqref{eq:projection-set}.
The endpoint \(\bar X_i\) protects the survival and curvature geometry,
whereas \(r^{\mathrm c}/p_i\) makes the box contract with the remaining
resource.  Taking their minimum enforces both requirements.  Finally, \(\lambda_{\max}\) is sufficient because every
store's fluid best response is zero at any resource price at least
\(\lambda_{\max}\).

\textbf{Step-size choice.}
The scale \(\chi\) in \eqref{eq:two-sided-stepsize} controls
responsiveness, while \(\tau_0\) prevents
the larger endpoint updates from being too aggressive.  Thus the two-sided
schedule permits faster learning initially and faster resource adjustment
near the end, with smaller updates in the middle.  The constructive choice
in \Cref{rem:computable-tuning} is horizon independent and certifies the
regret guarantee.  In practice, \((\theta,\chi,\tau_0)\) can instead be selected offline by a coarse grid search or rolling-origin validation for the intended operating instance.
Appendix~\ref{app:endpoint-envelope} describes how fixed support envelopes
can be used in the caps and outlines the associated localization argument.

\begin{remark}
\label{rem:two-rate-rapdl}
We also consider a two-rate stepsize variant
\[
 \alpha_t^x=\frac{\chi_x}{\tau_0+\min\{t,n_t\}},
 \qquad
 \alpha_t^\lambda=\frac{\chi_\lambda}{\tau_0+\min\{t,n_t\}}.
\]
\begingroup
Appendix~\ref{app:two-rate-rapdl} discusses this two-rate variant and sketches
its theoretical guarantee.
\endgroup

\end{remark}

\section{\texorpdfstring{Main Result and Regret Analysis}{Main Result and Regret Analysis}}
\label{sec:main-results}

The following theorem quantifies the performance of the
resource-adaptive learning framework.

\begin{theorem}
\label{thm:main}
Under \Cref{ass:primitives}, \RAPDL\ with the initialization in
\Cref{alg:rapdl} admits horizon-independent tuning based only on the known
primitives such that, for every \(T\ge2\), every
\(0\le\gamma\le\Dsum\), and every admissible within-period joint demand law,
\begin{equation*}
 \begin{aligned}
 &0\le J_T^{\RAPDL}-J_T^*
 \le R_T(\RAPDL)\le C\log T,
 \end{aligned}
\end{equation*}
where \(C<\infty\) depends only on the fixed model primitives and the selected
horizon-independent tuning, and not on \(T\), \(\gamma\).
\end{theorem}

\begingroup

\Cref{thm:main} concerns the direct safe-box formulation.  As a separate
extension, Appendix~\ref{app:endpoint-envelope} sketches how the analysis can
be adapted when exact endpoints are replaced by fixed deterministic support
envelopes.
\endgroup

\begin{remark}
\label{rem:computable-tuning}
The tuning in \Cref{thm:main} can be fixed offline by a finite deterministic
construction using only \(N\), the costs, and the known support and density
bounds.  It does not require the horizon, the demand distributions, unknown
quantiles, or any Primal-Dual optimizer.  The construction is given in
\Cref{app:tuning}.
\end{remark}

\subsection{\texorpdfstring{High-Level Idea of Regret Analysis}{High-Level Idea of Regret Analysis}}
\label{sec:analysis}

At a high level, regret consists of three distinct logarithmic
components.  

\emph{First, the re-solving gap.}  This is the intrinsic gap
between the resource-indexed fluid re-solving path and the fluid benchmark:
even if the policy could follow this path exactly, movements in the
remaining-resource rate would generate an \(O(\log T)\) cumulative gap.

\emph{Second, the learning gap.}  Because the demand laws are unknown, the
Primal-Dual re-solving path cannot be computed directly.  \RAPDL\ instead
learns the path by tracking it from censored sales, and the learning analysis
shows that the resulting cumulative tracking gap is \(O(\log T)\).

\emph{Third, the implementation gap.}  The
preferred actions produced by the learning recursion may differ from the
actions that can be implemented under the current warehouse and store
inventories. The physical implementation analysis shows that this cumulative
discrepancy is also
\(O(\log T)\).  Adding the re-solving, learning, and implementation gaps
yields the overall \(O(\log T)\) regret bound.

\subsection{\texorpdfstring{Analytical Setup in Expected-Sales Space}{Analytical Setup in Expected-Sales Space}}
\label{sec:fluid-geometry}

Direct analysis in threshold space does not provide the uniform
primal curvature required by the tracking argument.  Although \(\ell_i\) is
strongly convex, the nonlinear resource term changes the curvature of the
threshold-space Lagrangian to
\[
 \frac{\partial^2}{\partial y_i^2}
 \bigl(\ell_i(y_i)+\lambda m_i(y_i)\bigr)
 =(h_i+b_i-c_i-\lambda)f_i(y_i),
\]
which need not be uniformly positive.  Setting \(s_i=m_i(y_i)\) makes the
resource term affine, so the shadow price no longer erodes the primal
curvature.  This change of variables is purely analytical. The executable
policy remains in threshold space.  Since
\(m_i\) is strictly increasing on \([0,\bar d_i)\), define
\[
 s_i:=m_i(y_i),\qquad
 g_i(s):=\ell_i(m_i^{-1}(s)),\qquad 0\le s<\mu_i,
\]
and set \(g_i(\mu_i):=\ell_i(\bar d_i)\) by continuity.  Direct differentiation gives
\(g_i'(0)=-(b_i-c_i)\) and, whenever
\(1-F_i(m_i^{-1}(s))\ge\beta>0\),
\begin{equation}
 h_i\kappa_i\le g_i''(s)\le h_iK_i/\beta^3.
 \label{eq:g-curvature-bounds}
\end{equation}
Thus \(g_i\) is strongly convex, and the fluid program becomes
\[
 v(r)=\min_{0\le s_i\le\mu_i}
 \left\{\sum_{i=1}^Ng_i(s_i):\sum_{i=1}^Ns_i\le r\right\}.
\]
Its resource-indexed Lagrangian is
\[
 \mathcal L_r(\bm s,\lambda)
 :=\sum_{i=1}^Ng_i(s_i)
 +\lambda\left(\sum_{i=1}^Ns_i-r\right).
\]

Let
\[
 r_0:=\sum_{i=1}^N
 m_i\!\left(F_i^{-1}(1-p_i)\right)
\]
be the expected sales of the unconstrained zero-price solution.  The next
lemma collects exactly the path properties used in the regret analysis.

\begin{lemma}
\label{lem:kkt-path}
For every \(r\ge0\), the transformed fluid program has a unique primal
optimizer \(\bm s^*(r)\), with \(s_i^*(r)<\mu_i\), and a selected smallest
optimal multiplier \(\lambda^*(r)\).  The following properties hold,
including when several stores have identical margins.
\begin{enumerate}[label=\textup{(\roman*)},leftmargin=*]
\item The KKT residuals are
\begin{equation}
 \begin{aligned}
 g_i'(s_i^*(r))+\lambda^*(r)&\ge0,&
 s_i^*(r)\bigl(g_i'(s_i^*(r))+\lambda^*(r)\bigr)&=0,
 &&i\in[N],\\
 r-\sum_{i=1}^Ns_i^*(r)&\ge0,&
 \lambda^*(r)\left(r-\sum_{i=1}^Ns_i^*(r)\right)&=0.
 \end{aligned}
 \label{eq:KKT-residuals}
\end{equation}
In particular, store \(i\) is active if and only if
\(\lambda^*(r)<b_i-c_i\). On an active coordinate,
\(-g_i'(s_i^*(r))=\lambda^*(r)\).

\item The resource-clearing branches satisfy
\[
 \sum_{i=1}^Ns_i^*(r)=\min\{r,r_0\},\qquad
 (\bm s^*(0),\lambda^*(0))=(\bm0,\lambda_{\max}),\qquad
 \lambda^*(r)=0\quad(r\ge r_0).
\]

\item There is a primitive-computable constant \(L_v<\infty\) such that
\[
 |\lambda^*(r')-\lambda^*(r)|\le L_v|r'-r|,
 \qquad r,r'\ge0.
\]
Moreover, \(q^*(r):=(\bm s^*(r),\lambda^*(r))\) is globally Lipschitz.
Each \(s_i^*(r)\) is nondecreasing, whereas \(\lambda^*(r)\) is
nonincreasing.  The value function is continuously differentiable, with
\[
 v'(r)=-\lambda^*(r),\qquad
 |v'(r')-v'(r)|\le L_v|r'-r|,
\]
where the derivative at zero is the right derivative.
\end{enumerate}
\end{lemma}

The threshold target introduced in Section~\ref{sec:model} is the inverse
representation
\[
 y_i^*(r)=m_i^{-1}(s_i^*(r)),\qquad i\in[N].
\]
Thus \((\bm y^*(r),\lambda^*(r))\) describes the target tracked by the
policy, whereas \(\bm s^*(r)\) and \(g_i\) are used only in the analysis.

\subsection{\texorpdfstring{Step 1: Bounding the Re-Solving Gap}{Step 1: Bounding the Re-Solving Gap}}
\label{sec:resolving-gap}

The first component compares the resource-indexed fluid re-solving path with
the fluid benchmark.  To make this comparison, we work in expected-sales
space, where three vectors play different roles:
\(\bm m(\bm Y_t)\) is conditional expected sales under the implemented action,
\(\bm m(\bm X_t)\) is the expected-sales image of the preferred learning
state, and \(\bm s^*(r_t)\) is the unobserved fluid target.
Throughout the regret analysis, we use the pathwise safe-box invariant
$ 0\le X_{i,t},I_{i,t},Y_{i,t}\le\bar X_i, i\in[N],\ t\in[T],$
established from the algorithmic recursion in
\eqref{eq:appendix-physical-safe-box}.

We collect the learning and implementation gaps, respectively, as
\begin{align}
 \mathsf{PD}_T
 &:=
 \sum_{t=1}^T
 \E\!\left[
   \|\bm m(\bm X_t)-\bm s^*(r_t)\|^2
   +\left\langle
    \bm m(\bm X_t)-\bm s^*(r_t),
    \nabla_{\bm s}\mathcal L_{r_t}
    (\bm s^*(r_t),\lambda^*(r_t))
    \right\rangle
 \right],
 \label{eq:learning-gap-block}\\
 \mathsf{IMP}_T
 &:=
 \sum_{t=1}^T\E\|\bm Y_t-\bm X_t\|_1
 +\sum_{t=1}^T\E\|\bm Y_t-\bm X_t\|^2.
 \label{eq:implementation-gap-block}
\end{align}
Here \(\mathsf{PD}_T\) contains the squared tracking error and the first-order
KKT residual that enter the operational gap, whereas \(\mathsf{IMP}_T\)
contains the two implementation norms needed by the tracking and regret
arguments.

The following gap bound is the deterministic bridge from fluid geometry to
regret.  Set
\begin{equation*}
 L_g:=\max_{1\le i\le N}\frac{h_iK_i}{\beta_i^3}.
\end{equation*}

For \(r\ge0\) and any \(\bm s\) satisfying
\(0\le m_i^{-1}(s_i)\le\bar X_i\), define
\begin{equation}
 \begin{aligned}
 \Delta(\bm s;r)
 :={\mathcal L}_r(\bm s,\lambda^*(r))
   -{\mathcal L}_r(\bm s^*(r),\lambda^*(r))=\sum_{i=1}^Ng_i(s_i)-v(r)
   +\lambda^*(r)\left(\sum_{i=1}^Ns_i-r\right).
 \end{aligned}
 \label{eq:sales-gap}
\end{equation}
Because \(\bm s^*(r)\) minimizes
\({\mathcal L}_r(\cdot,\lambda^*(r))\), this gap is nonnegative.  Moreover,
\eqref{eq:g-curvature-bounds} and the bound
\(1-F_i(x)\ge\beta_i\) on \([0,\bar X_i]\) make the
Lagrangian \(L_g\)-smooth on the segment from \(\bm s^*(r)\) to \(\bm s\).
Thus the descent lemma, followed by the KKT complementarity identities in
\eqref{eq:KKT-residuals}, gives directly
\begin{align}
 0\le\Delta(\bm s;r)
 &\le\frac{L_g}{2}\|\bm s-\bm s^*(r)\|^2
 +\left\langle
 \bm s-\bm s^*(r),
 \nabla_{\bm s}{\mathcal L}_r(\bm s^*(r),\lambda^*(r))
 \right\rangle\notag\\
 &=\frac{L_g}{2}\|\bm s-\bm s^*(r)\|^2
 +\sum_{i=1}^Ns_i
 \bigl(g_i'(s_i^*(r))+\lambda^*(r)\bigr).
 \label{eq:gap-residual-bound}
\end{align}
For a threshold vector \(\bm y\), writing
\(\bm m(\bm y):=(m_i(y_i))_{i=1}^N\),
substitution of \(s_i=m_i(y_i)\) and
\(g_i(m_i(y_i))=\ell_i(y_i)\) into \eqref{eq:sales-gap} gives
\begin{equation}
 \Delta(\bm m(\bm y);r)
 =\sum_{i=1}^N\ell_i(y_i)-v(r)
 +\lambda^*(r)\left(\sum_{i=1}^Nm_i(y_i)-r\right).
 \label{eq:operational-gap}
\end{equation}

Fix \(t<T\).  The aggregate inventory recursion gives
\begin{equation}
 r_{t+1}-r_t
 =
 \frac{r_t-\sum_{i=1}^N S_{i,t}}{n_t-1},
 \qquad
 \E[r_{t+1}-r_t\mid\calH_t]
 =
 \frac{r_t-\sum_{i=1}^N m_i(Y_{i,t})}{n_t-1}.
 \label{eq:master-resource-drift}
\end{equation}
The second identity uses that \(\bm Y_t\) is chosen before current demand and
\(\E[S_{i,t}\mid\calH_t]=m_i(Y_{i,t})\).

By \Cref{lem:kkt-path}\textup{(iii)}, the fluid value is differentiable and
its derivative is Lipschitz:
\begin{equation*}
 v'(r)=-\lambda^*(r),\qquad
 |v'(r')-v'(r)|\le L_v|r'-r|,
\end{equation*}
with the right derivative understood at \(r=0\).

Suppose first that \(r_t\le\Dsum\).  Since both \(r_t\) and total realized
sales belong to \([0,\Dsum]\),
\[
 |r_{t+1}-r_t|\le\frac{\Dsum}{n_t-1}.
\]
By \(L_v\)-smoothness, with \(r=r_t\) and \(r'=r_{t+1}\), and then using
\(v'(r_t)=-\lambda^*(r_t)\), gives
\[
 v(r_{t+1})
 \le
 v(r_t)-\lambda^*(r_t)(r_{t+1}-r_t)
 +\frac{L_v}{2}(r_{t+1}-r_t)^2.
\]
Taking the conditional expectation, multiplying by \(n_t-1\), and using
\eqref{eq:master-resource-drift} and
\(|r_{t+1}-r_t|\le\Dsum/(n_t-1)\) give
\begin{align*}
 (n_t-1)\E[v(r_{t+1})\mid\calH_t]\le
 (n_t-1)v(r_t)
 +\lambda^*(r_t)
  \left(\sum_{i=1}^Nm_i(Y_{i,t})-r_t\right)
 +\frac{L_v\Dsum^2}{2(n_t-1)}.
\end{align*}
Adding \(\sum_{i=1}^N\ell_i(Y_{i,t})-n_tv(r_t)\) to both sides and using
\eqref{eq:operational-gap} with \(\bm y=\bm Y_t\) gives
\begin{align}
 &\sum_{i=1}^N\ell_i(Y_{i,t})
 +(n_t-1)\E[v(r_{t+1})\mid\calH_t]-n_tv(r_t)\notag\\
 &\qquad\le
 \sum_{i=1}^N\ell_i(Y_{i,t})-v(r_t)
 +\lambda^*(r_t)
  \left(\sum_{i=1}^Nm_i(Y_{i,t})-r_t\right)
 +\frac{L_v\Dsum^2}{2(n_t-1)}\notag\\
 &\qquad=
 \Delta(\bm m(\bm Y_t);r_t)
 +\frac{L_v\Dsum^2}{2(n_t-1)}.
 \label{eq:master-one-step-bellman}
\end{align}
If \(r_t>\Dsum\), then
\(\sum_{i=1}^N S_{i,t}\le\Dsum<r_t\), hence \(r_{t+1}>r_t\).
Both \(r_t\) and \(r_{t+1}\) lie in the slack-resource region, where
\(v\) is constant and \(\lambda^*=0\).  Thus
\eqref{eq:master-one-step-bellman} remains valid, in fact with zero
remainder.

Taking expectations and summing \eqref{eq:master-one-step-bellman} from
\(t=1\) to \(T-1\) telescopes the continuation values:
\begin{align*}
 \E\sum_{t=1}^{T-1}\sum_{i=1}^N\ell_i(Y_{i,t})
 +\E v(r_T)-Tv(\gamma)\le
 \E\sum_{t=1}^{T-1}\Delta(\bm m(\bm Y_t);r_t)
 +\frac{L_v\Dsum^2}{2}
 \sum_{t=1}^{T-1}\frac{1}{n_t-1}.
\end{align*}
Because \(n_t-1=T-t\), the last sum---the cumulative discrepancy generated by
movements along the re-solving path---is at most \(C\log T\).
Moreover, \(v(r_T)\ge0\), and the safe box
\(0\le Y_{i,T}\le\bar X_i\) gives a primitive upper bound on
\(\sum_{i=1}^N\ell_i(Y_{i,T})-v(r_T)\).  Adding the last period and using
\(\Delta(\bm m(\bm Y_t);r_t)\ge0\) therefore yields
\begin{equation}
 \E\sum_{t=1}^T\sum_{i=1}^N\ell_i(Y_{i,t})-Tv(\gamma)
 \le
 \E\sum_{t=1}^T\Delta(\bm m(\bm Y_t);r_t)
 +C\log T.
 \label{eq:master-gap-reduction}
\end{equation}
By the exact cost reduction \eqref{eq:exact-cost} and the regret
definition \eqref{eq:regret-definition}, the left side of
\eqref{eq:master-gap-reduction} is \(R_T(\RAPDL)\).

\textbf{Separating the learning and implementation gaps.}
The function \(m_i\) is
\(1\)-Lipschitz, and
\[
 |\ell_i'(y)|
 \le \max\{h_i,b_i-c_i\}
 \qquad (0\le y\le\bar X_i).
\]
Together with \(0\le\lambda^*(r_t)\le\lambda_{\max}\), these
Lipschitz bounds give
\[
 \Delta(\bm m(\bm Y_t);r_t)
 \le \Delta(\bm m(\bm X_t);r_t)
 +C\|\bm Y_t-\bm X_t\|_1.
\]
Applying \eqref{eq:gap-residual-bound} to the preferred action then yields
\begin{align}
 \Delta(\bm m(\bm Y_t);r_t)
 \le C\Bigl(
  \|\bm m(\bm X_t)-\bm s^*(r_t)\|^2
  +\left\langle
   \bm m(\bm X_t)-\bm s^*(r_t),
   \nabla_{\bm s}\mathcal L_{r_t}
   (\bm s^*(r_t),\lambda^*(r_t))
   \right\rangle
  +\|\bm Y_t-\bm X_t\|_1
 \Bigr).
 \label{eq:master-gap-comparison}
\end{align}

Taking expectations in \eqref{eq:master-gap-comparison}, summing, and
using \eqref{eq:learning-gap-block}--\eqref{eq:implementation-gap-block} gives
\begin{equation}
 \E\sum_{t=1}^T\Delta(\bm m(\bm Y_t);r_t)
 \le C(\mathsf{PD}_T+\mathsf{IMP}_T).
 \label{eq:master-gap-blocks}
\end{equation}
Substituting \eqref{eq:master-gap-blocks} into
\eqref{eq:master-gap-reduction} and absorbing fixed constants into
\(C\) yields the three-component regret reduction
\begin{equation}
 R_T(\RAPDL)
 \le
 C\left(\log T
 +
   \mathsf{PD}_T
   +\mathsf{IMP}_T
 \right).
 \label{eq:three-component-regret-bound}
\end{equation}

\subsection{\texorpdfstring{Step 2: Bounding the Learning Gap}{Step 2: Bounding the Learning Gap}}
\label{sec:learning-gap}

Under \eqref{eq:tuning}, we start from the definition in
\eqref{eq:learning-gap-block}.  Since
\[
 \nabla_{\bm s}\mathcal L_{r_t}
 (\bm s^*(r_t),\lambda^*(r_t))
 =
 \bigl(g_i'(s_i^*(r_t))+\lambda^*(r_t)\bigr)_{i=1}^N,
\]
the storewise complementarity conditions
\eqref{eq:KKT-residuals} give
\begin{align}
 \mathsf{PD}_T
 ={}&
 \sum_{t=1}^T
 \E\|\bm m(\bm X_t)-\bm s^*(r_t)\|^2+\sum_{t=1}^T
 \E\sum_{i=1}^N
 \bigl(m_i(X_{i,t})-s_i^*(r_t)\bigr)
 \bigl(g_i'(s_i^*(r_t))+\lambda^*(r_t)\bigr)\notag\\
 ={}&
 \sum_{t=1}^T
 \E\|\bm m(\bm X_t)-\bm s^*(r_t)\|^2+\sum_{t=1}^T
 \E\sum_{i=1}^N
 m_i(X_{i,t})
 \bigl(g_i'(s_i^*(r_t))+\lambda^*(r_t)\bigr).
 \label{eq:pd-expanded}
\end{align}
We now invoke \Cref{lem:technical-pd}.  Its tracking conclusion
\eqref{eq:cumulative-tracking}, after dropping the nonnegative squared dual
error, gives
\begin{equation}
 \sum_{t=1}^T
 \E\|\bm m(\bm X_t)-\bm s^*(r_t)\|^2
 \le
 C\log T
 +C\sum_{t=1}^T
 \E\|\bm Y_t-\bm X_t\|^2.
 \label{eq:pd-tracking-from-technical}
\end{equation}
Its residual conclusion \eqref{eq:cumulative-residual} gives
\begin{align}
\sum_{t=1}^T
 \E\!\left[
  \sum_{i=1}^N
  m_i(X_{i,t})
  \bigl(g_i'(s_i^*(r_t))+\lambda^*(r_t)\bigr)
 \right]\le
 C\log T
 +C\sum_{t=1}^T
 \E\|\bm Y_t-\bm X_t\|^2.
 \label{eq:pd-residual-from-technical}
\end{align}
Thus \eqref{eq:pd-residual-from-technical} directly bounds the second
sum in \eqref{eq:pd-expanded}.  Adding it to
\eqref{eq:pd-tracking-from-technical}, and using
\[
 \sum_{t=1}^T\E\|\bm Y_t-\bm X_t\|^2
 \le\mathsf{IMP}_T,
\]
we obtain the learning-gap bound
\begin{equation}
 \mathsf{PD}_T
 \le
 C\left(\log T+\mathsf{IMP}_T\right).
 \label{eq:learning-gap-bound}
\end{equation}

\subsection{\texorpdfstring{Step 3: Bounding the Implementation Gap}{Step 3: Bounding the Implementation Gap}}
\label{sec:implementation-gap}

The implementation gap arises solely from physical feasibility: the implemented
action must satisfy \(Y_{i,t}\ge I_{i,t}\) for every store and
\(\sum_iY_{i,t}\le C_t\), whereas the preferred action \(\bm X_t\) need not
satisfy these two restrictions.

The following physical bound controls this component.

\begin{lemma}
\label{lem:implementation-gap}
Under the conditions of \Cref{thm:main}, \RAPDL\ satisfies
\begin{equation}
 \mathsf{IMP}_T\le C\log T.
 \label{eq:implementation-gap-bound}
\end{equation}
\end{lemma}

Combining the re-solving reduction
\eqref{eq:three-component-regret-bound}, the learning bound
\eqref{eq:learning-gap-bound}, and the implementation bound
\eqref{eq:implementation-gap-bound} gives directly
\begin{align*}
 \mathsf{IMP}_T&\le C\log T,\qquad
 \mathsf{PD}_T
 \le C\bigl(\log T+\mathsf{IMP}_T\bigr)
 \le C\log T,\\
 R_T(\RAPDL)
 &\le C\bigl(\log T+\mathsf{PD}_T+\mathsf{IMP}_T\bigr)
 \le C\log T.
\end{align*}
This completes the proof of \Cref{thm:main}.\hfill\(\square\)

\section{Numerical Experiments}
\label{sec:numerical-experiments}

\textbf{Experimental design.}
The numerical study evaluates the finite-horizon performance of
\RAPDL\ across horizon lengths and inventory regimes and compares it with the
Double Binary Search (DBS) algorithm of \citet{bekci2023}, the closest
directly comparable learning benchmark.  We use matched demand paths,
initialization, tuning budgets, and feasibility constraints.

We use
\[
 \rho\in\{0.30,0.60,0.90,1.20\},\qquad W=T\rho\gamma^\star,
\]
where \(\gamma^\star\) is unconstrained optimal expected sales.  These values
represent four distinct operating regimes---severe scarcity, moderate
scarcity, near sufficiency, and abundant inventory. In S2 they also cross
fluid active-set sizes three, five, and all six stores.  Total Cost is
reported for \(T\in\{100,200,300\}\), while the figures use
\(T=20,40,\ldots,200\), each with newly initialized inventory.  In both
settings the demand vectors are independent and identically distributed over
time.  Stores are independent within a period in S1, whereas S2 introduces
the contemporaneous dependence specified below.  All calculations use full
precision.

The experimental label \RAPDL\ denotes the two-rate
variant in \Cref{rem:two-rate-rapdl}.  The label DBS denotes Double Binary
Search in \citet{bekci2023}.  Because S1 is homogeneous and satisfies their
Assumption~2 by symmetry, its benchmark is precisely their DBS policy in
Algorithm~2.  S2 is heterogeneous and may have stores outside the fluid
active set.  Its benchmark is therefore their heterogeneous active-set
implementation: Algorithm~2 equipped with the relaxed-Assumption~2 extension
described in their online appendix.  For a fair comparison, both policies use
the same midpoint initialization, \(\lambda_1=0\), demand paths, and
feasibility constraints. For \RAPDL, \(X_{i,1}=\bar d_i/2\).  Each policy is
calibrated offline on independent demand paths for every
\((\text{setting},\rho,T)\) cell under the same tuning budget.  The selected
parameters are then locked and evaluated on untouched paired paths.

\textbf{S1: homogeneous setting.}
There are \(N=2\) identical stores, with
\[
 \bm h=(6.00,6.00),\qquad
 \bm b=(60.00,60.00),\qquad
 \bm c=(0.50,0.50).
\]
For every store and period, demand is a
\(\mathcal N(50.00,50.00^2)\) random variable conditioned to lie in
\([0,175.00]\).  With \(\Phi\) denoting the standard-normal distribution
function, its common distribution function is
\[
F_{\mathrm{S1}}(d)=
\begin{cases}
0, & d<0,\\
\displaystyle
\frac{\Phi((d-50.00)/50.00)-\Phi(-1.00)}
     {\Phi(2.50)-\Phi(-1.00)},
   & 0\le d\le175.00,\\
1, & d>175.00.
\end{cases}
\]
This is the homogeneous synthetic specification used by
\citet{bekci2023}.  Its unconstrained expected-sales normalizer is
\(\gamma^\star=123.43\).

\textbf{S2: heterogeneous setting.}
There are \(N=6\) stores, with store-specific holding, net shipment, and
lost-sales coefficients:
\[
\begin{aligned}
 \bm h&=(2.28,2.46,2.52,2.42,2.34,2.38),\\
 \bm c&=(0.195,0.190,0.205,0.220,0.210,0.180),\\
 \bm b&=(38.40,40.00,41.60,43.20,44.80,46.40).
\end{aligned}
\]
For store \(i\) and each period, demand is
\(\operatorname{Uniform}[0,\bar d_i]\), where
\[
 \bar{\bm d}=(1.18,1.26,1.08,1.32,1.14,1.22).
\]
Its distribution function is
\[
F_{\mathrm{S2},i}(d)=
\begin{cases}
0, & d<0,\\
d/\bar d_i, & 0\le d\le\bar d_i,\\
1, & d>\bar d_i.
\end{cases}
\]
Thus both the cost coefficients and demand laws are store-specific.  The
six Uniform marginals are positively correlated within each period through
a Gaussian copula.  Specifically, let \(U_t\) and \(\varepsilon_{i,t}\) be independent
standard-normal variables and set
\[
 Z_{i,t}=\sqrt{0.60}\,U_t+\sqrt{0.40}\,\varepsilon_{i,t},
 \qquad
 D_{i,t}=\bar d_i\Phi(Z_{i,t}).
\]
A fresh \(U_t\) and fresh idiosyncratic shocks are drawn every period.  Hence
the demand vectors remain independent and identically distributed over time,
while the stores experience a positive common within-period shock.  
The resulting unconstrained expected-sales normalizer
is \(\gamma^\star=3.59\).  This correlated-demand specification is admissible
under \Cref{ass:primitives} and shows that our framework handles cross-store
dependence.  It is also more practical because retail stores may share
weather, holiday, and market shocks.

\textbf{Performance metrics and results.}
For policy \(\pi\), define the realized store--period cost by
\[
 \mathsf{C}_{i,t}^\pi
 :=c_iS_{i,t}^\pi
   +h_i(Y_{i,t}^\pi-D_{i,t})^+
   +b_i(D_{i,t}-Y_{i,t}^\pi)^+,
 \qquad
 S_{i,t}^\pi:=D_{i,t}\wedge Y_{i,t}^\pi .
\]
The terms are net shipment, holding, and lost-sales costs.  Here
\(w=0\), so \(c_i=k_i\). Equation~\eqref{eq:exact-cost} makes this sales-based
accounting equivalent to physical shipment cost net of terminal store credit.

The realized \emph{Total Cost} of one complete \(T\)-period path is
\[
 \mathsf{TC}_T^\pi
 :=\sum_{t=1}^T\sum_{i=1}^N\mathsf{C}_{i,t}^\pi .
\]

For each cell, \Cref{tab:formal-total-cost} reports mean Total Cost
over \(M=2{,}000\) post-tuning paths and the paired
\RAPDL-minus-DBS difference.

The horizon-by-horizon Relative Regret results for S1 and S2 are shown in
\Cref{fig:formal-regret-s1,fig:formal-regret-s2}, respectively.  For these
figures, set \(\gamma:=W/T=\rho\gamma^\star\).  The quantity
\(v(\gamma)\) is the one-period constrained fluid lower-bound value, so
\(Tv(\gamma)\) is the corresponding \(T\)-period benchmark.  We define the
estimated \emph{Relative Regret} by
\[
 \widehat{\operatorname{RR}}_T^\pi
 :=100\times
 \frac{\widehat{\E}[\mathsf{TC}_T^\pi]-Tv(\gamma)}
      {Tv(\gamma)}.
\]

Thus a plotted value of \(2.00\) denotes a mean cost \(2.00\%\) above
the fluid benchmark.

For each \((\text{setting},\rho,T,\pi)\) cell, the figures show 1,000
bootstrap resamples of the mean Relative Regret: boxes span the 25th--75th
percentiles and whiskers the 5th--95th.  These intervals concern the mean
estimator, and each horizon is a newly initialized problem rather than a
prefix of a longer simulation.

\begin{table}[H]
\centering
\footnotesize
\setlength{\tabcolsep}{3pt}
\caption{Mean Total Cost over 2,000 paired production paths.}
\label{tab:formal-total-cost}
\begin{tabular}{ccrrrrrr}
\toprule
& & \multicolumn{3}{c}{S1} & \multicolumn{3}{c}{S2}\\
\cmidrule(lr){3-5}\cmidrule(lr){6-8}
$T$ & $\rho$ & RAPDL & DBS & RAPDL--DBS & RAPDL & DBS & RAPDL--DBS\\
\midrule
100 & 0.30 & 544,099.04 & 549,055.18 & -4,956.15 & 10,565.88 & 10,603.04 & -37.17\\
100 & 0.60 & 333,446.00 & 338,801.06 & -5,355.06 & 6,190.46 & 6,249.52 & -59.05\\
100 & 0.90 & 140,755.01 & 141,484.92 & -729.91 & 2,136.95 & 2,215.54 & -78.59\\
100 & 1.20 & 101,075.81 & 102,271.96 & -1,196.15 & 999.33 & 1,007.25 & -7.92\\
200 & 0.30 & 1,086,641.01 & 1,094,896.22 & -8,255.21 & 21,069.98 & 21,146.08 & -76.09\\
200 & 0.60 & 664,709.73 & 677,088.42 & -12,378.68 & 12,300.85 & 12,462.86 & -162.01\\
200 & 0.90 & 277,786.27 & 278,650.01 & -863.74 & 4,176.88 & 4,341.13 & -164.24\\
200 & 1.20 & 198,100.89 & 201,156.71 & -3,055.82 & 1,895.43 & 1,906.57 & -11.14\\
300 & 0.30 & 1,629,585.18 & 1,638,005.07 & -8,419.89 & 31,545.71 & 31,652.20 & -106.49\\
300 & 0.60 & 991,727.32 & 1,014,562.07 & -22,834.75 & 18,376.37 & 18,686.17 & -309.80\\
300 & 0.90 & 414,885.00 & 415,403.95 & -518.95 & 6,186.55 & 6,411.03 & -224.48\\
300 & 1.20 & 295,226.20 & 299,404.12 & -4,177.92 & 2,790.04 & 2,802.44 & -12.40\\
\bottomrule
\end{tabular}
\end{table}

As shown in \Cref{tab:formal-total-cost}, RAPDL has lower mean Total Cost in
all 24 paired finite-horizon comparisons.

\begin{figure}[p]
\centering
\begin{minipage}{0.44\textwidth}
\centering
\includegraphics[width=\linewidth]{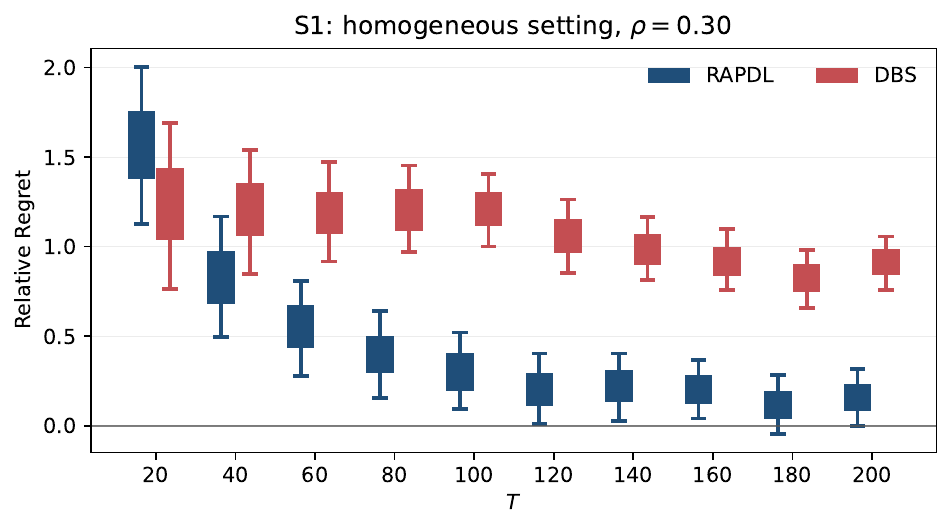}
\small (a) $\rho=0.30$
\end{minipage}\hfill\begin{minipage}{0.44\textwidth}
\centering
\includegraphics[width=\linewidth]{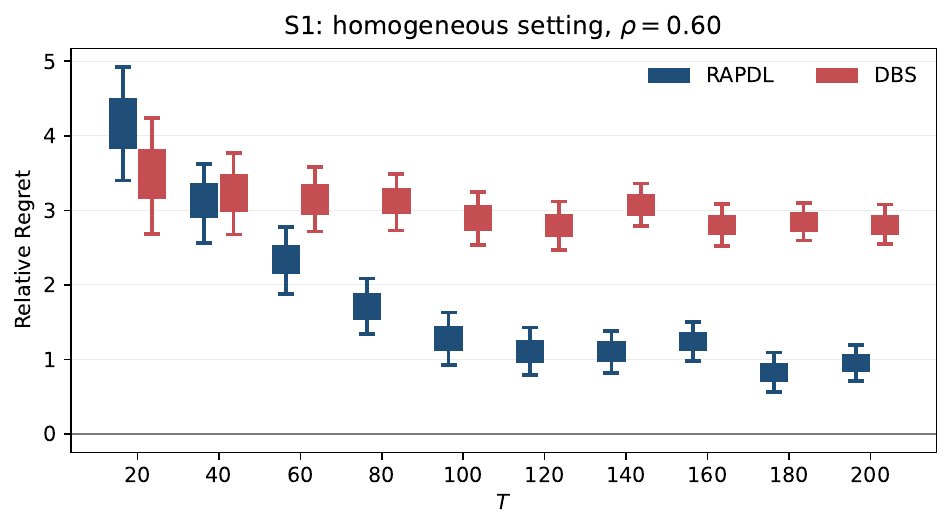}
\small (b) $\rho=0.60$
\end{minipage}
\par\smallskip
\begin{minipage}{0.44\textwidth}
\centering
\includegraphics[width=\linewidth]{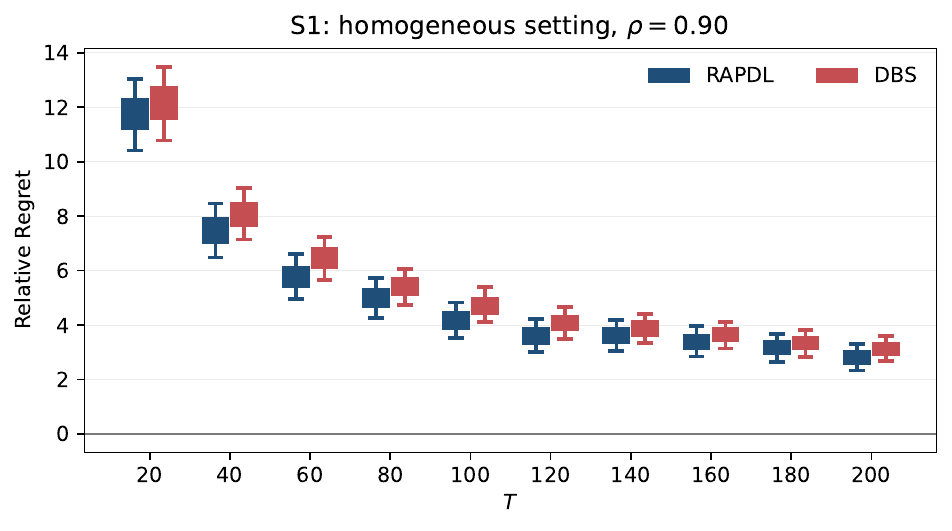}
\small (c) $\rho=0.90$
\end{minipage}\hfill\begin{minipage}{0.44\textwidth}
\centering
\includegraphics[width=\linewidth]{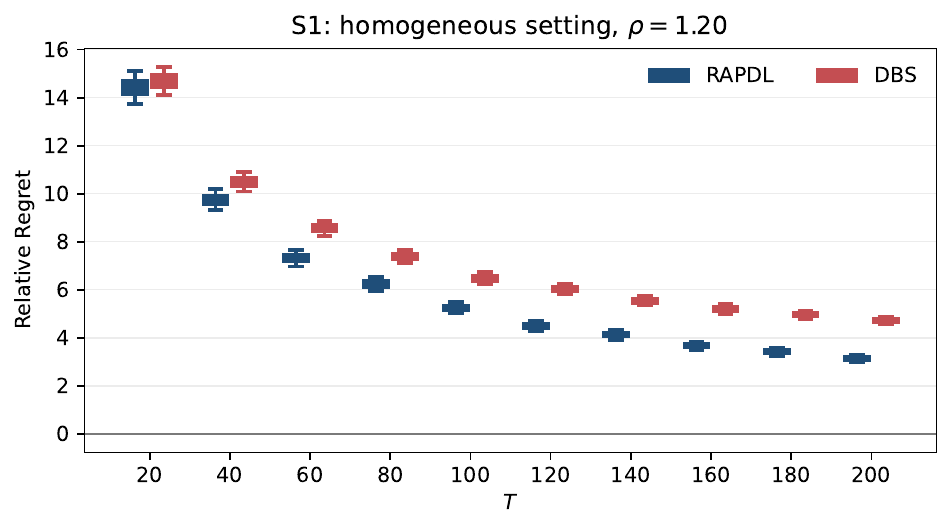}
\small (d) $\rho=1.20$
\end{minipage}
\caption{Relative Regret across complete horizons in S1.}
\label{fig:formal-regret-s1}
\end{figure}

\begin{figure}[p]
\centering
\begin{minipage}{0.44\textwidth}
\centering
\includegraphics[width=\linewidth]{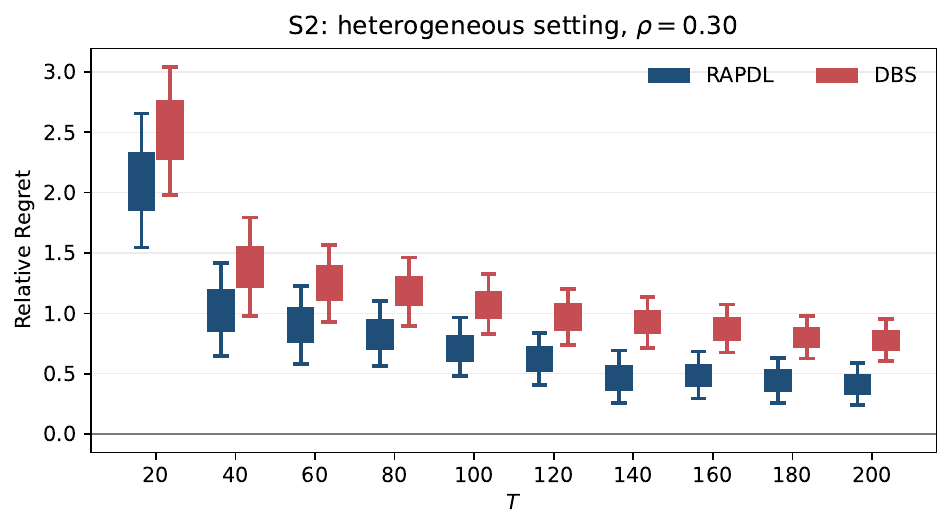}
\small (a) $\rho=0.30$
\end{minipage}\hfill\begin{minipage}{0.44\textwidth}
\centering
\includegraphics[width=\linewidth]{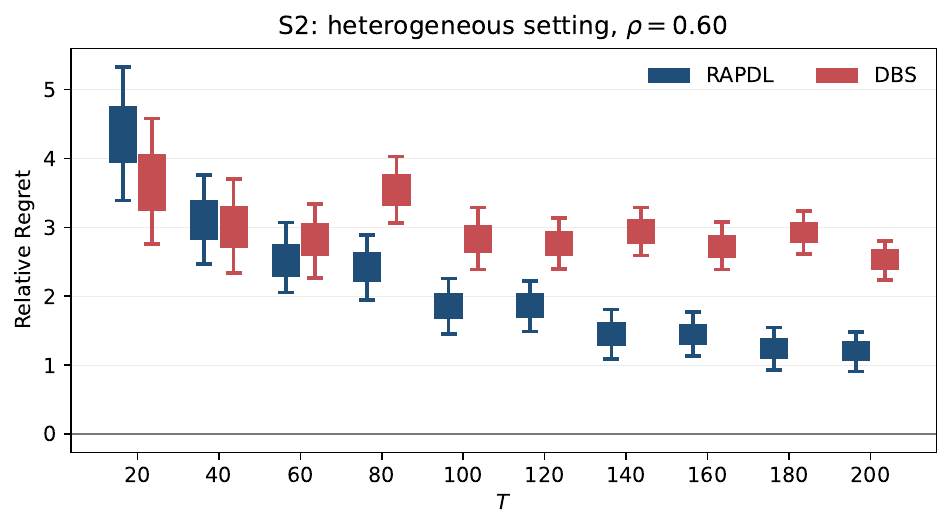}
\small (b) $\rho=0.60$
\end{minipage}
\par\smallskip
\begin{minipage}{0.44\textwidth}
\centering
\includegraphics[width=\linewidth]{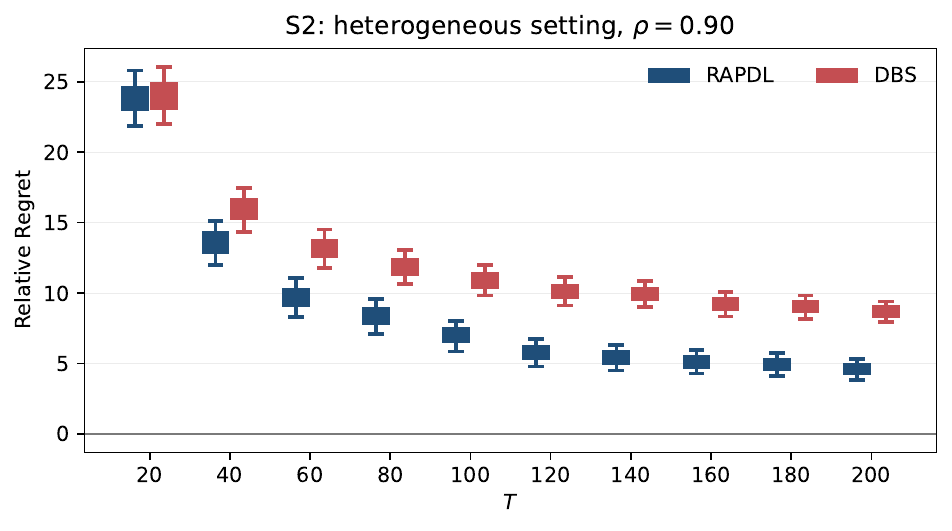}
\small (c) $\rho=0.90$
\end{minipage}\hfill\begin{minipage}{0.44\textwidth}
\centering
\includegraphics[width=\linewidth]{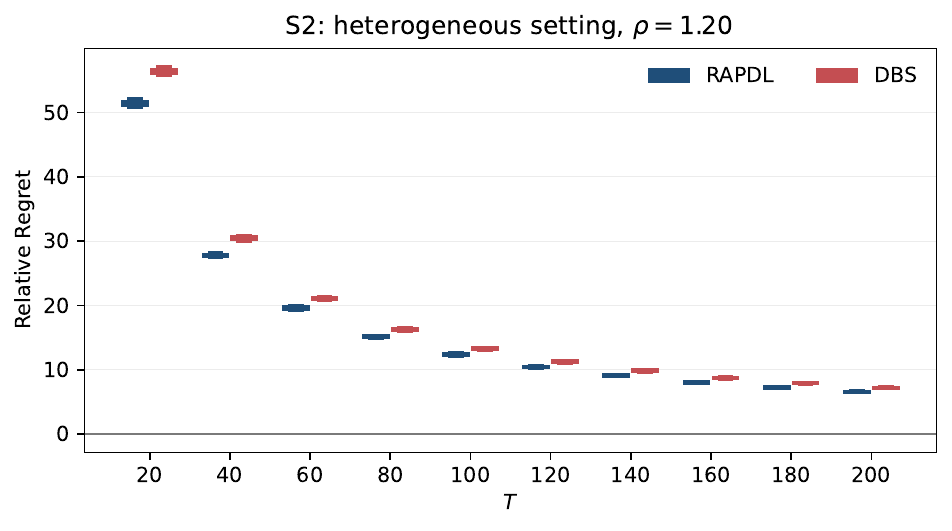}
\small (d) $\rho=1.20$
\end{minipage}
\caption{Relative Regret across complete horizons in S2.}
\label{fig:formal-regret-s2}
\end{figure}

The figures place both methods close to the fluid benchmark at practical
horizons.  At \(T=100\), the bootstrap median Relative Regret of \RAPDL\
ranges from \(0.30\%\) to \(12.38\%\). By \(T=200\), the range narrows to
\(0.15\%\)--\(6.58\%\).  At the short horizons \(T=60\) and \(T=100\), the
bootstrap median for \RAPDL\ is below that for DBS in all eight panels.  The
very earliest
\(T=20\) and \(T=40\) comparisons are mixed, and these outcomes are retained
in the plots.  In the \(T\in\{100,200,300\}\) cells reported in
\Cref{tab:formal-total-cost}, \RAPDL\ has lower mean Total Cost in all 24
paired comparisons.

\textbf{Sensitivity to endpoint information.}
We test the endpoint-envelope implementation in S1 at \(T=200\).  The true demand support
remains \([0,175]\).  The exact implementation uses
\(\bar X_i=170.21\), \(\Dsum=350\),
and \(X_{i,1}=87.5\), whereas the
envelope implementation uses
\(D_i^{\mathrm{up}}=200\),
\(D^{\mathrm{up}}=400\), and \(X_{i,1}=100\).
For each \(\rho\), both information regimes use the same
tuning budget and 2,000 paired production demands.  This comparison evaluates two separately calibrated end-to-end implementations with regime-specific midpoint initializations, rather than isolating the endpoint input under fixed parameters.
The results are reported in \Cref{tab:endpoint-sensitivity}. Its last
column gives the percentage difference
\(100\times
 \bigl(\widehat{\E}[\mathrm{TC}^{\mathrm{Envelope}}_{200}]
       -\widehat{\E}[\mathrm{TC}^{\mathrm{Exact}}_{200}]\bigr)
 /\widehat{\E}[\mathrm{TC}^{\mathrm{Exact}}_{200}]\).

\begin{table}[h]
\centering
\footnotesize
\setlength{\tabcolsep}{5pt}
\caption{Sensitivity of S1 Total Cost to endpoint information at \(T=200\).}
\label{tab:endpoint-sensitivity}
\begin{tabular}{crrr}
\toprule
\(\rho\) & Exact endpoint & Envelope & Envelope--Exact, \%\\
\midrule
0.30 & 1,087,811.49 & 1,088,103.80 & \phantom{-}0.03\\
0.60 &   666,128.51 &   667,172.93 & \phantom{-}0.16\\
0.90 &   279,716.69 &   283,190.98 & \phantom{-}1.24\\
1.20 &   198,172.87 &   197,300.71 &          -0.44\\
\bottomrule
\end{tabular}
\end{table}
Across the four resource levels, the absolute Total Cost difference
between the exact endpoint and the rough envelope is at most \(1.24\%\).
Thus, in these finite-horizon comparisons, replacing the exact endpoint by a
conservative upper bound has only a small numerical effect.

\section{Conclusion}
\label{sec:discussion}

This paper studies learning in a finite-horizon OWMS system with
nonreplenishable shared inventory,
unknown demand, and censored sales.  We
propose \RAPDL, a resource-adaptive Primal-Dual framework that replaces
fixed-target learning with online tracking of
the endogenous
Primal-Dual re-solving path.  Censored sales update the preferred store
thresholds and warehouse scarcity price, while water-filling maps the learned
state to feasible physical actions.  The regret analysis separates the
re-solving, learning, and implementation gaps and bounds each at logarithmic
order.  Numerical
experiments on a practical variant show strong finite-horizon performance
across inventory regimes.

More broadly, the framework turns re-solving from repeated full-information
optimization into a path that can be tracked incrementally under unknown and
censored feedback.  This resource-adaptive perspective may inform the design
of learning-and-control methods for other systems with depleting shared
resources.

\bibliographystyle{informs2014}
\bibliography{ref}
\clearpage
\begin{APPENDICES}
\SingleSpacedXI
\begin{center}
{\Large\bfseries Online Appendix for\\[4pt]
``Resource-Adaptive Primal-Dual Learning for One-Warehouse Multi-Store
Systems with Censored Demand''}
\end{center}
\bigskip
\numberwithin{table}{section}
\section{\texorpdfstring{Technical Bound for the Learning Gap}{Technical Bound for the Learning Gap}}
\label{app:learning-gap}

The following estimate supplies the squared-tracking and first-order residual
bounds used to control the learning gap.  Both follow from the same one-step
mirror recursion developed below.

\begin{lemma}
\label{lem:technical-pd}
Choose \(\chi,\tau_0\) as in \eqref{eq:tuning}.  Then there is a primitive
constant \(C_{\rm md}\) such that
\medskip

\emph{(i) Cumulative Primal-Dual tracking.}
\begin{align}
 \sum_{t=1}^T\E\left[
  \|\bm m(\bm X_t)-\bm s^*(r_t)\|^2
  +|\lambda_t-\lambda^*(r_t)|^2
 \right]
 &\le C_{\rm md}\log T
 +C_{\rm md}\sum_{t=1}^T\E\|\bm Y_t-\bm X_t\|^2,
 \label{eq:cumulative-tracking}
\end{align}

\emph{(ii) Unweighted first-order residual.}
\begin{align}
 \sum_{t=1}^T\E\!\left[
  \sum_{i=1}^N
  m_i(X_{i,t})
  \bigl(g_i'(s_i^*(r_t))+\lambda^*(r_t)\bigr)
 \right]
 &\le C_{\rm md}\log T
 +C_{\rm md}\sum_{t=1}^T\E\|\bm Y_t-\bm X_t\|^2.
 \label{eq:cumulative-residual}
\end{align}

\end{lemma}

The proof proceeds in four steps: construct the exact-target
potential, smooth the moving KKT path, establish one-step drift, and sum the
recursion.  Supporting proofs and local verifications are collected in
Appendix~\ref{app:auxiliary-verifications}.

\subsection{\texorpdfstring{Step 1: Construct the Exact-Target Potential}{Step 1: Construct the Exact-Target Potential}}
\label{sec:algorithmic-geometry}

Step~1 verifies that capping preserves the KKT target and places its
threshold representation in the moving box, then constructs a threshold-space
potential for exact sales-space error.

The next lemma verifies exact-target feasibility. Its complementarity
identities are also used in the state-update analysis.

\begin{lemma}
\label{lem:capped-rate}
For every \(r\ge0\), the selected moving target has the following
feasibility properties.
\begin{enumerate}[label=(\roman*),leftmargin=*]
\item \textbf{Target preservation under capping.}
\begin{equation}
 (\bm s^*(r),\lambda^*(r))
 =(\bm s^*(r^{\mathrm c}),\lambda^*(r^{\mathrm c})).
 \label{eq:cap-preserves-target}
\end{equation}

\item \textbf{Resource feasibility and complementarity.}
\begin{equation}
 r^{\mathrm c}-\sum_{i=1}^Ns_i^*(r)
 =(r^{\mathrm c}-r_0)^+,\qquad
 \lambda^*(r)(r^{\mathrm c}-r_0)^+=0.
 \label{eq:capped-target-complementarity}
\end{equation}

\item \textbf{Feasibility in the moving clipping box.}
\begin{equation}
 \begin{gathered}
 y_i^*(r)=m_i^{-1}(s_i^*(r))
 \le\min\left\{\bar X_i,\frac{r^{\mathrm c}}{p_i}\right\}
 =U_i(r),\qquad i\in[N],\\
 0\le\lambda^*(r)\le\lambda_{\max},\qquad
 \bigl((y_i^*(r))_{i=1}^N,\lambda^*(r)\bigr)\in\calK(r).
 \end{gathered}
 \label{eq:target-inside-threshold-cap}
\end{equation}
\end{enumerate}
\end{lemma}

The proof of \Cref{lem:capped-rate} is given in
Appendix~\ref{app:auxiliary-verifications}.

We therefore construct a potential
that measures error in sales space while following the algorithm's
threshold-space update.  For comparator sales \(\sigma\) and
current sales \(s\), both in \([0,m_i(\bar X_i)]\), define
\begin{equation*}
 \mathcal B_i(\sigma,s)
 :=
 \int_\sigma^s
 \frac{u-\sigma}
 {\left(1-F_i\bigl(m_i^{-1}(u)\bigr)\right)^2}\,\dd u.
\end{equation*}
Since \(m_i'(x)=1-F_i(x)\),
\begin{equation}
 \frac{\partial}{\partial x}
 \mathcal B_i\bigl(\sigma,m_i(x)\bigr)
 =
 \frac{m_i(x)-\sigma}{1-F_i(x)}.
 \label{eq:mirror-threshold-derivative}
\end{equation}

The mean threshold step contains a factor \(1-F_i(x)\), so
\eqref{eq:mirror-threshold-derivative} cancels that factor and leaves exactly
the sales-error--field pairing in \eqref{eq:learning-gap-block}.  This is why we
use \(\mathcal B_i\), rather than an ordinary squared distance.  The
comparator therefore remains in sales coordinates, while the update state
remains in threshold coordinates.

Because the pre-projection update can leave \([0,\bar X_i]\), we extend the
potential quadratically by one unit at each endpoint.  Matching the boundary
slopes below gives bounded curvature and monotonicity toward the comparator,
so projection cannot increase the potential:
\begin{equation}
 \left.
 \frac{\partial}{\partial x}\mathcal B_i(\sigma,m_i(x))
 \right|_{x=0}
 =-\sigma,
 \qquad
 \left.
 \frac{\partial}{\partial x}\mathcal B_i(\sigma,m_i(x))
 \right|_{x=\bar X_i}
 =
 \frac{m_i(\bar X_i)-\sigma}{1-F_i(\bar X_i)}.
 \label{eq:mirror-boundary-slopes}
\end{equation}
Using these slopes and the ordinary squared
discrepancy for the dual coordinate, define
\begin{equation}
 \begin{aligned}
 \mathcal W\bigl((\bm\sigma,\zeta),(\bm x,\lambda)\bigr)
 &:=
 \sum_{i=1}^N
 \begin{cases}
  \mathcal B_i(\sigma_i,0)-\sigma_i x_i+\dfrac12x_i^2,
  &-1\le x_i<0,\\[5pt]
  \mathcal B_i\bigl(\sigma_i,m_i(x_i)\bigr),
  &0\le x_i\le\bar X_i,\\[5pt]
  \begin{aligned}
   &\mathcal B_i\bigl(\sigma_i,m_i(\bar X_i)\bigr)
   +\dfrac{m_i(\bar X_i)-\sigma_i}{1-F_i(\bar X_i)}
    (x_i-\bar X_i)\\[-1pt]
   &\qquad+\dfrac12(x_i-\bar X_i)^2,
  \end{aligned}
  &\bar X_i<x_i\le\bar X_i+1
 \end{cases}\\
 &\quad+\frac12(\lambda-\zeta)^2.
 \end{aligned}
 \label{eq:joint-lyapunov}
\end{equation}
The outer derivatives are nonpositive to the left of the comparator
and nonnegative to its right, so the extensions preserve projection
compatibility while adding bounded curvature.
On the safe box, only the middle branch is active, and hence
\begin{equation*}
 \mathcal W\bigl((\bm\sigma,\zeta),(\bm x,\lambda)\bigr)
 =
 \sum_{i=1}^N
 \mathcal B_i\bigl(\sigma_i,m_i(x_i)\bigr)
 +\frac12(\lambda-\zeta)^2,
 \qquad 0\le x_i\le\bar X_i.
\end{equation*}

Define the comparator and update domains, respectively, by
\begin{align*}
 \mathcal Q
 :=\prod_{i=1}^N[0,m_i(\bar X_i)]\times[0,\lambda_{\max}],~~~
 \mathcal Z
 :=\prod_{i=1}^N[-1,\bar X_i+1]
 \times[-1,\lambda_{\max}+1].
\end{align*}
The joint domain is \(\mathcal Q\times\mathcal Z\).  The containment
below, together with the step-size bound \(\alpha_tG\le1\) verified below,
ensures that every pre-projection update remains in \(\mathcal Z\):
\begin{equation*}
 \calK(r)
 \subseteq
 \prod_{i=1}^N[0,\bar X_i]\times[0,\lambda_{\max}]
 \subseteq\mathcal Z.
\end{equation*}
The extra unit margin is used only to evaluate the potential at one
pre-projection update. It does not enlarge the feasible actions onto which
the algorithm projects.

\begin{lemma}
\label{lem:mirror-toolkit}
For \(q=(\bm\sigma,\zeta)\in\mathcal Q\) and
\(z=(\bm x,\lambda)\in\mathcal Z\) with
\(0\le x_i\le\bar X_i\) for every \(i\), there is a primitive constant
\(C_W\ge1\) such that
\begin{equation}
\begin{aligned}
 C_W^{-1}\bigl(\|\bm m(\bm x)-\bm\sigma\|^2+|\lambda-\zeta|^2\bigr)
 &\le \mathcal W((\bm\sigma,\zeta),(\bm x,\lambda))
 \le
 C_W\bigl(\|\bm m(\bm x)-\bm\sigma\|^2+|\lambda-\zeta|^2\bigr).
\end{aligned}
\label{eq:mirror-metric-comparison}
\end{equation}
\begin{equation}
 \|\nabla_q\mathcal W(q,z)\|^2+
 \|\nabla_z\mathcal W(q,z)\|^2
 \le C\mathcal W(q,z).
 \label{eq:mirror-gradient-control}
\end{equation}
\end{lemma}

The proof of \Cref{lem:mirror-toolkit} is given in
Appendix~\ref{app:auxiliary-verifications}.

We can now define the quantity that measures the tracking error of actual
interest.  Set
\begin{equation*}
 \begin{aligned}
 q_t^*&:=q^*(r_t)=(\bm s^*(r_t),\lambda^*(r_t)),
 &z_t&:=(\bm X_t,\lambda_t),\\
 W_t&:=\mathcal W(q_t^*,z_t)
 =\sum_{i=1}^N
 \mathcal B_i\bigl(s_i^*(r_t),m_i(X_{i,t})\bigr)
 +\frac12\bigl(\lambda_t-\lambda^*(r_t)\bigr)^2.
 \end{aligned}
\end{equation*}
By \eqref{eq:target-inside-threshold-cap}, both \(q_t^*\) and the
projected state \(z_t\) lie in the required domains.  The lower metric bound in
\eqref{eq:mirror-metric-comparison}
therefore gives
\begin{equation*}
 \|\bm m(\bm X_t)-\bm s^*(r_t)\|^2
 +|\lambda_t-\lambda^*(r_t)|^2
 \le C W_t.
\end{equation*}
Taking expectations and summing yields
\begin{equation}
 \sum_{t=1}^T\E\left[
  \|\bm m(\bm X_t)-\bm s^*(r_t)\|^2
  +|\lambda_t-\lambda^*(r_t)|^2
 \right]
 \le C\sum_{t=1}^T\E W_t.
 \label{eq:tracking-to-potential-sum}
\end{equation}

\subsection{\texorpdfstring{Step 2: Construct the Smoothed Comparator}{Step 2: Construct the Smoothed Comparator}}
\label{sec:moving-comparator-setup}

\textbf{The natural exact comparator and its obstruction.}
The exact target \(q^*(r_t)\) is feasible and ideal for the
fixed-target restoring inequality, but its path has corners when active sets
change or the resource constraint becomes slack.  At a crossing no derivative
gives a uniform expansion
\[
 q^*(r+\Delta r)-q^*(r)=\dot q^*(r)\Delta r+O(|\Delta r|^2).
\]
The remainder is generally first order.  Lipschitz continuity alone yields
\[
 C\sqrt{W_t}|r_{t+1}-r_t|
 \le \varepsilon\alpha_tW_t
 +C_\varepsilon|r_{t+1}-r_t|^2/\alpha_t,
\]
whose last term is \(O(\alpha_t)\), rather than the
\(O(\alpha_t^2)\) forcing required for harmonic tracking.

\textbf{High-level idea of smoothing.}
We therefore average \(q^*(r-\eta u)\) over a backward window of
width \(\eta\).  This analysis-only smoothing stays \(O(\eta)\) from the
target, has Taylor remainder \(O(|\Delta r|^2/\eta)\), and remains feasible
because it uses no more resource than \(r\).  The choice
\(\eta_t=\sqrt{\alpha_t}\) balances approximation and motion errors.

Formally, extend \(q^*\) constantly to negative \(r\), and
define
\begin{equation}
 \rho(u):=6u(1-u)\1_{[0,1]}(u),\qquad
 \int_0^1\rho(u)\,\dd u=1,
 \label{eq:fixed-kernel}
\end{equation}
and, for \(0<\eta\le1\),
\begin{equation*}
 q^\eta(r):=\int_0^1\rho(u)q^*(r-\eta u)\,\dd u
 =:(\bm s^\eta(r),\lambda^\eta(r)).
\end{equation*}

\begin{lemma}
\label{lem:smoothing}
There is a primitive constant \(C_{\mathrm{sm}}<\infty\) such that, for every
\(r,r'\ge0\) and \(0<\eta,\eta'\le1\), the following hold.
\begin{enumerate}[label=(\roman*)]
\item \textbf{Approximation and admissibility:}
\begin{equation}
 \|q^\eta(r)-q^*(r)\|\le C_{\mathrm{sm}}\eta,\qquad
 m_i^{-1}(s_i^\eta(r))\le y_i^*(r)\le U_i(r),\quad i\in[N],\qquad
 0\le\lambda^\eta(r)\le\lambda_{\max}.
 \label{eq:smoothing-approximation}
\end{equation}
\item \textbf{Monotonicity and constant tail:} the primal coordinates of
\(q^\eta\) are nondecreasing, the dual
coordinate is nonincreasing, and \(q^\eta(r)=q^*(r)\) whenever
\(r\ge r_0+\eta\).
\item \textbf{Motion in resource:} the derivative
\[
 \frac{\dd q^\eta(r)}{\dd r}
 =\left(
   \frac{\dd\bm s^\eta(r)}{\dd r},
   \frac{\dd\lambda^\eta(r)}{\dd r}
  \right)
\]
exists, \(\left\|\frac{\dd q^\eta(r)}{\dd r}\right\|
\le C_{\mathrm{sm}}\), and
\begin{equation}
 \left\|q^\eta(r')-q^\eta(r)
 -\frac{\dd q^\eta(r)}{\dd r}(r'-r)\right\|
 \le C_{\mathrm{sm}}|r'-r|^2/\eta.
 \label{eq:smoothing-taylor}
\end{equation}
\item \textbf{Motion in radius:}
\(\|q^{\eta'}(r)-q^\eta(r)\|\le C_{\mathrm{sm}}|\eta'-\eta|\).
\end{enumerate}
\end{lemma}

The proof of \Cref{lem:smoothing} is given in
Appendix~\ref{app:auxiliary-verifications}.

Choose \(\alpha_0\in(0,1]\) so that \(\alpha_0G\le1\), and set
\begin{equation}
 \eta_t:=\sqrt{\alpha_t},\qquad
 q_t:=q^{\eta_t}(r_t)=:(\bm\sigma_t,\zeta_t).
 \label{eq:qz-notation}
\end{equation}
Under \(\tau_0\ge\chi/\alpha_0\) in \eqref{eq:tuning},
\[
 \alpha_t
 =\frac{\chi}{\tau_0+\min\{t,n_t\}}
 \le\frac{\chi}{\tau_0}\le\alpha_0.
\]
Thus \(0<\eta_t\le1\) and \(\alpha_tG\le1\).  The auxiliary potential in
\eqref{eq:direct-moving-mirror-expansion} is now
\begin{equation*}
 \widetilde W_t=\mathcal W(q_t,z_t)
 =\sum_{i=1}^N
 \mathcal B_i\bigl(\sigma_{i,t},m_i(X_{i,t})\bigr)
 +\frac12(\lambda_t-\zeta_t)^2.
\end{equation*}

\textbf{Admissibility of the chosen comparator.}
\Cref{lem:smoothing}\textup{(i)} puts \(q_t,q_{t+1}\) in \(\mathcal Q\) and
their threshold representations in \(\calK(r_t)\) and
\(\calK(r_{t+1})\), respectively.  Together with
\(\alpha_tG\le1\), this records the admissibility of the comparator
sequence in \eqref{eq:qz-notation}. The remaining pre-projection domain and
projection conditions are verified below.

\subsection{\texorpdfstring{Step 3: Establish the One-Step Drift}{Step 3: Establish the One-Step Drift}}
\label{sec:state-update-contraction}

\begin{lemma}[Projected moving-comparator bound]
\label{lem:generic-moving-mirror}
Let \(\mathcal Q\) and \(\mathcal Z\) be convex sets, and let
\(W:\mathcal Q\times\mathcal Z\to\mathbb R\) be differentiable with
\(L_W\)-Lipschitz joint gradient.  Fix \(q,q^+\in\mathcal Q\),
\(z\in\mathcal Z\), a direction \(g\), a step size \(\alpha>0\), and a
nonempty closed convex set \(K^+\subseteq\mathcal Z\).  Suppose that
\(z-\alpha g\in\mathcal Z\) and
\begin{equation}
 W\!\left(q^+,\proj_{K^+}(z-\alpha g)\right)
 \le W(q^+,z-\alpha g).
 \label{eq:generic-projection-compatibility}
\end{equation}
Then, with \(z^+:=\proj_{K^+}(z-\alpha g)\),
\begin{align}
 W(q^+,z^+)
 &\le W(q,z)
 -\alpha\langle\nabla_zW(q,z),g\rangle
 +\langle\nabla_qW(q,z),q^+-q\rangle\notag\\
 &\quad+\frac{L_W}{2}\alpha^2\|g\|^2
 +L_W\alpha\|g\|\,\|q^+-q\|
 +\frac{L_W}{2}\|q^+-q\|^2.
 \label{eq:generic-moving-mirror}
\end{align}
\end{lemma}

The proof is given in Appendix~\ref{app:auxiliary-verifications}.

\textbf{Verification for the present potential.}
\label{app:one-step-verification}
For \(W=\mathcal W\), the following facts verify the hypotheses of
\Cref{lem:generic-moving-mirror}.
For \(i\in[N]\), set
\[
 G_i:=\max\{h_i,b_i-c_i,\lambda_{\max}\},
 \qquad
 G:=\left(\sum_{i=1}^NG_i^2+
 (\Dsum+\theta G_j)^2\right)^{1/2}.
\]
Then \(G<\infty\), and there is a primitive
\(L_{\mathcal W}<\infty\), such that the following hold.
\begin{enumerate}[label=(\roman*)]
\item \textbf{Regularity.}
On \(\mathcal Q\times\mathcal Z\), \(\mathcal W\) is differentiable and
has \(L_{\mathcal W}\)-Lipschitz joint gradient.

\item \textbf{Executable bounds.}
The algorithm is adapted and, for every \(t\),
\[
 z_t\in\calK(r_t),\qquad
 0\le X_{i,t},I_{i,t},Y_{i,t}\le\bar X_i<\bar d_i,\qquad
 0\le\lambda_t\le\lambda_{\max},\qquad
 \|\widehat G_t\|\le G.
\]

\item \textbf{
Pre-projection update and projection.}
Fix \(t<T\) and \(q_t,q_{t+1}\in\mathcal Q\).  If
\(\alpha_tG\le1\) and the threshold representation of \(q_{t+1}\) belongs
to \(\calK(r_{t+1})\), then
\[
 z_t-\alpha_t\widehat G_t\in\mathcal Z
\]
and, writing
\(z_{t+1}:=\proj_{\calK(r_{t+1})}
(z_t-\alpha_t\widehat G_t)\),
\begin{equation}
 \mathcal W(q_{t+1},z_{t+1})
 \le
 \mathcal W(q_{t+1},z_t-\alpha_t\widehat G_t).
 \label{eq:joint-projection-compatibility}
\end{equation}
These are the domain and projection conditions needed below.
\end{enumerate}

The three claims above are proved in
Appendix~\ref{app:auxiliary-verifications}, under
\emph{Details for the one-step verification}.

For the smoothed comparator in \eqref{eq:qz-notation}, the checklist
and \Cref{lem:smoothing}\textup{(i)} permit
\Cref{lem:generic-moving-mirror} with
\(g=\widehat G_t\), \(\alpha=\alpha_t\), and
\(K^+=\calK(r_{t+1})\), yielding
\begin{align}
 \widetilde W_{t+1}
 &\le \widetilde W_t
 -\alpha_t
 \left\langle
 \nabla_z\mathcal W(q_t,z_t),\widehat G_t
 \right\rangle+
 \left\langle
 \nabla_q\mathcal W(q_t,z_t),q_{t+1}-q_t
 \right\rangle
 +\frac{L_{\mathcal W}G^2}{2}\alpha_t^2\notag\\
 &\quad+
 L_{\mathcal W}G\alpha_t\|q_{t+1}-q_t\|
 +\frac{L_{\mathcal W}}{2}\|q_{t+1}-q_t\|^2.
 \label{eq:direct-moving-mirror-expansion}
\end{align}
The next lemma controls the state-update and comparator-motion terms,
respectively.

Conditioning on \(\calH_t\) and using the marginal demand laws gives
\begin{equation}
 \E[\widehat G_t\mid\calH_t]
 =G^\theta(\bm Y_t,\lambda_t;r_t),\qquad
 \|G^\theta(\bm Y_t,\lambda_t;r_t)
      -G^\theta(\bm X_t,\lambda_t;r_t)\|
 \le C\|\bm Y_t-\bm X_t\|.
 \label{eq:feasibility-bias}
\end{equation}
Let \(\mu_g:=\min_{1\le i\le N}h_i\kappa_i\) and
\(L_{g,j}:=h_jK_j/\beta_j^3\).  Choose \(\theta\) and define \(\mu_0\) by
\begin{equation}
 0<\theta\le
 \min\left\{1,\frac{\mu_g\beta_j}{L_{g,j}^2}\right\},
 \qquad
 \mu_0:=\frac12\min\{\mu_g,\theta\beta_j\}>0.
 \label{eq:restoring-constants}
\end{equation}
\begin{lemma}
\label{lem:two-contributions}
Let \(C_W\) be the upper metric-comparison constant in
\eqref{eq:mirror-metric-comparison} and set the state-update contraction
modulus \(\mu_{\mathrm{upd}}:=\mu_0/(8C_W)\).

\medskip
\noindent\textup{(i)} \textbf{State-update contraction.}
There is a primitive remainder constant \(C<\infty\) such that, for every
\(t<T\),
\begin{align}
 &-\alpha_t\E\left[
 \left\langle
 \nabla_z\mathcal W(q_t,z_t),\widehat G_t
 \right\rangle\mathrel{\Big|}\calH_t
 \right]
 \le -\mu_{\mathrm{upd}}\alpha_t\widetilde W_t\notag\\
 &\quad-\alpha_t\left(
 \sum_{i=1}^N
 m_i(X_{i,t})
 \bigl(g_i'(s_i^*(r_t))+\lambda^*(r_t)\bigr)
 +\lambda_t(r_t^{\mathrm c}-r_0)^+
 \right)
 +C\alpha_t\bigl(\alpha_t+\|\bm Y_t-\bm X_t\|^2\bigr).
 \label{eq:direct-state-update-bound}
\end{align}

\medskip
\noindent\textup{(ii)} \textbf{Comparator-motion control.}
There is a primitive constant \(C_{\mathrm{pair}}<\infty\), given explicitly
in \eqref{eq:explicit-c-pair}.  Set
\begin{equation}
 \chi_M:=\max\left\{1,\frac{32C_{\mathrm{pair}}}{\mu_{\mathrm{upd}}}\right\},
 \qquad
 n_0:=\left\lceil\max\left\{
 8,
 1+\frac{32C_{\mathrm{pair}}}{\mu_{\mathrm{upd}}\alpha_0},
 1+\frac{2}{\sqrt{\alpha_0}}
 \right\}\right\rceil.
 \label{eq:motion-cutoffs}
\end{equation}
Then there is a primitive constant \(C_{\rm m}<\infty\) such that, if
\[
 \chi\ge\chi_M,\qquad
 \chi/\alpha_0\le\tau_0\le2\chi/\alpha_0,
\]
then, for every \(t<T\) with \(n_t\ge n_0\),
\begin{align}
 &\E\left[
 \left\langle
 \nabla_q\mathcal W(q_t,z_t),q_{t+1}-q_t
 \right\rangle
 +L_{\mathcal W}G\alpha_t\|q_{t+1}-q_t\|
 +\frac{L_{\mathcal W}}2\|q_{t+1}-q_t\|^2
 \mathrel{\Big|}\calH_t\right]\notag\\
 &\qquad\le
 \frac{\mu_{\mathrm{upd}}}{2}\alpha_t\widetilde W_t
 +C_{\rm m}\alpha_t^2
 +C_{\rm m}\alpha_t\|\bm Y_t-\bm X_t\|^2.
\label{eq:conditional-comparator-motion-bound}
\end{align}
\end{lemma}

The proof of \Cref{lem:two-contributions}, including the fixed-target
restoring calculation used in part~\textup{(i)}, is given in
Appendix~\ref{app:auxiliary-verifications}.

Set
\(\mu_{\mathrm{rec}}:=\min\{\mu_{\mathrm{upd}}/2,1/2\}\), and decrease
\(\alpha_0\), if necessary, so that
\begin{equation*}
 \alpha_0G\le1,\qquad \mu_{\mathrm{rec}}\alpha_0\le\frac12.
\end{equation*}
With this \(\alpha_0\), take \(\chi_M,n_0\) from
\Cref{lem:two-contributions}\textup{(ii)}.  For the
rest of this step, impose
\begin{equation}
 \chi\ge\chi_M,\qquad
 \chi\mu_{\mathrm{rec}}>3,\qquad
 \chi/\alpha_0\le\tau_0\le2\chi/\alpha_0,
 \label{eq:tuning}
\end{equation}
and fix \(t<T\) with \(n_t\ge n_0\).  The storewise KKT inequalities give
\(g_i'(s_i^*(r_t))+\lambda^*(r_t)\ge0\), while
\(m_i(X_{i,t})\ge0\), \(\lambda_t\ge0\), and
\((r_t^{\mathrm c}-r_0)^+\ge0\).  Thus the residual below is nonnegative.
Combining
\Cref{lem:two-contributions}\textup{(i)--(ii)} in
\eqref{eq:direct-moving-mirror-expansion}, and using
\(\mu_{\mathrm{rec}}\le\mu_{\mathrm{upd}}/2\) and
\(\mu_{\mathrm{rec}}\le1\), gives
\begin{align}
 \E[\widetilde W_{t+1}\mid\calH_t]
 &\le(1-\mu_{\mathrm{rec}}\alpha_t)\widetilde W_t
 -\mu_{\mathrm{rec}}\alpha_t\left(
 \sum_{i=1}^N
 m_i(X_{i,t})
 \bigl(g_i'(s_i^*(r_t))+\lambda^*(r_t)\bigr)
 +\lambda_t(r_t^{\mathrm c}-r_0)^+
 \right)\notag\\
 &\quad+C\alpha_t^2
 +C\alpha_t\|\bm Y_t-\bm X_t\|^2.
 \label{eq:mirror-recursion}
\end{align}
The condition \(\chi\mu_{\mathrm{rec}}>3\), in particular, guarantees
\(\mu_{\mathrm{rec}}>1/\chi\), as needed for the reciprocal-step
absorption below.

This completes the one-step drift argument.

\subsection{\texorpdfstring{Step 4: Sum the Recursion}{Step 4: Sum the Recursion}}
\label{sec:tracking-assembly}

With the one-step recursion \eqref{eq:mirror-recursion} established in
Appendix~\ref{sec:state-update-contraction}, it remains
to sum it in two ways and then return from the smoothed comparator to the
exact target.

\textbf{Summation.}
Dropping the nonnegative parenthesized residual in
\eqref{eq:mirror-recursion} and taking expectations gives
\begin{equation}
 \E\widetilde W_{t+1}
 \le(1-\mu_{\mathrm{rec}}\alpha_t)\E\widetilde W_t
 +C\alpha_t^2
 +C\alpha_t\E\|\bm Y_t-\bm X_t\|^2.
 \label{eq:scalar-potential-recursion}
\end{equation}

The first required bound is
\begin{equation}
 \sum_{t=1}^T\E\widetilde W_t
 \le C\log T
 +C\sum_{t=1}^T\E\|\bm Y_t-\bm X_t\|^2.
 \label{eq:potential-sum}
\end{equation}
The second bound compares the
exact and smoothed potentials:
\begin{equation}
 W_t\le C\widetilde W_t+C\alpha_t.
 \label{eq:exact-to-smoothed-potential}
\end{equation}
Consequently,
\begin{equation}
 \sum_{t=1}^T\E W_t
 \le C\sum_{t=1}^T\E\widetilde W_t+C\log T.
 \label{eq:exact-to-smoothed-potential-sum}
\end{equation}
To prove \eqref{eq:potential-sum}, first suppose \(T\ge n_0\), set
\(L=T-n_0+1\), and write
\[
 u_t:=\E\widetilde W_t,
 \qquad
 b_t:=C\E\|\bm Y_t-\bm X_t\|^2,
 \qquad
 h_t:=\frac1{\alpha_t}.
\]
Compactness gives \(0\le u_t\le U\).  For \(t\le L\), divide
\eqref{eq:scalar-potential-recursion} by \(\alpha_t\) and sum to obtain
\[
 \mu_{\mathrm{rec}}\sum_{t=1}^L u_t
 \le \sum_{t=1}^L h_t(u_t-u_{t+1})
 +C\sum_{t=1}^L\alpha_t+\sum_{t=1}^L b_t.
\]
Writing \(d_t=\min\{t,T-t+1\}\), summation by parts and
\(h_t=(\tau_0+d_t)/\chi\) give
\[
 \sum_{t=1}^Lh_t(u_t-u_{t+1})
 =h_1u_1+\sum_{t=2}^L(h_t-h_{t-1})u_t-h_Lu_{L+1}
 \le C+\frac1\chi\sum_{t=2}^Lu_t.
\]
Because \(\mu_{\mathrm{rec}}>1/\chi\), this sum is absorbed.  Moreover,
\[
 \sum_{t=1}^T\alpha_t
 \le2\chi\sum_{k=1}^{\lceil T/2\rceil}(\tau_0+k)^{-1}
 \le C\log T,
\]
and the final \(n_0-1\) potentials contribute at most \((n_0-1)U\).
This proves \eqref{eq:potential-sum}. For \(T<n_0\), boundedness gives the
same conclusion directly.

\textbf{Proof of the exact-to-smoothed comparison.}
The upper metric comparison in
\eqref{eq:mirror-metric-comparison} gives
\[
 W_t\le C_W\left(
 \|\bm m(\bm X_t)-\bm s^*(r_t)\|^2
 +|\lambda_t-\lambda^*(r_t)|^2\right).
\]
Insert \(q_t=(\bm\sigma_t,\zeta_t)\) separately in the primal and dual
coordinates and use \(\|a+b\|^2\le2\|a\|^2+2\|b\|^2\).  The expression in
parentheses is at most
\[
 2\left(\|\bm m(\bm X_t)-\bm\sigma_t\|^2
 +|\lambda_t-\zeta_t|^2\right)
 +2\|q_t-q^*(r_t)\|^2.
\]
The lower metric comparison in \eqref{eq:mirror-metric-comparison} gives
\[
 \|\bm m(\bm X_t)-\bm\sigma_t\|^2+|\lambda_t-\zeta_t|^2
 \le C\widetilde W_t,
\]
while \Cref{lem:smoothing}\textup{(i)} and \(\eta_t^2=\alpha_t\) give
\[
 \|q_t-q^*(r_t)\|^2\le C_{\mathrm{sm}}^2\eta_t^2\le C\alpha_t.
\]
Substitution proves \eqref{eq:exact-to-smoothed-potential}.  Taking
expectations and summing this pointwise bound yields
\[
 \sum_{t=1}^T\E W_t
 \le C\sum_{t=1}^T\E\widetilde W_t+C\sum_{t=1}^T\alpha_t.
\]
The harmonic estimate above now proves
\eqref{eq:exact-to-smoothed-potential-sum}.

Combining \eqref{eq:tracking-to-potential-sum},
\eqref{eq:exact-to-smoothed-potential-sum}, and \eqref{eq:potential-sum}
gives
\begin{align*}
 &\sum_{t=1}^T\E\left[
  \|\bm m(\bm X_t)-\bm s^*(r_t)\|^2
  +|\lambda_t-\lambda^*(r_t)|^2
 \right]
 \le
 C\sum_{t=1}^T\E W_t\notag\\
 &\le C\sum_{t=1}^T\E\widetilde W_t+C\log T
 \le
 C\log T
 +C\sum_{t=1}^T\E\|\bm Y_t-\bm X_t\|^2,
\end{align*}
which is \eqref{eq:cumulative-tracking}.

\textbf{Residual summation.}
\label{app:pd-residual-proof}
It remains to prove \Cref{lem:technical-pd}\textup{(ii)} from the same
recursion.  For \(T<n_0\), each storewise residual is bounded by
\(\lambda_{\max}m_i(\bar X_i)\), so the claim follows after enlarging
\(C_{\rm md}\).   Suppose therefore that \(T\ge n_0\), and let
\(L=T-n_0+1\).  For the residual, write
\(w_t=\E\widetilde W_t\), rearrange
\eqref{eq:mirror-recursion}, take expectations, and divide by
\(\mu_{\mathrm{rec}}\alpha_t\).  Discarding only the nonpositive term
\(-w_t\) gives
\begin{align}
 &\E\!\left[
  \sum_{i=1}^N
  m_i(X_{i,t})
  \bigl(g_i'(s_i^*(r_t))+\lambda^*(r_t)\bigr)
  +\lambda_t(r_t^{\mathrm c}-r_0)^+
 \right]\notag\\
 &\qquad\le\frac{w_t-w_{t+1}}{\mu_{\mathrm{rec}}\alpha_t}
 +C\alpha_t
 +C\E\|\bm Y_t-\bm X_t\|^2.
 \label{eq:residual-rearrange}
\end{align}
The intermediate inequality contains an additional resource term.
By \Cref{lem:capped-rate},
\[
 r_t^{\mathrm c}-\sum_{i=1}^Ns_i^*(r_t)
 =(r_t^{\mathrm c}-r_0)^+,
 \qquad
 \lambda^*(r_t)(r_t^{\mathrm c}-r_0)^+=0.
\]
Consequently, the target resource residual paired with the current dual
error satisfies
\[
 (\lambda_t-\lambda^*(r_t))
 \left(r_t^{\mathrm c}-\sum_{i=1}^Ns_i^*(r_t)\right)
 =(\lambda_t-\lambda^*(r_t))(r_t^{\mathrm c}-r_0)^+
 =\lambda_t(r_t^{\mathrm c}-r_0)^+\ge0.
\]
Thus this resource-dissipation term may be dropped from the left-hand side of
\eqref{eq:residual-rearrange} when proving
\Cref{lem:technical-pd}\textup{(ii)}.

The reciprocal-step telescope is the exact identity
\begin{align*}
 \sum_{t=1}^{L}\frac{w_t-w_{t+1}}{\alpha_t}
 &=\frac{w_1}{\alpha_1}
 +\sum_{t=2}^{L}w_t
  \left(\frac1{\alpha_t}-\frac1{\alpha_{t-1}}\right)
 -\frac{w_{L+1}}{\alpha_L}.
\end{align*}
For this proof, write \(d_t:=\min\{t,n_t\}\).  Because
\(1/\alpha_t=(\tau_0+d_t)/\chi\), the positive part of the
reciprocal increment is exactly \(1/\chi\) while \(d_t\) increases, zero on
a possible midpoint plateau, and
nonpositive while it decreases.  The first term is uniformly bounded and
the last is nonpositive.  Therefore the telescope is at most
\(C+(1/\chi)\sum_{t=2}^Lw_t\), which is controlled by
\eqref{eq:potential-sum}.  
After dropping the nonnegative resource term as above, summing
\eqref{eq:residual-rearrange} and adding the final \(n_0-1\) bounded
storewise residuals proves \eqref{eq:cumulative-residual}.

\subsection{\texorpdfstring{Auxiliary Proofs and Verifications}{Auxiliary Proofs and Verifications}}
\label{app:auxiliary-verifications}

This subsection collects, in order of use, the auxiliary proofs for
Steps~1--3 and the remaining one-step verifications.

\textbf{Proof of \Cref{lem:capped-rate}.}

\proof{Proof.}
\noindent\emph{Target preservation and complementarity.}
Since \(r_0=\sum_i\phi_i(0)\le\Dsum\), if \(r\le\Dsum\), capping does
nothing. If \(r>\Dsum\), then both
\(r\) and \(r^{\mathrm c}=\Dsum\) lie on the constant nonbinding branch
\((\bm\phi(0),0)\) of \eqref{eq:kkt-path}.  This proves
\eqref{eq:cap-preserves-target}.  Moreover, the binding branch gives
\(\sum_i s_i^*(r)=r=r^{\mathrm c}\) for \(r<r_0\), while the nonbinding
branch gives \(\sum_i s_i^*(r)=r_0\) and \(\lambda^*(r)=0\) for
\(r\ge r_0\).  Hence
\[
 r^{\mathrm c}-\sum_i s_i^*(r)=(r^{\mathrm c}-r_0)^+,
 \qquad \lambda^*(r)(r^{\mathrm c}-r_0)^+=0.
\]

\noindent\emph{Feasibility in the moving clipping box.}
It remains to verify feasibility in the moving clipping box.  A nonnegative
resource price
can only lower the store threshold below the zero-price newsvendor
threshold, and hence \(y_i^*(r)<\bar X_i\).  If store \(i\) is active, then
recall that \(p_i\) is its zero-price survival probability.  Therefore
\(1-F_i(y_i^*(r))\ge p_i\), and
\[
 s_i^*(r)
 =\int_0^{y_i^*(r)}\bigl(1-F_i(u)\bigr)\,\dd u
 \ge p_i y_i^*(r).
\]
The capped complementarity identity already proved gives
\[
 s_i^*(r)\le\sum_{k=1}^Ns_k^*(r)
 =r^{\mathrm c}
 -(r^{\mathrm c}-r_0)^+
 \le r^{\mathrm c}.
\]
This proves the primal bound in
\eqref{eq:target-inside-threshold-cap}.  For an inactive store,
\(y_i^*(r)=0\), so the same conclusion holds.  Finally,
\(0\le\lambda^*(r)\le\lambda_{\max}\) is part of the selected KKT path.
Together with the coordinatewise primal bounds and the definition of
\(\calK(r)\), it proves the asserted threshold--dual membership.
\endproof

\phantomsection
\label{app:proof-mirror-toolkit}

\textbf{Proof of \Cref{lem:mirror-toolkit}.}

\proof{Proof.}
For a fixed store \(i\), abbreviate
\[
 S_i(x):=1-F_i(x),\qquad M_i:=m_i(\bar X_i),\qquad
 w_i(u):=S_i\bigl(m_i^{-1}(u)\bigr)^{-2}.
\]
The safe-box bounds imply, for \(0\le u\le M_i\),
\[
 1\le w_i(u)\le\beta_i^{-2}.
\]

\medskip
\noindent\emph{Metric comparison.}
For \(s\in[0,M_i]\), the change of variables
\(u=\sigma+\tau(s-\sigma)\) gives
\[
 \mathcal B_i(\sigma,s)
 =(s-\sigma)^2\int_0^1
 \tau w_i\bigl(\sigma+\tau(s-\sigma)\bigr)\,\dd\tau.
\]
Consequently,
\[
 \frac12(s-\sigma)^2
 \le\mathcal B_i(\sigma,s)
 \le\frac1{2\beta_i^2}(s-\sigma)^2.
\]
Setting \(s=m_i(x_i)\), summing over \(i\), and adding the dual quadratic
proves \eqref{eq:mirror-metric-comparison}.

\medskip
\noindent\emph{Gradient control.}
Let \(H_i(\sigma,x)\) denote the \(i\)-th primal summand in
\eqref{eq:joint-lyapunov}.  On the safe box, direct differentiation of its
middle branch gives
\[
 |\partial_xH_i(\sigma_i,x_i)|
 \le\beta_i^{-1}|m_i(x_i)-\sigma_i|,
 \qquad
 |\partial_\sigma H_i(\sigma_i,x_i)|
 \le\beta_i^{-2}|m_i(x_i)-\sigma_i|.
\]
The two derivatives of the dual quadratic have magnitude
\(|\lambda-\zeta|\).  Therefore,
\[
 \|\nabla_q\mathcal W(q,z)\|^2+
 \|\nabla_z\mathcal W(q,z)\|^2
 \le C\bigl(\|\bm m(\bm x)-\bm\sigma\|^2
             +|\lambda-\zeta|^2\bigr).
\]
The lower bound in \eqref{eq:mirror-metric-comparison} converts the
right-hand side to \(C\mathcal W(q,z)\), proving
\eqref{eq:mirror-gradient-control}.
\endproof

\textbf{Proof of \Cref{lem:smoothing}.}

\proof{Proof.}
Let \(L_q\) be a global Lipschitz constant of \(q^*\).  Coupling the
integrands at the same kernel point gives
\begin{align*}
 \|q^\eta(r)-q^*(r)\|
 &\le L_q\eta\int_0^1u\rho(u)\,\dd u\le L_q\eta,\\
 \|q^{\eta'}(r)-q^\eta(r)\|
 &\le L_q|\eta'-\eta|\int_0^1u\rho(u)\,\dd u
 \le L_q|\eta'-\eta|.
\end{align*}
Averaging preserves the coordinatewise monotonicity of the KKT path.  Since
the kernel uses only \(r-\eta u\le r\), primal monotonicity also gives
\[
 m_i^{-1}(s_i^\eta(r))
 \le m_i^{-1}(s_i^*(r))
 =y_i^*(r)\le U_i(r).
\]
The multiplier average remains in \([0,\lambda_{\max}]\), and every integrand
equals the constant nonbinding target when \(r\ge r_0+\eta\).  This proves
parts~\textup{(i)}, \textup{(ii)}, and \textup{(iv)}.

Extend \(\rho\) by zero outside \([0,1]\).  A Lipschitz path is absolutely
continuous and has an a.e.\ derivative \(\dot q^*\) with
\(\|\dot q^*\|_\infty\le L_q\).  With
\[
 K_\eta(u):=\eta^{-1}\rho(u/\eta)\1_{[0,\eta]}(u),
\]
the smoothing is \(K_\eta*q^*\), and its weak derivative is
\[
 \frac{\dd q^\eta(r)}{\dd r}=(K_\eta*\dot q^*)(r).
\]
Translation continuity in \(L^1\) makes this derivative continuous, so it
is the classical derivative everywhere.  The zero extension of \(\rho\)
has total variation
\[
 V_\rho:=|\rho(0)|+|\rho(1)|
 +\int_0^1|\rho'(u)|\,\dd u.
\]
Consequently
\[
 \left\|\frac{\dd q^\eta(r)}{\dd r}\right\|\le L_q,
 \qquad
 \left\|\frac{\dd q^\eta(r')}{\dd r'}
 -\frac{\dd q^\eta(r)}{\dd r}\right\|
 \le (V_\rho L_q/\eta)|r'-r|,
\]
because
\(\|K_\eta(\cdot-h)-K_\eta\|_1
\le(V_\rho/\eta)|h|\).
The integral form of Taylor's theorem gives
\eqref{eq:smoothing-taylor}.  Choosing \(C_{\mathrm{sm}}\) to dominate the displayed
constants proves part~\textup{(iii)} and completes the proof.
\endproof

\textbf{Proof of \Cref{lem:generic-moving-mirror}.}

\proof{Proof.}
By \eqref{eq:generic-projection-compatibility},
\[
 W(q^+,z^+)\le W(q^+,z-\alpha g).
\]
Applying the descent lemma first in the comparator argument gives
\begin{align*}
 W(q^+,z-\alpha g)
 &\le W(q,z-\alpha g)
 +\left\langle\nabla_qW(q,z-\alpha g),q^+-q\right\rangle
 +\frac{L_W}{2}\|q^+-q\|^2.
\end{align*}
Moreover,
\[
 \|\nabla_qW(q,z-\alpha g)-\nabla_qW(q,z)\|
 \le L_W\alpha\|g\|,
\]
while the descent lemma in the state argument gives
\[
 W(q,z-\alpha g)
 \le W(q,z)-\alpha\langle\nabla_zW(q,z),g\rangle
 +\frac{L_W}{2}\alpha^2\|g\|^2.
\]
Substitution and Cauchy--Schwarz prove
\eqref{eq:generic-moving-mirror}.
\endproof

\textbf{Proof of \Cref{lem:two-contributions}.}

\proof{Proof.}
\medskip
\noindent\textbf{Part (i): State-update contraction.}

\medskip
\noindent\emph{Substep 1: fixed-target restoring inequality.}
For every \(r\ge0\), \(0\le\lambda\le\lambda_{\max}\), and
\(0\le x_i\le\bar X_i\), the fixed-target calculation below gives
\begin{align}
 &\left\langle
 \nabla_z\mathcal W\bigl(q^*(r),(\bm x,\lambda)\bigr),
 G^\theta(\bm x,\lambda;r)
 \right\rangle
 \ge
 \mu_0\left(
 \|\bm m(\bm x)-\bm s^*(r)\|^2
 +|\lambda-\lambda^*(r)|^2\right)\notag\\
 &\quad+
 \sum_{i=1}^N
 m_i(x_i)\bigl(g_i'(s_i^*(r))+\lambda^*(r)\bigr)
 +\lambda(r^{\mathrm c}-r_0)^+.
 \label{eq:stabilized-merit}
\end{align}

\medskip
\noindent\emph{Substep 2: transfer to the smoothed comparator and identify the
ideal pairing.}
Fix \(r\ge0\), \(0<\eta\le1\),
\(0\le\lambda\le\lambda_{\max}\), and \(0\le X_i\le\bar X_i\).
Subtracting the exact-target pairing in \eqref{eq:stabilized-merit} from
the corresponding pairing with \(q^\eta(r)\) gives
\begin{align*}
 &-\sum_{i=1}^N\bigl(s_i^\eta(r)-s_i^*(r)\bigr)
  \bigl(g_i'(m_i(X_i))+\lambda\bigr)\\
 &\quad-\bigl(\lambda^\eta(r)-\lambda^*(r)\bigr)
 \left(r^{\mathrm c}-\sum_{i=1}^Nm_i(X_i)
 +\theta\bigl(1-F_j(X_j)\bigr)
 \bigl(g_j'(m_j(X_j))+\lambda\bigr)\right).
\end{align*}
If store \(i\) is inactive at \(r\), backward smoothing leaves
\(s_i^\eta(r)=s_i^*(r)=0\).  If it is active, its KKT residual vanishes,
and safe-box regularity gives
\[
 |g_i'(m_i(X_i))+\lambda|
 \le C\bigl(\|\bm m(\bm X)-\bm s^*(r)\|
          +|\lambda-\lambda^*(r)|\bigr).
\]
For the dual line, use
\[
 r^{\mathrm c}-\sum_{i=1}^N m_i(X_i)
 =(r^{\mathrm c}-r_0)^+
  -\sum_{i=1}^N\bigl(m_i(X_i)-s_i^*(r)\bigr),
\]
and \eqref{eq:anchor-zero}.  The product of
\(\lambda^\eta(r)-\lambda^*(r)\) and
\((r^{\mathrm c}-r_0)^+\) vanishes outside
\([r_0,r_0+\eta]\). Inside that layer its two factors are \(O(\eta)\).
Together with \eqref{eq:smoothing-approximation}, the displayed difference
is bounded below by
\[
 -C\eta\bigl(\|\bm m(\bm X)-\bm s^*(r)\|
          +|\lambda-\lambda^*(r)|\bigr)-C\eta^2.
\]
Young's inequality yields
\[
 C\eta\bigl(\|\bm m(\bm X)-\bm s^*(r)\|
          +|\lambda-\lambda^*(r)|\bigr)
 \le\frac{\mu_0}{2}\left(
   \|\bm m(\bm X)-\bm s^*(r)\|^2
   +|\lambda-\lambda^*(r)|^2\right)+C\eta^2.
\]
Moreover, the squared triangle inequality,
\eqref{eq:smoothing-approximation}, and the upper bound in
\eqref{eq:mirror-metric-comparison} give
\begin{align*}
 &\|\bm m(\bm X)-\bm s^*(r)\|^2
   +|\lambda-\lambda^*(r)|^2\\
 &\qquad\ge
 \frac12\left(\|\bm m(\bm X)-\bm s^\eta(r)\|^2
   +|\lambda-\lambda^\eta(r)|^2\right)-C\eta^2\\
 &\qquad\ge
 \frac1{2C_W}\mathcal W\bigl(q^\eta(r),(\bm X,\lambda)\bigr)-C\eta^2.
\end{align*}
Combining these bounds with \eqref{eq:stabilized-merit} gives
\begin{align}
 &\sum_{i=1}^N
 \bigl(m_i(X_i)-s_i^\eta(r)\bigr)
 \bigl(g_i'(m_i(X_i))+\lambda\bigr)\notag\\
 &\quad+\bigl(\lambda-\lambda^\eta(r)\bigr)
 \left(r^{\mathrm c}-\sum_{i=1}^Nm_i(X_i)
 +\theta\bigl(1-F_j(X_j)\bigr)
 \bigl(g_j'(m_j(X_j))+\lambda\bigr)\right)\notag\\
 &\qquad\ge
 \frac{\mu_0}{4C_W}
 \mathcal W\bigl(q^\eta(r),(\bm X,\lambda)\bigr)
 +\sum_{i=1}^N
 m_i(X_i)\bigl(g_i'(s_i^*(r))+\lambda^*(r)\bigr)
 +\lambda(r^{\mathrm c}-r_0)^+-C\eta^2.
 \label{eq:smoothed-target-restoring}
\end{align}
For a safe state, \eqref{eq:mirror-threshold-derivative} and
\eqref{eq:threshold-field} show that the survival factors cancel, so the
left side of \eqref{eq:smoothed-target-restoring} is precisely
\(\langle\nabla_z\mathcal W(q^\eta(r),(\bm X,\lambda)),
G^\theta(\bm X,\lambda;r)\rangle\).
At \((r,\eta,\bm X,\lambda)=(r_t,\eta_t,\bm X_t,\lambda_t)\), using
\(\eta_t^2=\alpha_t\), we obtain
\begin{align*}
 &\left\langle
 \nabla_z\mathcal W(q_t,z_t),
 G^\theta(\bm X_t,\lambda_t;r_t)
 \right\rangle\\
 &\quad\ge
 \frac{\mu_0}{4C_W}\widetilde W_t
 +\sum_{i=1}^Nm_i(X_{i,t})
   \bigl(g_i'(s_i^*(r_t))+\lambda^*(r_t)\bigr)
 +\lambda_t(r_t^{\mathrm c}-r_0)^+-C\alpha_t.
\end{align*}

\medskip
\noindent\emph{Substep 3: transfer from the ideal field to the executable
stochastic update.}
Because \(\nabla_z\mathcal W(q_t,z_t)\) is \(\calH_t\)-measurable,
\eqref{eq:feasibility-bias} yields
\begin{align*}
 &\E\left[
  \left\langle\nabla_z\mathcal W(q_t,z_t),\widehat G_t\right\rangle
  \mathrel{\Big|}\calH_t\right]
 =
 \left\langle\nabla_z\mathcal W(q_t,z_t),
  G^\theta(\bm X_t,\lambda_t;r_t)\right\rangle\\
 &\quad+
 \left\langle\nabla_z\mathcal W(q_t,z_t),
  G^\theta(\bm Y_t,\lambda_t;r_t)
  -G^\theta(\bm X_t,\lambda_t;r_t)\right\rangle.
\end{align*}
By \eqref{eq:mirror-gradient-control}, \eqref{eq:feasibility-bias}, and
Young's inequality, the second inner product has absolute value at most
\[
 C\sqrt{\widetilde W_t}\,\|\bm Y_t-\bm X_t\|
 \le \mu_{\mathrm{upd}}\widetilde W_t
 +C\|\bm Y_t-\bm X_t\|^2.
\]
Consequently,
\begin{align*}
 &\E\left[
  \left\langle\nabla_z\mathcal W(q_t,z_t),\widehat G_t\right\rangle
  \mathrel{\Big|}\calH_t\right]
 \ge \mu_{\mathrm{upd}}\widetilde W_t\\
 &\quad+\sum_{i=1}^Nm_i(X_{i,t})
   \bigl(g_i'(s_i^*(r_t))+\lambda^*(r_t)\bigr)
 +\lambda_t(r_t^{\mathrm c}-r_0)^+
 -C\bigl(\alpha_t+\|\bm Y_t-\bm X_t\|^2\bigr).
\end{align*}
Multiplying by \(-\alpha_t\) therefore proves
\eqref{eq:direct-state-update-bound}.

\medskip
\noindent\textbf{Part (ii): Comparator-motion control.}

For brevity, write
\(D^\eta(r):=\frac{\dd q^\eta(r)}{\dd r}\).

The resource recursion gives
\begin{equation}
 r_{t+1}-r_t
 =\frac{r_t-\sum_{i=1}^NS_{i,t}}{n_t-1},\qquad
 \E[r_{t+1}-r_t\mid\calH_t]
 =\frac{r_t-\sum_{i=1}^Nm_i(Y_{i,t})}{n_t-1}.
 \label{eq:conditional-resource-motion}
\end{equation}
We first derive the first-order resource-pairing estimate.  Fix
\(r\ge0\), \(0<\eta\le1\), \(0\le x_i,y_i\le\bar X_i\), and
\(0\le\lambda\le\lambda_{\max}\).
For this calculation, abbreviate
\[
 \mathsf W:=\mathcal W\bigl(q^\eta(r),(\bm x,\lambda)\bigr),
 \qquad
 \beta_{\min}:=\min_{i\in[N]}\beta_i,
 \qquad
 C_{\nabla q}:=\sqrt2\,\beta_{\min}^{-2}.
\]
The derivative calculation behind
\eqref{eq:mirror-gradient-control}, together with the lower metric bound, gives
the first estimate below. \Cref{lem:smoothing}\textup{(iii)} gives the second:
\[
 \left\|
 \nabla_q\mathcal W\bigl(q^\eta(r),(\bm x,\lambda)\bigr)
 \right\|
 \le C_{\nabla q}\sqrt{\mathsf W},
 \qquad
 \left\|\frac{\dd q^\eta(r)}{\dd r}\right\|\le C_{\mathrm{sm}}.
\]

Set
\[
 A_\eta:=1+\sqrt N\,C_{\mathrm{sm}},
 \qquad A_W:=\sqrt{2N},
 \qquad A_Y:=\sqrt N.
\]

If \(r\le r_0\), the resource constraint binds and
\(r=\sum_{i=1}^N s_i^*(r)\).  The smoothing approximation, the lower bound in
\eqref{eq:mirror-metric-comparison}, and the one-Lipschitz property of
every \(m_i\) give
\begin{align*}
 \left|r-\sum_{i=1}^N m_i(y_i)\right|
 &\le \|\bm s^*(r)-\bm s^\eta(r)\|_1
 +\|\bm s^\eta(r)-\bm m(\bm x)\|_1
 +\sum_{i=1}^N|m_i(x_i)-m_i(y_i)|\\
 &\le A_\eta\eta+A_W\sqrt{\mathsf W}
 +A_Y\|\bm y-\bm x\|.
\end{align*}
If \(r_0<r<r_0+\eta\), then
\(\bm s^*(r)=\bm s^*(r_0)\),
\(|r-r_0|\le\eta\), and
\(\|\bm s^\eta(r)-\bm s^*(r_0)\|\le C_{\mathrm{sm}}\eta\). Inserting
\(r_0=\sum_{i=1}^N s_i^*(r_0)\) gives the same bound.  Finally, if
\(r\ge r_0+\eta\), \Cref{lem:smoothing}\textup{(ii)--(iii)} gives
\(\frac{\dd q^\eta(r)}{\dd r}=0\).  Thus all three resource regimes satisfy
\begin{align*}
 &\left|
 \left\langle
 \nabla_q\mathcal W\bigl(q^\eta(r),(\bm x,\lambda)\bigr),
 \frac{\dd q^\eta(r)}{\dd r}
 \left(r-\sum_{i=1}^Nm_i(y_i)\right)
 \right\rangle
 \right|\\
 &\qquad\le
 C_{\nabla q}C_{\mathrm{sm}}\sqrt{\mathsf W}
 \left(A_\eta\eta+A_W\sqrt{\mathsf W}
 +A_Y\|\bm y-\bm x\|\right).
\end{align*}
Using \(ab\le(a^2+b^2)/2\), a valid explicit choice is
\begin{equation}
 C_{\mathrm{pair}}
 :=C_{\nabla q}C_{\mathrm{sm}}
 \left(A_W+\frac{A_\eta+A_Y}{2}\right)
 =\sqrt2\,\beta_{\min}^{-2}C_{\mathrm{sm}}
 \left[\sqrt{2N}
 +\frac{1+\sqrt N C_{\mathrm{sm}}+\sqrt N}{2}\right].
 \label{eq:explicit-c-pair}
\end{equation}
Indeed, for all the stated arguments,
\begin{equation*}
\begin{aligned}
 &\left|
 \left\langle
 \nabla_q\mathcal W\bigl(q^\eta(r),(\bm x,\lambda)\bigr),
 \frac{\dd q^\eta(r)}{\dd r}
 \left(r-\sum_{i=1}^Nm_i(y_i)\right)
 \right\rangle
 \right|\\
 &\qquad\le C_{\mathrm{pair}}\left(
 \mathcal W\bigl(q^\eta(r),(\bm x,\lambda)\bigr)
 +\|\bm y-\bm x\|^2+\eta^2\right).
\end{aligned}
\end{equation*}

\medskip
\noindent\emph{Substep 1: control the two sources of comparator motion.}
Fix \(t<T\).  The comparator increment is
\begin{align*}
 q_{t+1}-q_t
 &=
 q^{\eta_t}(r_{t+1})-q^{\eta_t}(r_t)\\
 &\quad+
 q^{\eta_{t+1}}(r_{t+1})-q^{\eta_t}(r_{t+1}).
\end{align*}
If \(r_t\le\max(\Dsum,r_0+\eta_t)\), then
\(r_0\le\Dsum\), \(\eta_t\le1\), and
\(\sum_{i=1}^N S_{i,t}\le\Dsum\) imply
\[
 |r_{t+1}-r_t|
 =\frac{|r_t-\sum_{i=1}^N S_{i,t}|}{n_t-1}
 \le\frac{\Dsum+1}{n_t-1}.
\]
If \(r_t>\max(\Dsum,r_0+\eta_t)\), then
\(\sum_{i=1}^N S_{i,t}\le\Dsum<r_t\), so
\(r_{t+1}>r_t>r_0+\eta_t\). Both
\(D^{\eta_t}(r_t)\) and
\(q^{\eta_t}(r_{t+1})-q^{\eta_t}(r_t)\) vanish because the smoothed path
is constant there.  Thus \Cref{lem:smoothing}\textup{(iii)} gives, in all
cases,
\[
 \left\|
 q^{\eta_t}(r_{t+1})-q^{\eta_t}(r_t)
 -D^{\eta_t}(r_t)(r_{t+1}-r_t)
 \right\|
 \le\frac{C}{\eta_t(n_t-1)^2}.
\]

Moreover,
\[
 \left|
 \min\{t+1,n_{t+1}\}-\min\{t,n_t\}
 \right|\le1,
 \qquad
 \eta_t=\sqrt{\chi}\,
 \bigl(\tau_0+\min\{t,n_t\}\bigr)^{-1/2}.
\]
The mean-value theorem and
\Cref{lem:smoothing}\textup{(iv)} therefore give
\[
 \left\|
 q^{\eta_{t+1}}(r_{t+1})-q^{\eta_t}(r_{t+1})
 \right\|
 \le C\alpha_t^{3/2}.
\]

\medskip
\noindent\emph{Substep 2: bound the first-order resource increment.}
The vectors
\(\nabla_q\mathcal W(q_t,z_t)\) and \(D^{\eta_t}(r_t)\) are
\(\calH_t\)-measurable.  Hence
\eqref{eq:conditional-resource-motion} gives the exact identity
\begin{align*}
 &\E\!\left[
 \left\langle
 \nabla_q\mathcal W(q_t,z_t),
 D^{\eta_t}(r_t)(r_{t+1}-r_t)
 \right\rangle
 \mathrel{\Big|}\calH_t\right]\\
 &\quad=
 \frac1{n_t-1}
 \left\langle
 \nabla_q\mathcal W(q_t,z_t),
 D^{\eta_t}(r_t)
 \left(r_t-\sum_{i=1}^N m_i(Y_{i,t})\right)
 \right\rangle.
\end{align*}
Applying the first-order resource-pairing estimate above at
\((r,\eta,\bm x,\bm y,\lambda)
=(r_t,\eta_t,\bm X_t,\bm Y_t,\lambda_t)\) yields
\begin{align*}
 &\E\!\left[
 \left\langle
 \nabla_q\mathcal W(q_t,z_t),
 D^{\eta_t}(r_t)(r_{t+1}-r_t)
 \right\rangle
 \mathrel{\Big|}\calH_t\right]\\
 &\qquad\le\frac{C_{\mathrm{pair}}}{n_t-1}
 \bigl(\widetilde W_t+\|\bm Y_t-\bm X_t\|^2+\eta_t^2\bigr).
\end{align*}

Since
\[
 \alpha_t\ge\frac{\chi}{\tau_0+n_t},
 \qquad
 \tau_0\le\frac{2\chi}{\alpha_0},
\]
the explicit choices in \eqref{eq:motion-cutoffs} give, for
\(\chi\ge\chi_M\) and \(n:=n_t\ge n_0\),
\[
 \frac{C_{\mathrm{pair}}}{(n-1)\alpha_t}
 \le
 \frac{2C_{\mathrm{pair}}}{\alpha_0(n-1)}
 +\frac{2C_{\mathrm{pair}}}{\chi}
 \le\frac{\mu_{\mathrm{upd}}}{8}.
\]
Moreover, for every \(n\ge2\),
\[
 \frac1{(n-1)\alpha_t}
 \le\frac2{\alpha_0(n-1)}+\frac{n}{\chi(n-1)}
 \le\frac2{\alpha_0}+2,
\]
while
\[
 \alpha_t
 \ge\frac{\chi}{2\chi/\alpha_0+n}
 \ge\min\left\{\frac{\alpha_0}{4},\frac1{2n}\right\}.
\]
Since \(n\ge n_0\) implies \(n\ge8\) and
\(n-1\ge2/\sqrt{\alpha_0}\), the last bound gives
\(\alpha_t(n-1)^2\ge1\).  We have therefore established
\begin{equation}
 \frac{C_{\mathrm{pair}}}{n_t-1}
 \le\frac{\mu_{\mathrm{upd}}}{8}\alpha_t,
 \qquad
 \frac1{n_t-1}\le\left(\frac2{\alpha_0}+2\right)\alpha_t,
 \qquad
 \eta_t(n_t-1)\ge1.
 \label{eq:explicit-large-n-bounds}
\end{equation}
The first inequality absorbs the \(\widetilde W_t\)-term. The other two
will be used below.  Since \(\eta_t^2=\alpha_t\),
\begin{align*}
 &\E\!\left[
 \left\langle
 \nabla_q\mathcal W(q_t,z_t),
 D^{\eta_t}(r_t)(r_{t+1}-r_t)
 \right\rangle
 \mathrel{\Big|}\calH_t\right]\\
 &\qquad\le
 \frac{\mu_{\mathrm{upd}}}{8}\alpha_t\widetilde W_t
 +C\alpha_t\|\bm Y_t-\bm X_t\|^2+C\alpha_t^2.
\end{align*}

\medskip
\noindent\emph{Substep 3: bound the Taylor and radius remainders.}
By \eqref{eq:mirror-gradient-control},
\(\|\nabla_q\mathcal W(q_t,z_t)\|
\le C\sqrt{\widetilde W_t}\).  Young's inequality, the direct
resource-motion remainder above, and the large-\(n_t\) bounds give
\begin{align*}
 &\left|
 \left\langle
 \nabla_q\mathcal W(q_t,z_t),
 q^{\eta_t}(r_{t+1})-q^{\eta_t}(r_t)
 -D^{\eta_t}(r_t)(r_{t+1}-r_t)
 \right\rangle
 \right|\\
 &\qquad\le
 \frac{C\sqrt{\widetilde W_t}}{\eta_t(n_t-1)^2}\le
 \frac{\mu_{\mathrm{upd}}}{8}\alpha_t\widetilde W_t
 +\frac{C}{\alpha_t\eta_t^2(n_t-1)^4}
\\
 &\qquad\le
 \frac{\mu_{\mathrm{upd}}}{8}\alpha_t\widetilde W_t
 +C\alpha_t^2.
\end{align*}
\begingroup

To make the final Taylor-remainder estimate explicit, set
\(K_0:=2/\alpha_0+2\).  Since \(\eta_t^2=\alpha_t\), the second inequality
in \eqref{eq:explicit-large-n-bounds} gives
\[
 \frac{1}{\alpha_t\eta_t^2(n_t-1)^4}
 =
 \frac{1}{\alpha_t^2}
 \left(\frac{1}{n_t-1}\right)^4
 \le K_0^4\alpha_t^2.
\]
Thus the last term produced by Young's inequality is
\(O(\alpha_t^2)\).
\endgroup
Similarly, the radius-motion bound above and Young's inequality give
\begin{align*}
 &\left|
 \left\langle
 \nabla_q\mathcal W(q_t,z_t),
 q^{\eta_{t+1}}(r_{t+1})-q^{\eta_t}(r_{t+1})
 \right\rangle
 \right|\\
 &\qquad\le
 C\sqrt{\widetilde W_t}\,\alpha_t^{3/2}
 \le
 \frac{\mu_{\mathrm{upd}}}{8}\alpha_t\widetilde W_t
 +C\alpha_t^2.
\end{align*}
Combining these estimates with the first-order resource bound yields
\begin{align*}
 &\E\!\left[
 \left\langle
 \nabla_q\mathcal W(q_t,z_t),q_{t+1}-q_t
 \right\rangle
 \mathrel{\Big|}\calH_t\right]\\
 &\qquad\le
 \frac{3\mu_{\mathrm{upd}}}{8}\alpha_t\widetilde W_t
 +C\alpha_t^2
 +C\alpha_t\|\bm Y_t-\bm X_t\|^2.
\end{align*}

Finally, the two direct motion bounds and
\(\eta_t(n_t-1)\ge1\) give
\[
 \|q_{t+1}-q_t\|
 \le\frac{C}{n_t-1}+C\alpha_t^{3/2}
 \le C\alpha_t.
\]
Consequently,
\[
 L_{\mathcal W}G\alpha_t\|q_{t+1}-q_t\|
 +\frac{L_{\mathcal W}}2\|q_{t+1}-q_t\|^2
 \le C\alpha_t^2
\]
pathwise.  Adding this estimate to the preceding conditional bound and using
\(3\mu_{\mathrm{upd}}/8\le\mu_{\mathrm{upd}}/2\) proves
\eqref{eq:conditional-comparator-motion-bound}, after enlarging the primitive
constant \(C_{\rm m}\).  No same-period noise factorization has been used.
\endproof

\textbf{Fixed-target restoring calculation for
\Cref{lem:two-contributions}\textup{(i)}.}
\label{app:stabilized-merit-details}

\proof{Proof.}

\medskip
\noindent\emph{Substep 1: expand the full unstabilized pairing.}
Fix \(r\ge0\), \(0\le\lambda\le\lambda_{\max}\), and
\(0\le x_i\le\bar X_i\) for every \(i\in[N]\).
By \eqref{eq:mirror-threshold-derivative} and
\eqref{eq:threshold-field}, the survival factors cancel and the pairing is
\[
 \begin{aligned}
 &\left\langle\nabla_z\mathcal W(q^*(r),(\bm x,\lambda)),
 G^\theta(\bm x,\lambda;r)\right\rangle\\
 &=\sum_i(m_i(x_i)-s_i^*(r))(g_i'(m_i(x_i))+\lambda)
 +(\lambda-\lambda^*(r))\left(r^{\mathrm c}-\sum_i m_i(x_i)\right)\\
 &\quad+\theta(\lambda-\lambda^*(r))(1-F_j(x_j))
 (g_j'(m_j(x_j))+\lambda).
 \end{aligned}
\]
Expand the first two terms by adding and subtracting the target KKT
quantities:
\begin{align*}
 \sum_{i=1}^N
 \bigl(m_i(x_i)-s_i^*(r)\bigr)
 \bigl(g_i'(m_i(x_i))+\lambda\bigr)
 &=
 \sum_{i=1}^N
 \bigl(m_i(x_i)-s_i^*(r)\bigr)
 \bigl(g_i'(m_i(x_i))-g_i'(s_i^*(r))\bigr)\\
 &\qquad+
 \sum_{i=1}^N
 \bigl(m_i(x_i)-s_i^*(r)\bigr)
 \bigl(g_i'(s_i^*(r))+\lambda^*(r)\bigr)\\
 &\qquad+
 \bigl(\lambda-\lambda^*(r)\bigr)
 \sum_{i=1}^N\bigl(m_i(x_i)-s_i^*(r)\bigr),\\[2pt]
 \bigl(\lambda-\lambda^*(r)\bigr)
 \left(r^{\mathrm c}-\sum_{i=1}^Nm_i(x_i)\right)
 &=
 \bigl(\lambda-\lambda^*(r)\bigr)
 \left(r^{\mathrm c}-\sum_{i=1}^Ns_i^*(r)\right)\\
 &\quad-
 \bigl(\lambda-\lambda^*(r)\bigr)
 \sum_{i=1}^N\bigl(m_i(x_i)-s_i^*(r)\bigr).
\end{align*}

Only after this algebraic cancellation do we use
\eqref{eq:KKT-residuals} and
\eqref{eq:capped-target-complementarity}, which give
\[
 s_i^*(r)\bigl(g_i'(s_i^*(r))+\lambda^*(r)\bigr)=0,\qquad
 r^{\mathrm c}-\sum_{i=1}^N s_i^*(r)=(r^{\mathrm c}-r_0)^+,\qquad
 \lambda^*(r)(r^{\mathrm c}-r_0)^+=0.
\]
More explicitly,
\eqref{eq:g-curvature-bounds} implies, for every \(i\),
\[
 \bigl(u-v\bigr)\bigl(g_i'(u)-g_i'(v)\bigr)
 \ge h_i\kappa_i(u-v)^2,
 \qquad u,v\in[0,m_i(\bar X_i)].
\]
Therefore, with \(u=m_i(x_i)\) and \(v=s_i^*(r)\),
\[
 \sum_{i=1}^N
 \bigl(m_i(x_i)-s_i^*(r)\bigr)
 \bigl(g_i'(m_i(x_i))-g_i'(s_i^*(r))\bigr)
 \ge
 \mu_g\|\bm m(\bm x)-\bm s^*(r)\|^2,
 \qquad
 \mu_g:=\min_{1\le i\le N}h_i\kappa_i.
\]
This is the strong-convexity inequality used in the final step below.
\begin{align}
 &\sum_{i=1}^N
 \bigl(m_i(x_i)-s_i^*(r)\bigr)
 \bigl(g_i'(m_i(x_i))+\lambda\bigr)
 +\bigl(\lambda-\lambda^*(r)\bigr)
 \left(r^{\mathrm c}-\sum_{i=1}^Nm_i(x_i)\right)\notag\\
 &=\sum_{i=1}^N
 \bigl(m_i(x_i)-s_i^*(r)\bigr)
 \bigl(g_i'(m_i(x_i))-g_i'(s_i^*(r))\bigr)
 +\sum_{i=1}^N
 \bigl(m_i(x_i)-s_i^*(r)\bigr)
 \bigl(g_i'(s_i^*(r))+\lambda^*(r)\bigr)\notag\\
 &\quad+\bigl(\lambda-\lambda^*(r)\bigr)
 \left(r^{\mathrm c}-\sum_{i=1}^Ns_i^*(r)\right)\notag\\
 &=\sum_{i=1}^N
 \bigl(m_i(x_i)-s_i^*(r)\bigr)
 \bigl(g_i'(m_i(x_i))-g_i'(s_i^*(r))\bigr)
 +\sum_{i=1}^N
 m_i(x_i)\bigl(g_i'(s_i^*(r))+\lambda^*(r)\bigr)
 +\lambda(r^{\mathrm c}-r_0)^+\notag\\
 &\ge
 \mu_g\|\bm m(\bm x)-\bm s^*(r)\|^2
 +\sum_{i=1}^N
 m_i(x_i)\bigl(g_i'(s_i^*(r))+\lambda^*(r)\bigr)
 +\lambda(r^{\mathrm c}-r_0)^+.
 \label{eq:raw-field-merit}
\end{align}

\medskip
\noindent\emph{Substep 2: the anchor restores the multiplier error.}
For the maximum-margin anchor \(j\), the selected KKT path satisfies
\begin{equation}
 g_j'(s_j^*(r))+\lambda^*(r)=0.
 \label{eq:anchor-zero}
\end{equation}
Indeed, for \(r>0\), \(\lambda^*(r)<b_j-c_j\) makes the anchor active, so
its KKT condition is an equality. At \(r=0\),
\(\lambda^*(0)=b_j-c_j=-g_j'(0)\) gives the same equality.
On the safe box,
\[
 1-F_j(x_j)\ge\beta_j,\qquad
 |g_j'(m_j(x_j))-g_j'(s_j^*(r))|
 \le L_{g,j}|m_j(x_j)-s_j^*(r)|,
 \qquad L_{g,j}:=\frac{h_jK_j}{\beta_j^3}.
\]
Using \eqref{eq:anchor-zero}, the anchor pairing first separates exactly
into a positive multiplier square and one cross term:
\begin{align*}
 &\bigl(\lambda-\lambda^*(r)\bigr)\bigl(1-F_j(x_j)\bigr)
 \bigl(g_j'(m_j(x_j))+\lambda\bigr)\notag\\
 &\quad=\bigl(1-F_j(x_j)\bigr)
 \left(|\lambda-\lambda^*(r)|^2
 +\bigl(\lambda-\lambda^*(r)\bigr)
 \bigl(g_j'(m_j(x_j))-g_j'(s_j^*(r))\bigr)\right).
\end{align*}
Since \(1-F_j(x_j)\le1\), the survival lower bound, the preceding Lipschitz
bound, and Young's inequality yield, in that order,
\begin{align}
 &\bigl(\lambda-\lambda^*(r)\bigr)
 \bigl(1-F_j(x_j)\bigr)
 \bigl(g_j'(m_j(x_j))+\lambda\bigr)\notag\\
 &\quad\ge
 \beta_j|\lambda-\lambda^*(r)|^2
 -|\lambda-\lambda^*(r)|\,
  |g_j'(m_j(x_j))-g_j'(s_j^*(r))|\notag\\
 &\quad\ge
 \beta_j|\lambda-\lambda^*(r)|^2
 -L_{g,j}|\lambda-\lambda^*(r)|\,
  |m_j(x_j)-s_j^*(r)|\notag\\
 &\quad\ge
 \frac{\beta_j}{2}|\lambda-\lambda^*(r)|^2
 -\frac{L_{g,j}^2}{2\beta_j}
  |m_j(x_j)-s_j^*(r)|^2.
 \label{eq:anchor-dual-curvature}
\end{align}

\medskip
\noindent\emph{Substep 3: choose \(\theta\) and absorb the cross term.}
Use the choice of \(\theta\) and the joint restoring modulus \(\mu_0\) in
\eqref{eq:restoring-constants}.  The
bound involving \(L_{g,j}\) is the one used for absorption below.
\(\theta\le1\) is only a convenient normalization for the later uniform
field bounds.
Multiplying \eqref{eq:anchor-dual-curvature} by \(\theta\) and adding it to
\eqref{eq:raw-field-merit} gives
\begin{align*}
 &\sum_{i=1}^N
 \bigl(m_i(x_i)-s_i^*(r)\bigr)
 \bigl(g_i'(m_i(x_i))+\lambda\bigr)\\
 &\quad+\bigl(\lambda-\lambda^*(r)\bigr)
 \left(r^{\mathrm c}-\sum_{i=1}^Nm_i(x_i)
 +\theta\bigl(1-F_j(x_j)\bigr)
 \bigl(g_j'(m_j(x_j))+\lambda\bigr)\right)\\
 &\quad\ge
 \left(\mu_g-\frac{\theta L_{g,j}^2}{2\beta_j}\right)
 \|\bm m(\bm x)-\bm s^*(r)\|^2
 +\frac{\theta\beta_j}{2}|\lambda-\lambda^*(r)|^2\\
 &\qquad\quad
 +\sum_{i=1}^Nm_i(x_i)
   \bigl(g_i'(s_i^*(r))+\lambda^*(r)\bigr)
 +\lambda(r^{\mathrm c}-r_0)^+.
\end{align*}
Condition \eqref{eq:restoring-constants} makes the first coefficient at least
\(\mu_g/2\). The definition of \(\mu_0\) then controls both squared errors.
Together with the inner-product identity derived in Substep~1, this is exactly
\eqref{eq:stabilized-merit}.

\endproof

\textbf{Details for the one-step verification.}

This paragraph proves the three model-specific conditions recorded in
Appendix~\ref{app:one-step-verification}, which in turn verify the hypotheses
of \Cref{lem:generic-moving-mirror}.

\proof{Proof.}
For a fixed store \(i\), abbreviate
\[
 S_i(x):=1-F_i(x),\qquad M_i:=m_i(\bar X_i),\qquad
 w_i(u):=S_i\bigl(m_i^{-1}(u)\bigr)^{-2}.
\]
On \([0,M_i]\), the safe-box bounds give
\[
 1\le w_i\le\beta_i^{-2},
 \qquad |w_i'|\le2K_i\beta_i^{-4}.
\]
Let \(H_i(\sigma,x)\) be the \(i\)-th primal summand in
\eqref{eq:joint-lyapunov}.  Its Hessian on the middle branch is
\[
 \begin{pmatrix}
  w_i(\sigma)&-S_i(x)^{-1}\\
  -S_i(x)^{-1}&1+\bigl(m_i(x)-\sigma\bigr)f_i(x)S_i(x)^{-2}
 \end{pmatrix},
\]
while the left and right branches replace the off-diagonal entry by
\(-1\) and \(-S_i(\bar X_i)^{-1}\), respectively, and have lower-right
entry \(1\).  All three Hessians are uniformly bounded.  The values and
first derivatives match at \(x=0\) and \(x=\bar X_i\) by
\eqref{eq:mirror-boundary-slopes}.  Integrating the uniform Hessian bound
along line segments therefore proves the differentiability and Lipschitz
joint-gradient claim after summing over \(i\) and adding the dual quadratic.

For the executable bounds, initialization lies in the safe box.  Moreover,
\(C_t\ge\sum_{i=1}^N I_{i,t}\), so the water-filling map is the infimum of a
nonempty closed sublevel set of a continuous piecewise-linear function of the
adapted inputs and is measurable.  It gives
\[
 I_{i,t}\le Y_{i,t}\le\max\{I_{i,t},X_{i,t}\},
 \qquad \sum_{i=1}^N Y_{i,t}\le C_t.
\]
The inventory recursion and the next coordinatewise projection then preserve
the safe box and adaptedness by induction.  The two possible values in
\eqref{eq:observable-samples} give
\(|\widehat G_{i,t}|\le G_i\), and
\[
 |\widehat G_{0,t}|
 \le\Dsum+\theta|\widehat G_{j,t}|
 \le\Dsum+\theta G_j.
\]
Thus \(\|\widehat G_t\|\le G\).  If \(\alpha_tG\le1\), every coordinate
of the pre-projection update moves by at most one.  Since \(z_t\) lies in the
safe state box and \(\mathcal Z\) extends that box by one in every coordinate,
this gives \(z_t-\alpha_t\widehat G_t\in\mathcal Z\), which is the
pre-projection domain condition used above.

It remains to verify projection compatibility.  Write
\(q_{t+1}=(\bm\sigma_{t+1},\zeta_{t+1})\) and put
\(y_i=m_i^{-1}(\sigma_{i,t+1})\).  The derivative
\(\partial_xH_i(\sigma_{i,t+1},x)\) is nonpositive on \(x\le y_i\) and
nonnegative on \(x\ge y_i\). This follows directly from the three branches
in \eqref{eq:joint-lyapunov}.  If the threshold representation of
\(q_{t+1}\) lies in \(\calK(r_{t+1})\), projecting each primal coordinate
toward \([0,U_i(r_{t+1})]\) cannot increase \(H_i\).  The dual coordinate
satisfies
\[
 |\proj_{[0,\lambda_{\max}]}(\lambda)-\zeta_{t+1}|
 \le|\lambda-\zeta_{t+1}|.
\]
Summing the coordinatewise inequalities gives
\eqref{eq:joint-projection-compatibility}, completing the verification.
\endproof

\textbf{Verification of the conditional mean used in
\eqref{eq:feasibility-bias}.}
Conditionally on \(\calH_t\), demand independence gives
\[
 \begin{aligned}
 \E[\1\{D_{i,t}<Y_{i,t}\}\mid\calH_t]&=F_i(Y_{i,t}),&
 \E[\1\{D_{i,t}\ge Y_{i,t}\}\mid\calH_t]&=1-F_i(Y_{i,t}),\\
 \E[D_{i,t}\wedge Y_{i,t}\mid\calH_t]&=m_i(Y_{i,t}).
 \end{aligned}
\]
Substitution gives the first equality in \eqref{eq:feasibility-bias}.  On
\([0,\bar X_i]\), \(F_i\) is \(K_i\)-Lipschitz and \(m_i\) is
one-Lipschitz.  By the executable bounds in the checklist above, both
\(\bm X_t\) and \(\bm Y_t\) lie in the fixed safe rectangle
\(\prod_i[0,\bar X_i]\).  
Hence the stabilized field is uniformly Lipschitz there, and comparing
\(\bm Y_t\) with \(\bm X_t\) proves the bias bound.

\section{\texorpdfstring{Fluid Geometry and Remaining Proofs}{Fluid Geometry and Remaining Proofs}}
\label{app:supporting-proofs}
This appendix proves the remaining fluid-geometry,
implementation-gap, and tuning claims and sketches the endpoint-envelope
extension. Constants depend only on
fixed primitives and horizon-independent tuning.

\subsection{\texorpdfstring{Fluid Benchmark and Re-Solving-Path Geometry}{Fluid Benchmark and Re-Solving-Path Geometry}}
\label{app:proof-fluid-benchmark}

\subsubsection*{Proof of \Cref{prop:fluid-benchmark}}

\proof{Proof.}
We first verify the two direct calculations stated before the proposition.
The identities
\[
 Y_{i,t}-I_{i,t}=S_{i,t}+I_{i,t+1}-I_{i,t},
 \qquad
 B_{T+1}=W-\sum_{t=1}^T\sum_{i=1}^N(Y_{i,t}-I_{i,t})
\]
together with \(I_{i,1}=0\) and \(c_i=k_i-w\) give
\begin{align*}
 &\sum_{t=1}^T\sum_{i=1}^N
 \bigl(k_i(Y_{i,t}-I_{i,t})+h_iI_{i,t+1}
 +b_i(D_{i,t}-Y_{i,t})^+\bigr)
 +wB_{T+1}-\sum_{i=1}^Nc_iI_{i,T+1}\\
 &\qquad =wW+\sum_{t=1}^T\sum_{i=1}^N
 \bigl(c_i(D_{i,t}\wedge Y_{i,t})+h_i(Y_{i,t}-D_{i,t})^+
 +b_i(D_{i,t}-Y_{i,t})^+\bigr).
\end{align*}
Since \(\bm Y_t\) is chosen before current demand,
the conditional expectation of the bracketed store--period term on the
right-hand side is \(\ell_i(Y_{i,t})\).  Taking expectations proves
\eqref{eq:exact-cost}.

Moreover, shipments only relocate inventory, whereas sales deplete it.
Summing the aggregate system-inventory recursion therefore gives
\[
 B_{T+1}+\sum_{i=1}^NI_{i,T+1}
 =W-\sum_{t=1}^T\sum_{i=1}^NS_{i,t}.
\]
Taking expectations and using
\(\E^\pi[S_{i,t}\mid\calH_t^\pi]=m_i(Y_{i,t})\) proves
\eqref{eq:total-sales-budget}.

The expected-sales transformation in \Cref{sec:fluid-geometry} makes
\(g_i\) convex and writes
\[
 v(r)=\min\left\{\sum_i g_i(s_i):\sum_i s_i\le r,\ 0\le s_i\le\mu_i\right\}.
\]
Hence \(v\) is convex by mixing feasible allocations and nonincreasing
because its feasible set expands with \(r\).

If \(0\le y_i\le\bar d_i\), then \(\bm y\) is feasible for the program
defining \(v(\sum_{i=1}^N m_i(y_i))\).  If some \(y_i>\bar d_i\), clipping it
to \(\bar d_i\) leaves \(m_i(y_i)\) unchanged and weakly decreases
\(\ell_i(y_i)\).  Therefore, in either case,
\begin{equation}
 \sum_{i=1}^N\ell_i(y_i)
 \ge v\!\left(\sum_{i=1}^N m_i(y_i)\right)
 \qquad\text{pointwise in }\bm y.
 \label{eq:pointwise-fluid-value}
\end{equation}

Apply \eqref{eq:pointwise-fluid-value} to the random action
\(\bm Y_t\).  Taking expectations and using Jensen again,
\[
 \E^\pi\sum_{i=1}^N\ell_i(Y_{i,t})
 \ge\E^\pi v\!\left(\sum_{i=1}^N m_i(Y_{i,t})\right)
 \ge v\!\left(\E^\pi\sum_{i=1}^N m_i(Y_{i,t})\right).
\]
Summing over periods and applying Jensen across periods gives
\begin{align*}
 \sum_{t=1}^T\E^\pi\sum_{i=1}^N\ell_i(Y_{i,t})
 &\ge\sum_{t=1}^T
 v\!\left(\E^\pi\sum_{i=1}^N m_i(Y_{i,t})\right)\\
 &\ge T v\!\left(
 \frac1T\sum_{t=1}^T\E^\pi\sum_{i=1}^N m_i(Y_{i,t})\right)
 \ge Tv(\gamma),
\end{align*}
where the last inequality uses
\(\sum_{t=1}^T\E^\pi\sum_{i=1}^N m_i(Y_{i,t})\le W=\gamma T\) from
\eqref{eq:total-sales-budget} and the fact that \(v\) is nonincreasing.
Finally, \eqref{eq:exact-cost} implies
\(J_T^\pi\ge wW+Tv(\gamma)\).  Taking the infimum over all admissible
policies proves the oracle bound.
\endproof

\label{app:proof-supporting-geometry}

\subsubsection*{\texorpdfstring{Derivation of sales-space curvature}{Derivation of sales-space curvature}}

\proof{Derivation.}
For \(y=m_i^{-1}(s)\), \eqref{eq:expected-sales} and
\eqref{eq:expected-loss} give
\[
 m_i'(y)=1-F_i(y),\qquad m_i''(y)=-f_i(y),
\]
and
\[
 \ell_i'(y)=(h_i+b_i-c_i)F_i(y)-(b_i-c_i),\qquad
 \ell_i''(y)=(h_i+b_i-c_i)f_i(y).
\]
Thus \(g_i'(s)=\ell_i'(y)/m_i'(y)\).  Differentiating once more,
\begin{align*}
 g_i''(s)
 &=\frac{\ell_i''(y)m_i'(y)-\ell_i'(y)m_i''(y)}{m_i'(y)^3}\\
 &=\frac{f_i(y)\bigl((h_i+b_i-c_i)(1-F_i(y))
 +(h_i+b_i-c_i)F_i(y)-(b_i-c_i)\bigr)}
 {(1-F_i(y))^3}\\
 &=\frac{h_i f_i(y)}{(1-F_i(y))^3}.
\end{align*}
\endproof

\subsubsection*{\texorpdfstring{Proof of \Cref{lem:kkt-path}}{Proof of the resource-indexed KKT-path lemma}}

\proof{Proof.}
For the proof, define the fixed-price response
\begin{equation}
 \phi_i(\lambda):=
 \begin{cases}
 m_i\!\left(
 F_i^{-1}\!\left(
 \dfrac{b_i-c_i-\lambda}{h_i+b_i-c_i-\lambda}
 \right)\right),&\lambda<b_i-c_i,\\[6pt]
 0,&\lambda\ge b_i-c_i,
 \end{cases}
 \qquad
 \Phi(\lambda):=\sum_{i=1}^N\phi_i(\lambda),
 \label{eq:best-response}
\end{equation}
and write \(\bm\phi(\lambda):=(\phi_i(\lambda))_{i=1}^N\).  Also let
\[
 \mathcal J_{\max}:=\{i:b_i-c_i=\lambda_{\max}\},\qquad
 L_v:=\left(
 \sum_{j\in\mathcal J_{\max}}
 \frac{p_j^2}{(h_j+b_j-c_j)K_j}
 \right)^{-1}.
\]
These quantities depend only on the fixed model primitives.

For fixed \(\lambda\ge0\), the derivative of
\(g_i(s)+\lambda s\) at \(s=0\) is \(-(b_i-c_i)+\lambda\).
The strict
convexity in \eqref{eq:g-curvature-bounds} therefore makes \(s=0\)
the unique minimizer when \(\lambda\ge b_i-c_i\).  When
\(\lambda<b_i-c_i\), the unique interior stationary point solves
\[
 \frac{(h_i+b_i-c_i)F_i(y)-(b_i-c_i)}
 {1-F_i(y)}+\lambda=0,
\]
which gives the formula for \(\phi_i(\lambda)\) in
\eqref{eq:best-response}.  For the derivative calculation below, write
\[
 y_i^\lambda:=m_i^{-1}(\phi_i(\lambda)).
\]
This quantile lies strictly below \(\bar d_i\), so the upper sales constraint
is slack.  As \(\lambda\uparrow b_i-c_i\), it decreases to zero. Hence each
\(\phi_i(\lambda)\), and therefore \(\Phi\), is continuous at every
breakpoint.
At the endpoints, the definitions give
\(\sum_{i=1}^N\phi_i(0)=r_0\), while
\(\phi_i(\lambda_{\max})=0\) for every \(i\).

For \(\lambda<b_i-c_i\), implicit differentiation gives
\[
 \frac{\dd\phi_i(\lambda)}{\dd\lambda}
 =-\frac{(1-F_i(y_i^\lambda))^2}
 {(h_i+b_i-c_i-\lambda)f_i(y_i^\lambda)}.
\]
Every maximum-margin store is active for
\(\lambda<\lambda_{\max}\).  For such a store \(j\),
\[
 1-F_j(y_j^\lambda)
 =\frac{h_j}{h_j+b_j-c_j-\lambda}\ge p_j,\qquad
 (h_j+b_j-c_j-\lambda)f_j(y_j^\lambda)
 \le(h_j+b_j-c_j)K_j,
\]
so summing over \(\mathcal J_{\max}\) yields
\[
 -\Phi'(\lambda)\ge L_v^{-1}>0
 \quad\text{a.e.\ on }[0,\lambda_{\max}).
\]
Conversely,
\[
 |\phi_i'(\lambda)|\le\frac1{h_i\kappa_i}
\]
where the derivative exists.  Integrating on the finitely many intervals
cut out by the distinct values of \(b_i-c_i\) shows that \(\Phi\) is Lipschitz
and, for \(0\le\lambda<\lambda'\le\lambda_{\max}\),
\[
 \Phi(\lambda)-\Phi(\lambda')
 \ge L_v^{-1}(\lambda'-\lambda)>0.
\]
Thus \(\Phi\) is strictly decreasing from \(r_0\) to \(0\), its inverse is
well defined, and
\[
 |\Phi^{-1}(r')-\Phi^{-1}(r)|
 \le L_v|r'-r|,
 \qquad r,r'\in[0,r_0].
\]

We now establish the explicit representation
\begin{equation}
 (\bm s^*(r),\lambda^*(r))=
 \begin{cases}
 (\bm\phi(\Phi^{-1}(r)),\Phi^{-1}(r)),&0<r<r_0,\\
 (\bm0,\lambda_{\max}),&r=0,\\
 (\bm\phi(0),0),&r\ge r_0.
 \end{cases}
 \label{eq:kkt-path}
\end{equation}

For \(0<r<r_0\), choose \(\lambda=\Phi^{-1}(r)\).  The coordinate
minimizers sum to \(r\) and satisfy the KKT conditions, so strict convexity
gives the unique global optimizer.
At the other endpoint \(r=0\), primal feasibility forces
\(\bm s=\bm0\).  Because every sales coordinate is then at its lower
boundary, stationarity requires
\begin{equation*}
 g_i'(0)+\lambda
 =\lambda-(b_i-c_i)\ge0
 \qquad\text{for every }i.
\end{equation*}
Equivalently, \(\lambda\ge\max_i(b_i-c_i)=\lambda_{\max}\).  Conversely,
every \(\lambda\ge\lambda_{\max}\) satisfies these stationarity inequalities,
and complementary slackness also holds because
\(0-\sum_{i=1}^N s_i=0\).  Hence the KKT multiplier is not unique at
\(r=0\): its
admissible set is \([\lambda_{\max},\infty)\).  We select its smallest element,
\(\lambda^*(0)=\lambda_{\max}\).  This convention also matches
\(\Phi^{-1}(r)\uparrow\lambda_{\max}\) as \(r\downarrow0\), so the selected dual
target joins continuously to the unique multiplier for \(r>0\).  This establishes
\eqref{eq:kkt-path}, uniqueness
of the primal target, and minimality of the displayed multiplier.

The inverse bound is exactly the multiplier bound
in \Cref{lem:kkt-path}\textup{(iii)} on \([0,r_0]\), including \(r=0\) by
the continuous endpoint selection above.  The same bound holds globally because
\(\lambda^*(r)=0\) on \([r_0,\infty)\).  Since
\(|\phi_i'(\lambda)|\le(h_i\kappa_i)^{-1}\) wherever the derivative exists,
the representation in \eqref{eq:kkt-path} gives, for all \(r,r'\ge0\),
\[
 |s_i^*(r')-s_i^*(r)|
 \le \frac{L_v}{h_i\kappa_i}|r'-r|.
\]
Thus \(q^*\) is globally Lipschitz.  Finally,
\(\phi_i(\lambda)\) is nonincreasing in \(\lambda\), while
\(\lambda^*(r)\) is nonincreasing in \(r\). Hence each \(s_i^*(r)\) is
nondecreasing.  Integration across the finitely many breakpoints makes all
of these conclusions valid when several breakpoints tie.

The coordinate minimization above also gives \(s_i^*(r)<\mu_i\), the
storewise activity rule, and the KKT residuals in
\eqref{eq:KKT-residuals}.  The representation
\eqref{eq:kkt-path} yields
\(\sum_i s_i^*(r)=\min\{r,r_0\}\) and all endpoint claims in
\Cref{lem:kkt-path}\textup{(ii)}.  Finally, standard value sensitivity gives
\(v'(r)=-\lambda^*(r)\).  Continuity of the selected multiplier makes
\(v\) continuously differentiable, and the multiplier Lipschitz bound gives
the stated Lipschitz bound for \(v'\).
\endproof

\subsection{\texorpdfstring{Implementation Gap}{Implementation Gap}}
\label{app:implementation-gap}

\subsubsection*{Proof of \Cref{lem:implementation-gap}}

\proof{Proof.}
For every store and period, define the excess inventory above the preferred
threshold by
\begin{equation*}
 e_{i,t}:=(I_{i,t}-X_{i,t})^+.
\end{equation*}
Also let
\[
 H_0:=\left\lceil\sum_{i=1}^Np_i^{-1}\right\rceil,
 \qquad
 \bar X_{\max}:=\max_{1\le i\le N}\bar X_i.
\]
We organize the proof in four steps.

\medskip
\noindent\textbf{Step 1: Bound preferred-path variation and the physical
state.}
For \(t<T\), the update in \Cref{alg:rapdl} and the feasibility identity
\(X_{i,t}=\proj_{[0,U_i(r_t)]}(X_{i,t})\) imply
\begin{align}
 |X_{i,t+1}-X_{i,t}|
 &\le
 \left|
  \proj_{[0,U_i(r_{t+1})]}(X_{i,t}-\alpha_t\widehat G_{i,t})
  -\proj_{[0,U_i(r_{t+1})]}(X_{i,t})
 \right|\notag\\
 &\quad+
 \left|
  \proj_{[0,U_i(r_{t+1})]}(X_{i,t})
  -\proj_{[0,U_i(r_t)]}(X_{i,t})
 \right|\notag\\
 &\le
 \alpha_t|\widehat G_{i,t}|
 +|U_i(r_{t+1})-U_i(r_t)|.
 \label{eq:appendix-threshold-one-step}
\end{align}
The last inequality uses nonexpansiveness of projection onto a fixed interval
and the elementary endpoint bound
\[
 |\proj_{[0,a]}(x)-\proj_{[0,b]}(x)|\le|a-b|,
 \qquad a,b,x\ge0.
\]

If \(r_t\le\Dsum\), the one-Lipschitz cap and
\(0\le\sum_iS_{i,t}\le\Dsum\) give
\[
 |r_{t+1}^{\mathrm c}-r_t^{\mathrm c}|
 \le|r_{t+1}-r_t|
 =\frac{|r_t-\sum_iS_{i,t}|}{n_t-1}
 \le\frac{\Dsum}{n_t-1}.
\]
If \(r_t>\Dsum\), then \(r_{t+1}>r_t>\Dsum\), so
\(r_{t+1}^{\mathrm c}=r_t^{\mathrm c}=\Dsum\).  Thus, in both cases,
\begin{equation*}
 |r_{t+1}^{\mathrm c}-r_t^{\mathrm c}|
 \le\frac{\Dsum}{n_t-1}.
\end{equation*}
Because \(U_i(r)=\bar X_i\wedge r^{\mathrm c}/p_i\),
\[
 |U_i(r_{t+1})-U_i(r_t)|
 \le\frac{\Dsum}{p_i(n_t-1)}.
\]
Combining this inequality with \eqref{eq:appendix-threshold-one-step} and
\(|\widehat G_{i,t}|\le G_i\) yields
\[
 |X_{i,t+1}-X_{i,t}|
 \le G_i\alpha_t+\frac{\Dsum}{p_i(n_t-1)}.
\]
The two-sided step-size profile and \(n_t-1=T-t\) give
\begin{align*}
 \sum_{t=1}^{T-1}\alpha_t
 &\le2\chi\sum_{k=1}^{\lceil T/2\rceil}\frac1{\tau_0+k}
 \le C\log T,\\
 \sum_{t=1}^{T-1}\frac1{n_t-1}
 &=\sum_{k=1}^{T-1}\frac1k
 \le C\log T.
\end{align*}
Consequently,
\begin{equation*}
 \sum_{t=1}^{T-1}|X_{i,t+1}-X_{i,t}|
 \le C\log T
\end{equation*}
pathwise.

The water-filling formula also gives
\[
 I_{i,t+1}\le Y_{i,t}
 \le\max\{I_{i,t},X_{i,t}\}.
\]
Since \(I_{i,1}=0\) and \(0\le X_{i,t}\le\bar X_i\), induction yields
\begin{equation}
 0\le I_{i,t},Y_{i,t},e_{i,t}\le\bar X_i
 \qquad(1\le t\le T).
 \label{eq:appendix-physical-safe-box}
\end{equation}

\medskip
\noindent\textbf{Step 2: Telescope expected drainage.}
The water-filling formula and the physical update
\(I_{i,t+1}=(Y_{i,t}-D_{i,t})^+\) give
\[
 \begin{cases}
  Y_{i,t}=I_{i,t},\quad
  I_{i,t+1}=(I_{i,t}-D_{i,t})^+,
  & I_{i,t}>X_{i,t},\\
  Y_{i,t}\le X_{i,t},\quad
  I_{i,t+1}\le X_{i,t},
  & I_{i,t}\le X_{i,t}.
 \end{cases}
\]
In the first case,
\[
 e_{i,t+1}
 \le(e_{i,t}-D_{i,t})^+
 +|X_{i,t+1}-X_{i,t}|,
\]
whereas in the second case \(e_{i,t}=0\) and
\[
 e_{i,t+1}
 \le|X_{i,t+1}-X_{i,t}|.
\]
Hence, in both cases,
\begin{equation*}
 e_{i,t+1}
 \le(e_{i,t}-D_{i,t})^+
 +|X_{i,t+1}-X_{i,t}|.
\end{equation*}
Using \(a-(a-d)^+=a\wedge d\), this implies
\begin{equation}
 e_{i,t}\wedge D_{i,t}
 \le e_{i,t}-e_{i,t+1}
 +|X_{i,t+1}-X_{i,t}|.
 \label{eq:appendix-one-period-drainage}
\end{equation}
Conditionally on \(\calH_t\), \(e_{i,t}\) is fixed and \(D_{i,t}\) has law
\(F_i\), so
\[
 \E[e_{i,t}\wedge D_{i,t}\mid\calH_t]=m_i(e_{i,t}).
\]
Taking expectations and summing \eqref{eq:appendix-one-period-drainage}
through \(T-1\) gives
\begin{align*}
 \sum_{t=1}^{T-1}\E m_i(e_{i,t})
 &\le
 \E e_{i,1}-\E e_{i,T}
 +\sum_{t=1}^{T-1}\E|X_{i,t+1}-X_{i,t}|\notag\\
 &\le C\log T,
\end{align*}
where \(e_{i,1}=0\).  Since \(m_i(e_{i,T})\le e_{i,T}\le\bar X_i\),
\begin{equation}
 \sum_{t=1}^T\E m_i(e_{i,t})
 \le\bar X_i+C\log T.
 \label{eq:appendix-cumulative-expected-drainage}
\end{equation}

\medskip
\noindent\textbf{Step 3: Convert expected drainage into cumulative excess.}
From \(f_i\le K_i\)
and \(1=\int_0^{\bar d_i}f_i(u)\,\dd u\), one has
\(K_i^{-1}\le\bar d_i\).  For every \(0\le y\le\bar X_i\),
\(F_i(u)\le K_i u\) implies
\[
 m_i(y)
 \ge\int_0^{y\wedge K_i^{-1}}(1-K_i u)\,\dd u
 \ge
 \begin{cases}
  y/2, & 0\le y\le K_i^{-1},\\
  1/(2K_i), & y>K_i^{-1}.
 \end{cases}
\]
In the first range, \(y\le2m_i(y)\).  In the second,
\(y\le\bar X_i\le2K_i\bar X_i m_i(y)\).  Therefore, throughout the safe box,
\begin{equation}
 y\le2\max\{1,K_i\bar X_i\}\,m_i(y).
 \label{eq:appendix-drainage-coercivity}
\end{equation}
Applying \eqref{eq:appendix-drainage-coercivity} to \(y=e_{i,t}\) and using
\eqref{eq:appendix-cumulative-expected-drainage} yields
\begin{equation}
 \sum_{t=1}^T\E e_{i,t}\le C\log T.
 \label{eq:appendix-cumulative-excess}
\end{equation}

\medskip
\noindent\textbf{Step 4: Convert cumulative excess into the implementation
gap.}
Since \(X_{i,t}-\nu_t\le X_{i,t}\), the water-filling formula gives
\[
 Y_{i,t}>X_{i,t}
 \quad\Longleftrightarrow\quad
 Y_{i,t}=I_{i,t}>X_{i,t}.
\]
Thus
\begin{equation*}
 (Y_{i,t}-X_{i,t})^+=e_{i,t}.
\end{equation*}

The clipping rule gives \(X_{i,t}\le U_i(r_t)\le r_t/p_i\).  Hence, whenever
\(n_t\ge H_0\),
\[
 \sum_{i=1}^NX_{i,t}
 \le r_t\sum_{i=1}^Np_i^{-1}
 \le n_tr_t=C_t.
\]
If water-filling binds, then \(\sum_iY_{i,t}=C_t\), and
\[
 \sum_i(Y_{i,t}-X_{i,t})^+
 -\sum_i(X_{i,t}-Y_{i,t})^+
 =C_t-\sum_iX_{i,t}\ge0.
\]
If it does not bind, then \(\nu_t=0\) and
\(Y_{i,t}=\max\{I_{i,t},X_{i,t}\}\ge X_{i,t}\), so the downward discrepancy
is zero.  In either case, for \(n_t\ge H_0\),
\begin{equation}
 \|\bm Y_t-\bm X_t\|_1
 \le2\sum_{i=1}^N(Y_{i,t}-X_{i,t})^+
 =2\sum_{i=1}^Ne_{i,t}.
 \label{eq:appendix-l1-by-excess}
\end{equation}
There are at most \(H_0-1\) periods with \(n_t<H_0\), and
\eqref{eq:appendix-physical-safe-box} gives
\(\|\bm Y_t-\bm X_t\|_1\le\sum_i\bar X_i\) in every such period.
Combining this bound with \eqref{eq:appendix-cumulative-excess} and
\eqref{eq:appendix-l1-by-excess} gives
\begin{equation}
 \sum_{t=1}^T\E\|\bm Y_t-\bm X_t\|_1
 \le
 2\sum_{t=1}^T\sum_{i=1}^N\E e_{i,t}
 +(H_0-1)\sum_{i=1}^N\bar X_i
 \le C\log T.
 \label{eq:appendix-implementation-l1}
\end{equation}
Finally,
\[
 \|\bm Y_t-\bm X_t\|^2
 \le\|\bm Y_t-\bm X_t\|_\infty
      \|\bm Y_t-\bm X_t\|_1
 \le\bar X_{\max}\|\bm Y_t-\bm X_t\|_1.
\]
Summing and applying \eqref{eq:appendix-implementation-l1} prove the same
order for the squared norm.  By the definition
\eqref{eq:implementation-gap-block},
\[
 \mathsf{IMP}_T\le C\log T,
\]
which proves \eqref{eq:implementation-gap-bound}.
\endproof

\subsection{\texorpdfstring{Tuning Construction}{Tuning Construction}}
\label{app:tuning}

\subsubsection*{Verification of \Cref{rem:computable-tuning}}

\proof{Proof.}
First compute from the known primitives
\[
 p_i=\frac{h_i}{h_i+b_i-c_i},\qquad
 \bar X_i=\bar d_i-\frac{p_i}{2K_i},\qquad
 \beta_i=\frac{\kappa_i p_i}{2K_i},
\]
\[
 \mu_g:=\min_{1\le i\le N}h_i\kappa_i,\qquad
 L_{g,j}:=\frac{h_jK_j}{\beta_j^3},
\]
and choose
\[
 \theta:=\frac12\min\left\{1,\frac{\mu_g\beta_j}{L_{g,j}^2}\right\}.
\]
This choice satisfies \eqref{eq:restoring-constants}.

The constants entering the motion cutoff can also be fixed explicitly.  Put
\[
 \beta_{\min}:=\min_{i\in[N]}\beta_i,
 \qquad
 C_W:=\frac{1}{2\beta_{\min}^2},
 \qquad
 \mathcal J_{\max}:=\{i:b_i-c_i=\lambda_{\max}\},
\]
and
\[
 L_v:=\left(
 \sum_{j\in\mathcal J_{\max}}
 \frac{p_j^2}{(h_j+b_j-c_j)K_j}
 \right)^{-1},
 \qquad
 L_q:=L_v\left(1+\sum_{i=1}^N(h_i\kappa_i)^{-2}\right)^{1/2},
 \qquad
 C_{\mathrm{sm}}:=3L_q.
\]
The metric calculation in \Cref{lem:mirror-toolkit} validates the displayed
choice of \(C_W\).  The proof of \Cref{lem:kkt-path} gives the Lipschitz
constant \(L_q\), and the fixed kernel in \eqref{eq:fixed-kernel} has total
variation \(V_\rho=3\), so this \(C_{\mathrm{sm}}\) satisfies
\Cref{lem:smoothing}.  The field and joint-smoothness envelopes
\(G,L_{\mathcal W}\) are likewise obtained from the safe-box bounds. They are
proof envelopes rather than additional tuning inputs.  All quantities in
this paragraph use only the known primitives.  With the same notation as in
Section~A.1, take
\[
 \alpha_0:=\min\{1,G^{-1}\},
 \qquad
 \mu_0:=\frac12\min\{\mu_g,\theta\beta_j\},
 \qquad
 \mu_{\mathrm{upd}}:=\frac{\mu_0}{8C_W}>0.
\]
Then \(\alpha_0G\le1\), and the state-update argument gives the contraction
in \eqref{eq:direct-state-update-bound} with modulus
\(\mu_{\mathrm{upd}}\).

For any \(\chi\ge1\) and
\(\chi/\alpha_0\le\tau_0\le2\chi/\alpha_0\), the step profile satisfies,
whenever \(n=n_t\ge2\),
\[
 \frac1{(n-1)\alpha_t}
 \le\frac{\tau_0+n}{\chi(n-1)}
 \le\frac2{\alpha_0}+2.
\]
Take \(C_{\mathrm{pair}}\) from the explicit formula
\eqref{eq:explicit-c-pair}.  It depends only on
\((N,\bm h,\bm b,\bm c,\bm\kappa,\bm K)\).  Choose
\[
\begin{aligned}
 \mu_{\mathrm{rec}}&:=\min\{\mu_{\mathrm{upd}}/2,1/2\},\qquad
 \chi_M:=\max\left\{1,\frac{32C_{\mathrm{pair}}}{\mu_{\mathrm{upd}}}\right\},\\
 \chi&:=1+\max\left\{\chi_M,\frac3{\mu_{\mathrm{rec}}}\right\},
 \qquad
 \tau_0:=\left\lceil\frac\chi{\alpha_0}\right\rceil.
\end{aligned}
\]
and set
\[
 n_0:=\left\lceil\max\left\{
 8,
 1+\frac{32C_{\mathrm{pair}}}{\mu_{\mathrm{upd}}\alpha_0},
 1+\frac2{\sqrt{\alpha_0}}
 \right\}\right\rceil.
\]
These choices give \(\chi\ge\chi_M\),
\(\mu_{\mathrm{rec}}\chi>3\),
\(\chi/\alpha_0\le\tau_0\le2\chi/\alpha_0\), and
\(\mu_{\mathrm{rec}}\alpha_0\le1/2\).

Substituting these choices into the bounds used to establish
\eqref{eq:explicit-large-n-bounds} verifies all three inequalities in that
display.  Hence the state-update, motion, and Taylor remainders in
\eqref{eq:mirror-recursion} have computable primitive-only constants.  This
constructs horizon-independent tuning without using the demand laws.
\endproof

\subsection{\texorpdfstring{Fixed Support Envelopes: An Extension Sketch}
{Fixed Support Envelopes: An Extension Sketch}}
\label{app:endpoint-envelope}

This subsection sketches how the proof architecture of \Cref{thm:main} can be
adapted when exact support endpoints are replaced by fixed certified
envelopes.  The purpose is to isolate the additional localization mechanism.
\Cref{thm:main} remains the formal guarantee for the direct safe-box
formulation.

Suppose that the policy is supplied fixed, horizon-independent envelopes
\begin{equation*}
 D_i^{\mathrm{up}}\ge\bar d_i,
 \qquad
 D^{\mathrm{up}}:=\sum_{i=1}^ND_i^{\mathrm{up}}<\infty.
\end{equation*}
Under \eqref{eq:density-bounds}, the conservative choice
\(D_i^{\mathrm{up}}=1/\kappa_i\) is valid because
\(1=\int_0^{\bar d_i}f_i\ge\kappa_i\bar d_i\).  Replace the
support-dependent caps by
\begin{equation*}
 r^{\mathrm{c,up}}:=\min\{r,D^{\mathrm{up}}\},
 \qquad
 U_i^{\mathrm{up}}(r):=\min\left\{
 D_i^{\mathrm{up}},\frac{r^{\mathrm{c,up}}}{p_i}\right\},
 \qquad
 \calK^{\mathrm{up}}(r):=
 \prod_i[0,U_i^{\mathrm{up}}(r)]\times[0,\lambda_{\max}].
\end{equation*}
The envelope implementation \(\RAPDL^{\mathrm{up}}\) uses the same recursion
as \Cref{alg:rapdl}, with these caps and with
\[
 \widehat G_{0,t}^{\mathrm{up}}
 :=r_t^{\mathrm{c,up}}-\sum_{i=1}^NS_{i,t}
 +\theta\widehat G_{j,t}.
\]
Let \(G^{\theta,\mathrm{up}}\) denote the corresponding conditional-mean
field, obtained from \(G^\theta\) by replacing \(r^{\mathrm c}\) with
\(r^{\mathrm{c,up}}\).

\textbf{What is inherited from the main proof.}
Because \(D^{\mathrm{up}}\ge\sum_i\bar d_i\ge r_0\), replacing
\(r^{\mathrm c}\) by \(r^{\mathrm{c,up}}\) affects only resource states at
which the fluid solution is already on its slack branch.  The fluid target,
the KKT path, and the backward-smoothed comparator are therefore unchanged.
The projected
moving-comparator inequality in \Cref{lem:generic-moving-mirror}, the
comparator-motion argument in \Cref{lem:two-contributions}\textup{(ii)}, the
reciprocal-step summations in Appendix~\ref{sec:tracking-assembly}, and the
Bellman telescope can all be reused.  The implementation-gap proof also
transfers after replacing its support-dependent inputs by the envelope bounds
below.  The genuinely new issue is that a projected threshold may lie beyond
the true support, where the original global survival and mirror-metric bounds
are unavailable.

\textbf{Implementation-gap ingredient.}
The four changed inputs to the proof of \Cref{lem:implementation-gap} are
\begin{align*}
 |\widehat G_{0,t}^{\mathrm{up}}|
 &\le D^{\mathrm{up}}+\theta G_j,&
 0\le X_{i,t},I_{i,t},Y_{i,t}&\le D_i^{\mathrm{up}},\\
 |U_i^{\mathrm{up}}(r_{t+1})-U_i^{\mathrm{up}}(r_t)|
 &\le\frac{D^{\mathrm{up}}}{p_i(n_t-1)},&
 y&\le2\max\{1,K_iD_i^{\mathrm{up}}\}\,m_i(y).
\end{align*}
The cap-motion case split is the same as before: when
\(r_t>D^{\mathrm{up}}\), both capped rates lie on the constant branch.
Consequently, the preferred-path variation and excess-inventory recursion in
Appendix~\ref{app:implementation-gap} give, with an envelope-dependent
constant,
\begin{equation}
 \sum_{t=1}^T\E\|\bm Y_t-\bm X_t\|_1
 +\sum_{t=1}^T\E\|\bm Y_t-\bm X_t\|^2
 \le C_{\mathrm{up}}\log T.
 \label{eq:envelope-implementation-sketch}
\end{equation}

\textbf{The additional localization mechanism.}
For the analysis only, introduce the barriers
\[
 x_i^\star:=F_i^{-1}(1-p_i),\qquad
 x_i^{(1)}:=F_i^{-1}(1-p_i/4),\qquad
 x_i^{(2)}:=F_i^{-1}(1-p_i/8).
\]
The density upper bound gives
\begin{equation*}
 x_i^{(1)}-x_i^\star\ge\frac{3p_i}{4K_i},\qquad
 x_i^{(2)}-x_i^{(1)}\ge\delta_i:=\frac{p_i}{8K_i},\qquad
 1-F_i(x)\ge\beta_i^{\mathrm{loc}}:=\frac{p_i}{8}
 \quad(x\le x_i^{(2)}).
\end{equation*}
Every exact or backward-smoothed comparator threshold lies below
\(x_i^\star<x_i^{(1)}\).

Let \(\overline H_i(\sigma,x)\) agree with the sales-divergence block
\(\mathcal B_i(\sigma,m_i(x))\) on \([0,x_i^{(2)}]\), and use on
\([-1,D_i^{\mathrm{up}}+1]\) the same unit-curvature quadratic continuation
as in \eqref{eq:joint-lyapunov}, with \(\bar X_i\) and
\(m_i(\bar X_i)\) replaced by \(x_i^{(2)}\) and \(m_i(x_i^{(2)})\).
For constants \(A_i\ge1\), define
\begin{equation*}
 \mathcal W^{\mathrm{loc}}
 ((\bm\sigma,\zeta),(\bm x,\lambda))
 :=\sum_{i=1}^N\left[
 \overline H_i(\sigma_i,x_i)
 +\frac{A_i}{2}(x_i-x_i^{(1)})_+^2\right]
 +\frac12(\lambda-\zeta)^2.
\end{equation*}
Because every comparator threshold is below \(x_i^{(1)}\), the base
continuation and the hinge are monotone away from the comparator. Projection
onto \(\calK^{\mathrm{up}}(r)\) therefore cannot increase
\(\mathcal W^{\mathrm{loc}}\).  The matched continuation and the
\(C^{1,1}\) hinge also give a Lipschitz joint gradient on the fixed extended
box.  The reused step-size check, with the stochastic-direction bound replaced
by its fixed-envelope counterpart, keeps pre-projection points in that box.
On the interior region
\(\mathcal G:=\{\bm x:x_i\le x_i^{(2)}\ \forall i\}\), the proofs of
\Cref{lem:mirror-toolkit} and
\Cref{lem:two-contributions}\textup{(i)} apply with survival lower bound
\(\beta_i^{\mathrm{loc}}\).  The hinge handles the complement.  Indeed, if
\(e_i:=(x_i-x_i^{(1)})_+>0\), then
\[
 G_i^\theta(\bm x,\lambda;r)\ge\frac{3h_i}{4}.
 \]
On \(\mathcal G^{\mathrm c}\), at least one \(e_i\ge\delta_i\).  The base
terms have finite worst-case envelopes on the fixed box.  First choose
\(c>0\) so that
\(c(D_i^{\mathrm{up}}+1)\le3h_i/8\) for every \(i\), and let \(B_0\) bound
the worst-case deficit of the desired base restoring inequality, including
its base-potential and residual terms.  Then
\[
 A_ie_iG_i^\theta-cA_ie_i^2
 \ge\frac{3h_i}{8}A_ie_i.
\]
Choosing
\begin{equation*}
 A_i\ge\max\left\{1,\frac{8B_0}{3h_i\delta_i}\right\}
\end{equation*}
makes the inward hinge drift pay for that exterior loss.  In particular,
\[
 \1\{\mathcal G^{\mathrm c}\}
 \le\sum_{i=1}^N\frac{e_i^2}{\delta_i^2}
 \le C\mathcal W^{\mathrm{loc}}.
\]

This gives localized analogues of the three bounds used in the main proof:
metric and gradient control, a restoring inequality, and an operational-gap
charge.  Define
\begin{align*}
 \mathcal R^{\mathrm{up}}(\bm x,\lambda;r)
 &:=\sum_i m_i(x_i)
 \bigl(g_i'(s_i^*(r))+\lambda^*(r)\bigr)
 +\lambda(r^{\mathrm{c,up}}-r_0)^+,\\
 \Gamma(\bm y;r)
 &:=\sum_i\ell_i(y_i)-v(r)
 +\lambda^*(r)\left(\sum_i m_i(y_i)-r\right).
\end{align*}
For \(q=(\bm\sigma,\zeta)\) with
\(0\le\sigma_i\le m_i(x_i^\star)\) and
\(0\le\zeta\le\lambda_{\max}\), projected
\(z\in\calK^{\mathrm{up}}(r)\), and
\(0\le y_i\le D_i^{\mathrm{up}}\), the required schematic estimates are
\begin{align*}
 \|\bm m(\bm x)-\bm\sigma\|^2+|\lambda-\zeta|^2
 +\1\{\mathcal G^{\mathrm c}\}
 &\le C\mathcal W^{\mathrm{loc}}(q,z),\notag\\
 \|\nabla_q\mathcal W^{\mathrm{loc}}(q,z)\|^2
 +\|\nabla_z\mathcal W^{\mathrm{loc}}(q,z)\|^2
 &\le C\mathcal W^{\mathrm{loc}}(q,z),\notag\\
 \langle\nabla_z\mathcal W^{\mathrm{loc}}(q^\eta(r),z),
 G^{\theta,\mathrm{up}}(\bm x,\lambda;r)\rangle
 &\ge c\mathcal W^{\mathrm{loc}}(q^\eta(r),z)
 +c\mathcal R^{\mathrm{up}}(\bm x,\lambda;r)-C\eta^2,\notag\\
 0\le\Gamma(\bm y;r)
 &\le C\left[
 \mathcal W^{\mathrm{loc}}(q^*(r),z)
 +\mathcal R^{\mathrm{up}}(\bm x,\lambda;r)
 +\|\bm y-\bm x\|_1\right].
\end{align*}
The exact-to-smoothed transfer still uses the main mechanism.  On
\(\mathcal G\), the active-coordinate and slack-boundary-layer calculation in
\Cref{lem:two-contributions}\textup{(i)} is unchanged.  On
\(\mathcal G^{\mathrm c}\), joint-gradient Lipschitzness gives an
\(O(\eta)\) perturbation, while
\(\1\{\mathcal G^{\mathrm c}\}\le C\mathcal W^{\mathrm{loc}}\) converts it
to \(O(\eta\sqrt{\mathcal W^{\mathrm{loc}}})\). Young's inequality then
gives the same contraction loss and an \(O(\eta^2)\) remainder.

For the operational charge, apply the main bound first at \(\bm x\).  Outside
the true support, the additional holding term is
\[
 \ell_i(y)=g_i(m_i(y))+h_i(y-\bar d_i)^+.
\]
At \(\bm x\), this exterior case is charged by the hinge.  Since
\(\ell_i\) and \(m_i\) are Lipschitz on the fixed envelope interval,
replacing \(\bm x\) by the implemented \(\bm y\) costs at most
\(C\|\bm y-\bm x\|_1\).

\textbf{Reusing the proof architecture.}
Set \(\eta_t^2=\alpha_t\),
\(q_t=q^{\eta_t}(r_t)\), and
\(\widetilde W_t^{\mathrm{loc}}
:=\mathcal W^{\mathrm{loc}}(q_t,z_t)\).  Combining
\Cref{lem:generic-moving-mirror}, the localized ingredients above, and the
comparator-motion proof of
\Cref{lem:two-contributions}\textup{(ii)} gives the same type of recursion.
In that case split, \(D^{\mathrm{up}}\) replaces \(\Dsum\): below the cap the
resource increment is bounded, while above it the comparator is already
constant on the slack branch.  Thus
\begin{align*}
 \E[\widetilde W_{t+1}^{\mathrm{loc}}\mid\calH_t]
 &\le(1-c\alpha_t)\widetilde W_t^{\mathrm{loc}}
 -c\alpha_t\mathcal R^{\mathrm{up}}(\bm X_t,\lambda_t;r_t)\notag\\
 &\quad+C\alpha_t^2
 +C\alpha_t\|\bm Y_t-\bm X_t\|^2.
\end{align*}
The only altered Taylor estimate is
\[
 \frac{C\sqrt{\widetilde W_t^{\mathrm{loc}}}}
 {\eta_t(n_t-1)^2}
 \le \frac{c}{8}\alpha_t\widetilde W_t^{\mathrm{loc}}
 +C\alpha_t^2,
\]
which follows from \(\eta_t^2=\alpha_t\) and the same large-\(n_t\) bounds as
in Appendix~A.  The reciprocal-step summations, together with
\eqref{eq:envelope-implementation-sketch}, then control the localized
potential and residual at logarithmic order.  Finally, the Bellman telescope
uses \(D^{\mathrm{up}}\) in place of \(\Dsum\). When
\(r_t>D^{\mathrm{up}}\), both resource rates lie on the constant slack branch
of \(v\), so the Taylor remainder is zero.

These calculations identify the additional ingredients needed for a
fixed-envelope implementation and show how they enter the already established
proof architecture.  We record the route as an extension sketch rather than
as a separate theorem, and formal proof could be constructed from the sketch.

\subsection{\texorpdfstring{Two-Rate RAPDL: An Extension Sketch}
{Two-Rate RAPDL: An Extension Sketch}}
\label{app:two-rate-rapdl}

This subsection records the transfer argument for using different primal and
dual step scales.  It isolates the weighted-potential identity that allows the
common-rate proof to be reused, rather than repeating that proof.

Fix \(\varkappa>0\), treated as constant as the horizon varies, and write
\[
 P_\varkappa:=\operatorname{diag}(I_N,\varkappa),
 \qquad
 a_t:=\frac{\chi_x}{\tau_0+\min\{t,n_t\}}.
\]
The two-rate update is
\begin{equation*}
 z_{t+1}
 :=\proj_{\calK(r_{t+1})}
 \bigl(z_t-a_tP_\varkappa\widehat G_t\bigr),
 \qquad t<T.
\end{equation*}
Equivalently,
\(\alpha_t^x=a_t\),
\(\alpha_t^\lambda=\varkappa a_t\), and
\(\chi_\lambda=\varkappa\chi_x\).  Inventory decisions, observations,
resource updates, and clipping sets remain those of \Cref{alg:rapdl}.

\textbf{The weighted-potential cancellation.}
For the primal blocks \(H_i\) in \eqref{eq:joint-lyapunov}, define
\begin{equation*}
 \mathcal W_\varkappa(q,z)
 :=\sum_{i=1}^NH_i(\sigma_i,x_i)
 +\frac{1}{2\varkappa}(\lambda-\zeta)^2.
\end{equation*}
It is equivalent to the common-rate potential:
\begin{equation}
 \min\{1,\varkappa^{-1}\}\mathcal W(q,z)
 \le\mathcal W_\varkappa(q,z)
 \le\max\{1,\varkappa^{-1}\}\mathcal W(q,z).
 \label{eq:weighted-potential-equivalence-sketch}
\end{equation}
For any direction \(g=(g_1,\ldots,g_N,g_0)\),
\begin{equation}
 \left\langle
 \nabla_z\mathcal W_\varkappa(q,z),P_\varkappa g
 \right\rangle
 =\sum_{i=1}^N\partial_{x_i}H_i(\sigma_i,x_i)g_i
 +(\lambda-\zeta)g_0.
 \label{eq:weighted-first-order-cancellation-sketch}
\end{equation}
Thus the factor \(\varkappa\) in the dual step is canceled exactly by the
dual weight \(1/\varkappa\).  The first-order Primal-Dual pairing in
\Cref{lem:two-contributions}\textup{(i)} is unchanged.

\textbf{What changes in the reusable bounds.}
The metric and gradient constants become
\(\varkappa\)-dependent through
\eqref{eq:weighted-potential-equivalence-sketch}.  Projection compatibility is
unchanged: the primal projection is the same, and multiplying the dual square
by the positive constant \(1/\varkappa\) preserves interval-projection
monotonicity.  The stochastic-direction envelope becomes
\[
 \overline G_0:=\Dsum+\theta G_j,
 \qquad
 G_\varkappa
 :=\left(\sum_{i=1}^NG_i^2
 +\varkappa^2\overline G_0^2\right)^{1/2},
 \qquad
 \|P_\varkappa\widehat G_t\|\le G_\varkappa,
\]
so the mirror quadratic term remains \(O_\varkappa(a_t^2)\).

On the comparator side, the smoothing lemma is unchanged.  Weighted gradient
control and \eqref{eq:weighted-first-order-cancellation-sketch} transfer the
two parts of \Cref{lem:two-contributions}, with constants
\(C_{W,\varkappa}\),
\(L_{\mathcal W,\varkappa}\), and
\(C_{\mathrm{pair},\varkappa}\).  The last constant is obtained from
\eqref{eq:explicit-c-pair} by replacing the comparator-gradient envelope with
its weighted counterpart.  Set
\[
 a_0:=\min\{1,G_\varkappa^{-1}\},\qquad
 \mu_{\mathrm{upd},\varkappa}
 :=\frac{\mu_0}{8C_{W,\varkappa}},\qquad
 \mu_\varkappa
 :=\min\{\mu_{\mathrm{upd},\varkappa}/2,1/2\}.
\]
Under the reused initialization condition
\(\tau_0\ge\chi_x/a_0\), one has
\(a_t\le a_0\) and \(a_tG_\varkappa\le1\).  The tuning construction is then
reused under the
substitution
\[
 (G,C_W,L_{\mathcal W},\mu_{\mathrm{upd}},C_{\mathrm{pair}},
 \alpha_0,\chi)
 \mapsto
 (G_\varkappa,C_{W,\varkappa},L_{\mathcal W,\varkappa},
 \mu_{\mathrm{upd},\varkappa},C_{\mathrm{pair},\varkappa},
 a_0,\chi_x).
\]

\textbf{Resulting recursion and summation.}
Set \(\eta_t^2=a_t\),
\(q_t=q^{\eta_t}(r_t)\), and
\(\widetilde W_{\varkappa,t}:=\mathcal W_\varkappa(q_t,z_t)\).  With
\(\mathcal R_t\) denoting the same nonnegative residual as in
\eqref{eq:mirror-recursion}, the weighted moving-comparator calculation has
the form
\begin{align*}
 \E[\widetilde W_{\varkappa,t+1}\mid\calH_t]
 &\le(1-\mu_\varkappa a_t)\widetilde W_{\varkappa,t}
 -\mu_\varkappa a_t\mathcal R_t\notag\\
 &\quad+C_\varkappa a_t^2
 +C_\varkappa a_t\|\bm Y_t-\bm X_t\|^2.
\end{align*}
Appendix~\ref{sec:tracking-assembly} applies with \(a_t\) in place of
\(\alpha_t\).  The implementation argument also changes only through the
primal variation bound
\[
 |X_{i,t+1}-X_{i,t}|
 \le G_i a_t+\frac{\Dsum}{p_i(n_t-1)},
\]
because the dual rate does not enter the physical excess-inventory recursion.
The exact-to-smoothed comparison, implementation-gap bound, and re-solving
reduction then follow the common-rate proof with constants allowed to depend
on \(\varkappa\).

These observations indicate how the common-rate analysis extends to any fixed
\(\varkappa>0\).  We record the transfer argument as a sketch rather than
repeat the proof of \Cref{thm:main}.

\end{APPENDICES}
\end{document}